\documentclass{article}

    \usepackage[preprint]{neurips_2024}

\usepackage[utf8]{inputenc} 
\usepackage[T1]{fontenc}    
\usepackage{hyperref}       
\usepackage{url}            
\usepackage{booktabs}       
\usepackage{amsfonts}       
\usepackage{nicefrac}       
\usepackage{microtype}      
\usepackage{xcolor}         
\usepackage{graphicx}
\usepackage{amsmath}
\usepackage{amsthm}
\usepackage{amssymb}
\usepackage{bbm}
\usepackage{float}
\usepackage{subfigure}
\usepackage[capitalize,noabbrev]{cleveref}
\usepackage{pifont}

\usepackage{multirow}
\usepackage[normalem]{ulem}
\usepackage{tikz}
\usetikzlibrary{calc}
\usepackage{tablefootnote}
\usepackage{tabularx} \newcolumntype{C}[1]{>{\centering\arraybackslash}p{#1}} 

\graphicspath{{fig/}}
\usepackage{algorithm} 
\usepackage[noend]{algpseudocode}

\theoremstyle{plain}
\newtheorem{theorem}{Theorem}

\newtheorem{lemma}[theorem]{Lemma}

\newtheorem{assumption}[theorem]{Assumption}

\theoremstyle{definition}

\newtheorem{example}[theorem]{Example}
\newtheorem{remark}[theorem]{Remark}

\newtheorem*{theorem*}{Theorem}
\newtheorem*{thm_convex}{Theorem \ref{thm:conv_as}}
\newtheorem*{thm_strong_convex}{Theorem \ref{thm:convergence_strong_convex}}

\newcommand{\E}{\mathbb{E}}

\newcommand{\F}{\mathcal{F}}
\newcommand{\R}{\mathbb{R}}

\newcommand{\pmi}{p_{\min}}

\title{Asynchronous Federated Stochastic Optimization for Heterogeneous Objectives Under Arbitrary Delays}

\author{%
Charikleia Iakovidou$^{1}$ \quad Kibaek Kim$^{1}$ \\
$^1$MCS, Argonne National Laboratory \\
\textsc{\{ciakovidou,kimk\}@anl.gov}
}

\begin{document}

\maketitle

\begin{abstract}
 Federated learning (FL) was recently proposed to securely train models with data held over multiple locations (``clients'') under the coordination of a central server. Prolonged training times caused by slow clients may hinder the performance of FL; while asynchronous communication is a promising solution, highly heterogeneous client response times under non-IID local data may introduce significant bias to the global model, particularly in client-driven setups where sampling is infeasible. To address this issue, we propose \underline{A}synch\underline{R}onous \underline{E}xact \underline{A}veraging (\textsc{AREA}), a stochastic (sub)gradient method that leverages asynchrony for scalability and uses client-side memory to correct the bias induced by uneven participation, without client sampling or prior knowledge of client latencies. \textsc{AREA} communicates model residuals rather than gradient estimates, reducing exposure to gradient inversion, and is compatible with secure aggregation. Under standard assumptions and unbounded, heterogeneous delays with finite mean, AREA achieves optimal convergence rates: $\mathcal{O}(1/K)$ in the strongly convex, smooth regime and $\mathcal{O}(1/\sqrt{K})$ in the convex, nonsmooth regime. For strongly convex, smooth objectives, we demonstrate theoretically and empirically that AREA accommodates larger step sizes than existing methods, enabling fast convergence without adversely impacting model generalization. In the convex, nonsmooth setting, to our knowledge we are the first to obtain rates that scale with the average client update frequency rather than the minimum or maximum, indicating increased robustness to outliers.
\end{abstract}

\section{Introduction}
	\label{sec:intro}
Federated learning (FL) is a distributed machine learning paradigm that trains a shared model across multiple decentralized data sources, without the communication of data points \cite{konecny_federated_2015,fedavg}. A common formulation of FL involves a finite-sum optimization problem of the form
\begin{equation}
\label{eq:prob_gen}
\min_{x \in \R^d} \quad f(x) = \frac{1}{n}\sum_{i=1}^n f_i(x),
\end{equation}
where the $i^{th}$ data-owning source (client $i \in \{1,...,n\}$) has private access to the function $f_i:\R^d \rightarrow \R$ representing the loss over its local data, and  $x \in \R^d$ is the vector of model parameters. Throughout this work, we will refer to $f$  and the functions $f_i$ for $i \in \{1,...,n\}$ as the global and local loss functions, respectively.

\begin{figure}
    \centering
    \begin{tikzpicture}
\begin{footnotesize}

    \node[draw, fill=black!20,  minimum size=.5cm, label=above:server] (circle1) at (0,0) {};
    \node[draw, fill=black!20, minimum size=.5cm, label=above:{$w_3 \gets w_0 - \alpha \textcolor{blue}{g_0^1}$}] (circle2) at (2,0) {};
    \node[draw, fill=black!20,minimum size=.5cm, label=above:{$w_3 \gets w_0 - \alpha (\textcolor{blue}{g_0^1}+\textcolor{blue}{g_1^1})$}] (circle3) at (5.1,0) {};
    \node[draw, fill=black!20, minimum size=.5cm, label=above:{$w_3 \gets w_0 - \alpha (\textcolor{blue}{g_0^1}+\textcolor{blue}{g_1^1}+\textcolor{red}{g_0^2})$}] (circle4) at (9,0) {};

    \node[draw, fill=blue!50, circle, minimum size=.5cm, label=below:{\textcolor{blue}{client 1}}] (blue1) at (0.8,-1.5) {};
    \node[draw, fill=blue!50, circle, minimum size=.5cm] (blue2) at (3,-1.5) {};
    \node[draw, fill=blue!50, circle, minimum size=.5cm] (blue3) at (6.5,-1.5) {};

    \node[draw, fill=red!50, circle, minimum size=.5cm, label=below:{\textcolor{red}{client 2}}] (red1) at (0.8,-3) {};

    \draw[->] (circle1) -- (circle2); 
    \draw[->] (circle2) -- (circle3); 
    \draw[->] (circle3) -- (circle4); 
    \draw[->] (circle1) to[out=-90, in=-180] node[midway,  right] {$w_0$} (blue1);
      \draw[->] (circle2) to[out=-90, in=-180] node[midway,  right] {$w_1$} (blue2);
    \draw[->] (circle3) to[out=-90, in=-180] node[midway,  right] {$w_2$} (blue3);
    \draw[->] (circle1) to[out=-90, in=-180] node[near end,  right] {$w_0$} (red1);
\draw[->] (blue1) to[out=0, in=-180] node[midway,  left] {\textcolor{blue}{$g_0^1$}} (circle2);
\draw[->] (blue2) to[out=0, in=-180] node[midway,  left] {\textcolor{blue}{$g_1^1$}} (circle3);
    \draw[->] (circle2) to[out=-90, in=-180] (blue2);
    \draw[->] (circle3) to[out=-90, in=-180] (blue3);
    \draw[->] (red1) to[out=0, in=-90] node[very near start, above] {\textcolor{red}{$g_0^2$}} (circle4);
\end{footnotesize}
\end{tikzpicture}
    \caption{Asynchronous FL with non-uniform client update frequencies.}
    \label{fig:hetero_clients}
\end{figure}
During a typical FL algorithm round, the server selects a subset of clients and sends them the global model. The selected clients independently optimize the model parameters using their local data, and share their estimates with the server. After all selected clients have returned their updates, the server produces the global model via some aggregation process (e.g., weighted averaging). As a result, the convergence of FL is limited by the speed of the slowest client (``straggler''). The adoption of asynchronous communication is a promising solution to this problem: it allows the server to perform global updates without waiting for the slowest client, and better utilizes resources by reducing idle time for fast clients. However, when local data is non-IID and client response times vary widely, the asynchronous aggregation of client updates may introduce significant bias to the global model. A simple example is illustrated in Fig. \ref{fig:hetero_clients}: under a naive, asynchronous aggregation protocol where the server updates the global model upon receipt of any client update, the global model is disproportionately influenced by the fastest client. Extrapolating to scenarios with more than two clients, it becomes clear that a client's impact on the global model can be significantly influenced by its response latency, implicitly favoring faster, frequently communicating clients. This client-dependent deviation away from the optimal global model is distinct from the well-documented phenomenon of  ``client drift'' \cite{scaffold, fedlin}, caused by the execution of multiple local stochastic gradient descent (SGD) steps at each client between communications with the server. 

Asynchronous SGD \cite{lian_asynchronous_2019,stich_error-feedback_2021} does not suffer from client drift due to restricting clients to a single local SGD step between communications with the server. Nevertheless, convergence to a  solution of Problem \ref{eq:prob_gen} generally requires the adoption of an unbiased client-sampling scheme (e.g., uniform over clients, or  proportional to their importance in the global objective) \cite{koloskova2022sharper, even2024asynchronous}; otherwise the convergence rate is dominated by a non-vanishing error term proportional to cross-client data heterogeneity \cite{mishchenko2022asynchronous}. Similar negative results have been established for FL \cite{yang2022anarchic, fraboni_general_theory}. Alternatively,  the strong assumption of uniform client participation may be adopted to guarantee the convergence of asynchronous FL under non-IID data \cite{fedbuff, toghani2022unbounded, yang2022anarchic}.

Although client-sampling techniques effectively remove the bias caused by uneven client participation, they may require prior knowledge of client data or latency distributions, or rely on timely responses from sampled clients. They also impose substantial server-side orchestration, whereas in practical settings communication may be driven by clients \cite{yang2022anarchic}. This work addresses uneven participation bias in a fully asynchronous, heterogeneous, client-driven setting. Specifically, we propose and analyze \underline{A}synch\underline{R}onous \underline{E}xact \underline{A}veraging (\textsc{AREA}), a new asynchronous FL method that provably corrects  uneven participation bias via a memory-based correction scheme, without client sampling or any prior knowledge of client response times. Although prior knowledge can be utilized to optimally tune the proposed method, it is not required to guarantee convergence to a solution of Problem \ref{eq:prob_gen}.

In addition, our guarantees do not require the strong assumption of a bounded maximum delay common in prior work (e.g.,  \cite{avdiukhin2021federated,stich_critical_2021,koloskova2022sharper,fedbuff}). Instead, we adopt mild assumptions on delay statistics (i.e., finite mean), which enables providing guarantees for the more realistic model of ``total asynchrony'' \cite{bertsekas2015parallel}, i.e., arbitrarily long delays between iterations. Methods analyzed under total asynchrony typically employ delay-adaptive step sizes \cite{koloskova2022sharper, mishchenko2022asynchronous,feyzmahdavian2023asynchronous}, which may substantially reduce the step size to accommodate long delays and hence result in slow convergence. In contrast, we use step size schedules that are independent of the per-iteration delays.

\paragraph{Contributions.} We propose \textsc{AREA}, a new asynchronous stochastic (sub)gradient method for FL  with non-IID data. \textsc{AREA} addresses stragglers and scalability via asynchronous communication and employs a memory-based correction mechanism that removes the bias induced by heterogeneous client update frequencies. Our contributions are:
\begin{itemize}
    \item     We introduce and analyze \textsc{AREA}, an asynchronous stochastic (sub)gradient method for FL with non-IID data. \textsc{AREA} is particularly well-suited for practical settings due to the following advantages: $i$) it is equipped with a memory-based scheme that corrects the bias induced by uneven client participation, without server-side coordination (e.g., sampling) or prior knowledge of client latencies; $ii$) it accommodates highly heterogeneous, client-driven communication patterns (e.g., ``anarchic FL'' \cite{yang2022anarchic}); $iii$) instead of local gradients, clients share with the server model residuals that are less susceptible to gradient inversion, enhancing privacy; moreover, \textsc{AREA} is compatible with secure aggregation; $iv$) \textsc{AREA} utilizes a flexible server aggregation protocol, that can be adapted to application demands without affecting the order of the convergence rate.
   
    \item We provide rigorous convergence guarantees for \textsc{AREA} for $i$) strongly convex, smooth objectives and $ii$) convex, Lipschitz (nonsmooth) objectives. In the strongly convex and smooth regime, we recover the optimal $\mathcal{O}(1/K)$ rate for the expected mean squared distance to optimality of the last iterate after $K$ iterations, under unbounded delays (with finite mean), without uniformly bounded gradients. We also show, theoretically and empirically, that \textsc{AREA} supports larger step sizes than existing  methods under comparable assumptions, and may achieve low generalization error faster as a result. 
    
    \item In the convex, nonsmooth regime, we are, to the best of our knowledge, the first to establish a convergence rate under asynchronous communication with unbounded delays and non-IID local data without the use client sampling. This rate depends on the average client frequency rather than the smallest or largest one, indicating  robustness to outliers. Moreover, our framework yields a closed-form expression for the optimal frequency of global updates at the server under a standard model of exponentially distributed delays.
\end{itemize}

For a comparison with prior work in the settings we study, see Tables \ref{tab:existing_strong} (strongly convex and smooth) and~\ref{tab:existing_not_strong} (convex and nonsmooth). With the exception of \cite{avdiukhin2021federated}, we compare to asynchronous SGD due to the scarcity of comparable FL results for convex objectives; under standard assumptions, multiple local steps affect constants (client drift error) but not the asymptotic order.

\textbf{Notation:} Throughout this work, the set of integers $\{1,...,n\}$ is denoted by $[n]$. We  use the notation $\F_k$ to refer to the $\sigma$-algebra containing the history of the algorithm up to and including the $k^{th}$ iteration. The subdifferential of the function $f:\R^d \rightarrow \R$ at $x \in \R^d$ is denoted by $\partial f(x)$. We will use
$I$ and $1_n$  to refer to the identity matrix and the vector of ones of size $n$, respectively, and the notation $\|\cdot\|$ for the $\ell_2$-norm. We denote the $i^{th}$ element of vector $u$ by $[u]_i$, and the element in the $i^{th}$ row and $j^{th}$ column of a matrix $M$ by $[M]_{ij}$. The transpose of a vector $v$ will be denoted by $v'$. For a set of vectors $\{v_1,...,v_n\} \in \R^d$ we  define the average vector $\bar{v} \triangleq \frac{1}{n}\sum_{i=1}^n v_i \in \R^d$. 

\begin{table}[t]
  \label{tab:existing_strong}
 \caption{Summary of asynchronous methods analyzed under strong convexity and smoothness. 
  We refer to the asynchronous version of \textsc{SGD} as \textsc{AsyncSGD}, and denote by $k$ and $K$ the global iteration count and the total number of iterations, respectively. For the first four rows, convergence rates are reported for the averaged iterates, while for the $5^{th}$ row (our work) the convergence rate is reported for the last iterate. Rates are stated in the original metrics (under the stated assumptions, function gap and squared distance are equivalent up to constants). For rates of the form $\mathcal{O}(e^{-cK} + \tfrac{\sigma^2}{K})$ for $c>0$ there is a trade-off between the two dominant terms, as the first decreases and the second increases with the step size. We use the term ``delay-dependent'', to refer to  non-decreasing step size schedules dependent on the maximum delay, and the term ``delay-adaptive'' to refer to non-decreasing step size schedules dependent on the per-iteration delays.}
\footnotesize
    \centering
\begin{tabular}{lcccc}
\toprule
\textbf{Reference} & \textbf{Data} & \textbf{Step} & \textbf{Rate}\\
\midrule
 \textsc{AsyncSGD } & \multirow{2}{*}{IID}  & \multirow{2}{*}{Bounded}  & \multirow{2}{*}{delay-dependent} & \multirow{2}{*}{$\mathcal{O}\left(e^{-cK} + \frac{\sigma^2}{K} \right)$}  \\
 Thm 16, \cite{stich_error-feedback_2021}&   & & & \\ 
\midrule
 \textsc{AsyncCommSGD } & \multirow{2}{*}{non-IID}  & \multirow{2}{*}{Bounded}  & \multirow{2}{*}{$\Theta(\tfrac{1}{k})$  (sync\tablefootnote{Clients reduce their step size according to a shared global clock.}) } & \multirow{2}{*}{$\mathcal{O}(\tfrac{\sigma^2}{K})$ }  \\
Thm 2.2, \cite{avdiukhin2021federated}&   & & & \\ 
\midrule
 \textsc{AsyncSGD } & \multirow{2}{*}{IID}  & \multirow{2}{*}{Unbounded}  & \multirow{2}{*}{delay-adaptive} & \multirow{2}{*}{$\mathcal{O}\left(e^{-cK} + \frac{\sigma^2}{K} \right)$}  \\
 Thm 3.2, \cite{mishchenko2022asynchronous}&   & & & \\ 
\midrule
 \textsc{AsyncSGD } & \multirow{2}{*}{IID}  & \multirow{2}{*}{Bounded}  & \multirow{2}{*}{delay-dependent} & \multirow{2}{*}{$\mathcal{O}\left(e^{-cK} + \frac{\sigma^2}{K} \right)$}  \\
Thm 16, \cite{feyzmahdavian2023asynchronous}&   & & & \\ 
\midrule
 \textsc{AREA} & \multirow{2}{*}{non-IID}  & \multirow{2}{*}{Unbounded}  & \multirow{2}{*}{ $\Theta(\tfrac{1}{k})$ (async\tablefootnote{Each client is assigned a step size from the server for their next training round when they receive the global model. This step size is kept constant until the next communication with the server, and clients do not share the server's global clock.})} & \multirow{2}{*}{$\mathcal{O}(\tfrac{\sigma^2}{K})$ }  \\
 Thm \ref{thm:convergence_strong_convex}, this work&   & & & \\ 
\bottomrule
\end{tabular}
\end{table}

\begin{table}[t]
 \caption{Summary of  asynchronous methods analyzed for convex and Lipschitz functions. 
   We refer to the asynchronous version of \textsc{SGD} as \textsc{AsyncSGD}, and denote by $K$ the total number of iterations. All convergence rates are reported for the function gap of the averaged iterates.
  }
  \label{tab:existing_not_strong}
\footnotesize
    \centering
\begin{tabular}{lccc}
\toprule
\textbf{Reference} & \textbf{Data} & \textbf{Client selection} & \textbf{Rate}\\
\midrule
 \textsc{AsyncSGD } & \multirow{2}{*}{IID}  & \multirow{2}{*}{no sampling}  & \multirow{2}{*}{$\mathcal{O}\left(1 / \sqrt{K}\right)$}  \\
 Thm 1, \cite{mishchenko2022asynchronous}&   & &  \\ 
\midrule
 \textsc{AsyncSGD } & \multirow{2}{*}{non-IID}  & \multirow{2}{*}{unbiased sampling}  & \multirow{2}{*}{$\mathcal{O}\left(1 / \sqrt{K}\right)$}  \\
Thm 1, \cite{even2024asynchronous}&   & & \\ 
\midrule
 \textsc{AREA} & \multirow{2}{*}{non-IID}  & \multirow{2}{*}{no sampling}   &  \multirow{2}{*}{$\mathcal{O}\left(1 / \sqrt{K}\right)$}  \\
 Thm \ref{thm:conv_as}, this work&   & &  \\ 
\bottomrule
\end{tabular}
\end{table}

\section{Related Work}
\label{sec:related}

\paragraph{Asynchronous FL.} There is a rich body of work on FL methods utilizing fully or partially asynchronous communication to alleviate straggler effects. A common strategy involves adapting the workload at each client to balance heterogeneous response times. However, this orchestration may come at the cost of weakened privacy guarantees due to sharing detailed device statistics with the server \cite{chen_fedsa_2021}, or limited scalability due to extensive record keeping \cite{wang2022asyncfeded} or the enforcement of partial synchronization \cite{li_fedcompass_2023}. Other partially asynchronous approaches allow the server to interrupt local training and request early, inexact updates, which may lead to additional systems complexity \cite{zakerinia_communication-efficient_2023,leconte_favas_2023}. Staleness-adaptive aggregation weights have been proposed for FL with IID local data \cite{fedasync}, but in the non-IID regime this approach may discount information from stragglers with unseen data distributions. 

Several FL methods readily accommodate asynchronous communication without additional server orchestration or adaptation. \textsc{FedBuff} \cite{fedbuff,toghani2022unbounded} aggregates stale updates indiscriminately and allows clients to work at their own pace, but requires uniform client sampling to guarantee convergence to a solution of Problem \ref{eq:prob_gen}. Asynchronous FL with non-uniform client participation has been studied in \cite{yang2022anarchic} and \cite{fraboni_general_theory} for smooth, nonconvex and smooth, convex objectives, respectively; in both cases, the convergence rate is dominated by a non-vanishing error term proportional to data heterogeneity, preventing convergence to a solution of Problem \ref{eq:prob_gen}. Conversely, \textsc{AsyncCommSGD} \cite{avdiukhin2021federated} achieves asymptotic convergence under heterogeneous client delays, but assumes a shared global clock on both server and clients, complicating FL deployment. Closest to our work are asynchronous FL methods storing corrective terms to offset participation bias \cite{yang2022anarchic,wang_tackling_nodate}; however, these require the server to store one full-model-size vector for each client, which is feasible only in small cross-silo setups. Finally, all aforementioned studies require a strict bound on the maximum staleness, yielding overly pessimistic guarantees; more importantly, in realistic scenarios this assumption is often violated. 

\textbf{Asynchronous SGD.}  Asynchronous \textsc{SGD} and its variants are the cornerstones of large-scale parallel and distributed optimization \cite{tsitsiklis_distributed_1986, ram_async_gossip,dean2012large, recht2011hogwild, peng_arock_2016}. The convergence properties of asynchronous \textsc{SGD} have been extensively studied both under IID \cite{lian_asynchronous_2019,stich2021critical,stich_error-feedback_2021,mishchenko2022asynchronous,feyzmahdavian2023asynchronous} and non-IID local data \cite{mishchenko2022asynchronous,
koloskova2022sharper,even2024asynchronous}, i.e., for homogeneous and heterogeneous functions $f_i$ in Problem \ref{eq:prob_gen}, respectively. Under heterogeneous local functions, convergence to a solution of Problem \ref{eq:prob_gen} can only be guaranteed if the server adopts an unbiased client sampling scheme, such as uniform \cite{koloskova2022sharper}, or according to the client's importance in the global objective \cite{even2024asynchronous}. Aggregating client updates indiscriminately under biased participation can speed up the initial progress of asynchronous \textsc{SGD} \cite{mishchenko2022asynchronous}; however, as in the case of asynchronous FL, the convergence rate is dominated by a non-vanishing error term proportional to function heterogeneity.

\section{\textsc{AREA}: An Asynchronous Method for Federated Stochastic Optimization}

We propose a new asynchronous stochastic (sub)gradient algorithm,  
\textsc{AREA}, to solve Problem~\ref{eq:prob_gen} in a federated manner under arbitrary delays. We assume that client $i \in [n]$ can obtain only a stochastic approximation of the true (sub)gradient of $f_i$ (e.g., stochatic (sub)gradient, mini-batch), satisfying the following standard assumption. 
\begin{assumption}
    \label{assum:grad_stoch} Let $d_i(x)$ be the  (sub)gradient  of $f_i$ at $x \in \R^d$, and let $g_i(x)$ be a stochastic approximation of $d_i(x)$. Then for all $i \in [n]$ and some $\sigma>0$, we have: 
    \begin{align}
        \E[ g_i(x) - d_i(x)|x] =  0 ,  \quad \E[\| g_i(x) - d_i(x)\|^2 |x] \leq \sigma^2
    \end{align}
   where the expectation is taken over all possible realizations of $g_i(x)$.
\end{assumption}
\begin{algorithm}
   \caption{\textsc{AREA} (server)}  \label{alg:area_server}
\begin{algorithmic}
\State \textbf{Initialization:} initialize model $x_{s,0} \in \R^d$, aggregator $u_{s,0} = 0 \in \R^d$, step size $\alpha_0 > 0$; broadcast $x_{s,0}$ and $\alpha_0$ to all clients $i \in [n]$
\For{$k=1,2,...,K$}
\If{a request from client $i$ is received}
\State receive $m_i$ from client $i$
\State update $u_{s,k} \gets u_{s,k-1} + \frac{1}{n}m_i$
\State send $x_{s,k-1}$ and $\alpha_k$ to client $i$
\ElsIf{the aggregation criterion is satisfied}
 \State update $x_{s,k} \gets x_{s,k-1} + u_{s,k-1}$ and $u_{s,k} \gets 0$
\EndIf
\EndFor
\State  send stopping signal to all clients
\end{algorithmic}
\end{algorithm}
The procedure of \textsc{AREA} at the server is summarized in Algorithm \ref{alg:area_server}. The server maintains the global model $x_s \in \R^d$ and an aggregator $u_s \in \R^d$, and uses a private aggregation criterion to determine when to update the global model. The server also tracks the global clock (iteration count $k$) and uses it to generate an appropriate step size sequence $\{\alpha_k\}$  to assign to clients. At initialization,  the server sets $u_{s,0}=0$, and initializes $x_{s,0}$ and $\alpha_0$, and broadcasts them to all clients. Whenever an update $m_i \in \R^d$ (to be defined later) arrives from client $i \in [n]$, the server performs the following steps: $i$) scales $m_i$ with $1/n$ and accumulates it into $u_s$, $ii$) sends the current versions of $x_s$ and $\alpha$ to client~$i$, and $iii$) increments the global iteration counter by $1$. Alternatively, if the aggregation criterion is satisfied, the server $i$) updates the global model by adding the contents of the aggregator, $ii$) resets the aggregator to zero, and  $iii$) increments the iteration count by~$1$.  These steps are repeated for $K$ global iterations, after which the server broadcasts a stopping signal to all clients.
\begin{remark}
The aggregation criterion can be selected on a case-by-case basis to meet application demands.  Possible options include aggregating every $\Delta >0$ client communications (e.g., \cite{fedbuff}), periodically after fixed time intervals (e.g., \cite{fraboni_general_theory}), or at random. We note that it is possible to incorporate secure aggregation to \textsc{AREA} for protection against honest-but-curious server attacks \cite{secagg, fedbuff}.
\end{remark}
\begin{algorithm}
   \caption{\textsc{AREA} (client $i$)}   \label{alg:area_client}
\begin{algorithmic}
\State \textbf{Initialization:} set $y_{i}  = x_{s,0}$
\While{stopping signal is not received}
 \State receive $x_s$ and $\alpha$ from server and set $x_i^0 \gets x_s$
\For{$m=1,...,M$}
\State calculate stochastic (sub)gradient $g_i(x_i^{m-1})$
\State update $x_i^m \gets x_i^{m-1} - \alpha g_i(x_i^{m-1})$
\EndFor
     \State send $m_i = x_i^M - y_i$ to server and request $x_s$ and $\alpha$
    \State update $y_i \gets x_i^M$
\EndWhile
\end{algorithmic}
\end{algorithm}
The process of \textsc{AREA} at client $i\in[n]$ is shown in Algorithm \ref{alg:area_client}. Client $i \in [n]$ maintains one memory variable $y_i \in \R^d$, storing its immediately previous local estimate of the global model. During the initialization phase, clients receive $\alpha_0$ and $x_{s,0}$ and set $y_i = x_{s,0}$. Each client then repeats the following at its own pace until a stopping signal from the server is received: it performs $M$ local stochastic (sub)gradient descent steps with step size equal to its most recently received $\alpha$, starting from its most recently received global model $x_s$, to obtain a new local estimate $x_i^M$. Client $i$ then constructs the message $m_i \triangleq x_i^M - y_i$ and sends it to the server, requesting the most recent versions of $x_s$ and $\alpha$. After sending $m_i$, it sets $y_i \gets x_i^M$.

\section{Theoretical Results}
\label{sec:analysis}
To capture more accurately  the behavior of FL clients, we adopt a probabilistic framework that permits arbitrarily long and heterogeneous response times. As shown in Algorithm \ref{alg:area_server}, an iteration  of \textsc{AREA} corresponds to an occurrence of exactly one event from the set $\{E_1,...,E_n,E_s\}$, where $E_i\triangleq\{$an update is received from client $i$$\}$, for $i \in [n]$ and $E_s \triangleq \{$the aggregation criterion is satisfied$\}$. We assume the client and server processes are independent\footnote{This assumption is not satisfied when a global update is triggered by receipt of a fixed number of client updates (e.g., \cite{fedbuff}). However, since our analysis assumes no prior information about client arrivals, the resulting bounds can be interpreted as worst-case upper bounds for strategies that exploit such information.}. Let $J_k \in \{1,...,n,s\}$ denote the index of the event occurring at iteration $k$, and define $p_{i,k} = \Pr(J_k = i)$ and $p_{s,k} \triangleq \Pr(J_k = s)$. We require $p_{i,k} = p_i > 0$ for $i \in [n]$ and $p_{s,k} = p_s > 0$ for all $k$, so that $p_s + \sum_{i=1}^n p_i =1$. Thus, $\{J_k\}$ is IID with stationary event probabilities. This framework allows unbounded, heterogeneous delays and is strictly weaker than the bounded delay assumption common in prior work. For the derivation of the iterates of \textsc{AREA} under the proposed probabilistic setup, see Appendix \ref{sec:iterate_derivation}.

\begin{example}\label{rmk:poisson} Our setup readily accommodates the Poisson process, a foundational model in queueing theory. Specifically, let the interarrival time between updates from client $i \in [n]$ be $t_i \sim \mathrm{Exp}(\lambda_i)$ and the interarrival time between global updates be $t_s \sim \mathrm{Exp}(\lambda_s)$ with $\lambda_i, \lambda_s > 0$ and all interarrivals independent. By the properties of the Poisson process, we have $p_i = \lambda_i / \bar{\lambda}$ for $i \in [n]$ and  $p_s = \lambda_s / \bar{\lambda}$, where $\bar{\lambda} = \lambda_s + \sum_{i=1}^n \lambda_i$. Poisson processes are widely used to model asynchronous communication over networks \cite{ram_async_gossip, mansoori_newton, chen2020vafl}.
\end{example}
Next, we present our theoretical results  for two instances of Problem~\eqref{eq:prob_gen}: $i$) smooth and strongly convex $f_i$, and $ii$) convex and Lipschitz $f_i$. We list these assumptions formally below.
\begin{assumption}(\textbf{Smoothness}) 
\label{assum:smooth} 
Each function $f_i$ has $L$-Lipschitz  gradients, i.e., $  \|\nabla f_i(x) - \nabla f_i(y)\| \leq L \|x-y\|,$ for all $x,y \in \R^d$ and $i \in [n]$. 
\end{assumption}
\begin{assumption}(\textbf{Strong convexity}) 
\label{assum:strong_convex} 
Each function   $f_i$  is $\mu$-strongly convex, i.e., $f_i(y) \geq f_i(x) + \langle \nabla f_i(x), y-x\rangle + \frac{\mu}{2}\|y-x\|^2,$ for all $x,y \in \R^d$ and $i \in [n]$.
\end{assumption}
Assumptions \ref{assum:smooth} and \ref{assum:strong_convex} imply that the function $f$ has $L$-Lipschitz gradients and is $\mu$-strongly convex.
\begin{assumption}\textbf{(Convex and Lipschitz functions)}
\label{assum:bounded_g}
Each function $f_i$ is  convex and $B$-Lipschitz, i.e.,  $|f_i(x) - f_i(y)| \leq B \|x-y\|$, for all  $x,y \in \R^d$ and $i \in [n]$.
\end{assumption}

\subsection{Results for strongly convex and smooth functions}
We now state our main result on the convergence of \textsc{AREA} for strongly convex, smooth functions; the proof can be found in Appendix \ref{sec:proof_of_thm_8}.
\begin{theorem}
\label{thm:convergence_strong_convex} 
Under Assumptions \ref{assum:grad_stoch}, \ref{assum:smooth} and \ref{assum:strong_convex}, let $x^\star = \arg \min_x f(x)$ and suppose that the step size sequence $\{\alpha_k\}$ in Algorithm \ref{alg:area_client} is defined as follows
\begin{align*}
    \alpha_k \triangleq \tfrac{1}{ \tfrac{\pmi M \gamma k }{48}  +  \tfrac{1}{\mathcal{D}}}, \text{ where }  \mathcal{D} \triangleq \min\left\{ \tfrac{2}{M(\mu + L)}, \tfrac{\gamma}{L^2(M-1)}, \sqrt{\tfrac{ M \gamma^2}{64L^4 (M-1)^3 }}, \sqrt[3]{\tfrac{ \gamma}{32 L^4 (M-1)^3 }} \right\},
\end{align*}
$M$ is the number of local SGD steps, $\gamma \triangleq \frac{2\mu L}{\mu + L}$, $L$ and $\mu$ are the Lipschitz and strong-convexity constants defined in Assumptions \ref{assum:smooth} and \ref{assum:strong_convex}, respectively, and $\pmi \triangleq \min\{\min_i p_i, p_s\}$.

Then after $K$ iterations, the sequence  $\{x_{s,k}\}$ of global server iterates in Algorithm \ref{alg:area_server} satisfies
\begin{equation}
\label{eq:short_rate_strong_convex}
\begin{split}
\E[\|x_{s,K} - x^\star\|^2] & = \mathcal{O}\Bigg( \frac{r_p (1+\delta \kappa) \sigma^2}{\pmi M K} + \frac{r_p \bar{q}^{1/2} L (1+\delta \kappa)^{1/2}\sigma \zeta}{\pmi M^{1/2} K^{3/2}}\\
&\quad+ \frac{r_p \bar{q}^{1/2} L^{3/2}(\sigma \zeta^3)^{1/2}}{\pmi^{1/2} K^{3/2}} + \frac{r_p  (\delta  + \pmi^{1/2} \bar{q}) L^4\zeta^2}{\pmi^2 K^2}\Bigg),       
    \end{split}
\end{equation}
where $r_p \triangleq \frac{p_s}{\pmi}$, $\bar{q} \triangleq \frac{1}{n}\sum_{i=1}^n \tfrac{1}{p_i}$, $\zeta \triangleq \max_i\|x^\star-u_i^\star\|$, $u_i^\star \triangleq \arg \min_x f_i(x)$, $\kappa \triangleq \tfrac{L}{\mu}$ is the condition number of $f$, and $\delta \triangleq 1 - \tfrac{1}{M}$.
\end{theorem}
\paragraph{Dependency on \boldmath{$\pmi$}.} The factor $1/\pmi$ in  \eqref{eq:short_rate_strong_convex} is caused by the choice of $\alpha_k=\Theta(1/(\pmi k))$, which permits steps up to $1/\pmi$ times larger than the standard $\Theta(1/k)$ schedule used in the analysis of asynchronous or delayed FL algorithms \cite{gu2021fast,avdiukhin2021federated} under strong convexity and smoothness. While larger step sizes amplify the errors due to client drift and stochastic gradient noise and thus adversely impact the value of the loss function, they may also significantly speed up convergence with negligible effect on model generalization. Our empirical results in Section \ref{sec:log_reg} support this observation: in some cases, \textsc{AREA} can tolerate step sizes nearly an order of magnitude larger compared to baseline methods, reaching and sustaining higher testing accuracy levels in shorter time, despite higher loss function values. We note that \textsc{AREA} achieves asymptotic convergence to a solution of Problem \ref{eq:prob_gen} under $\alpha_k=\Theta(1/k)$; however, the order of the rate can become suboptimal. Intuitively, the dependence on $\pmi$ in the strongly convex, smooth regime can be explained as follows: unlike FL methods where clients share gradient estimates, enabling the server to take gradient steps and tune the step size at a global level, clients in \textsc{AREA} communicate model residuals to increase privacy. As a result, the global iterate evolves through locally contractive client updates. To guarantee $\mathcal{O}(1/K)$ rate for the global system, the step size must be large enough to ensure the slowest client contracts sufficiently. As we show later in Theorem~\ref{thm:conv_as}, when the functions $f_i$ are convex and Lipschitz, the convergence rate scales with the average $\bar{q}$ rather than $\pmi$.

\paragraph{Two-phase behavior.} The convergence rate in \eqref{eq:short_rate_strong_convex} is obtained using a conservative $\Theta(1 - \pmi \alpha_k M \gamma / 12)$ bound for the contractive factor of \textsc{AREA}. However, a sharper bound established in Lemma \ref{lem:properties_of_Ak} in Appendix \ref{sec:proof_of_thm_8} is  $\Theta(1 - \pmi +\pmi( 1-\alpha_k M \gamma / 4)^{1/3})$. This implies a two-phase behavior  in the convergence of \textsc{AREA}: in the early stages  when the step size is large, the algorithm contracts at an almost constant geometric rate $\simeq 1-\pmi$, while as $\{\alpha_k\}$ decreases the convergence rate becomes sublinear. Consequently, \textsc{AREA} has the potential to achieve rapid progress in the early stages of learning provided $\pmi$ is large enough and the stochastic gradient and client drift errors do not dominate.

\paragraph{Synchronization error.} The two $\mathcal{O}(K^{-3/2})$ terms in \eqref{eq:short_rate_strong_convex}  arise from utilizing diminishing step sizes over asynchronous communication. The server increases the global iteration count whenever an update from a client is received, or a global update is performed; consequently, clients use different effective step sizes depending on the frequency of their communications with the server, and for the same client, the step size may decrease arbitrarily between consecutive communications. This misalignment induces a synchronization error of order  $K^{-3/2}$ (see the proof of Theorem \ref{thm:convergence_strong_convex} in Appendix \ref{sec:proof_of_thm_8} for more detail). Employing constant step sizes eliminates these intermediate terms, but then a convergence rate of $\mathcal{O}(e^{-\pmi K} + 1/K)$ can be attained for the averaged iterates at best.
\paragraph{Comparison with existing methods.} For synchronous FedAvg with  arbitrary client sampling strategy $\pi$, non-IID local data and step size $\alpha_k = \Theta(1/k)$, a convergence rate of $\mathcal{O}\left((\sigma^2+ \zeta^2)/K  + Q(\pi) \right)$  is established in [Theorem 3.1, \cite{cho2020client}] for strongly convex, smooth functions, where $Q(\pi)$ measures the bias of the sampling strategy $\pi$ ($Q(\pi)=0$ only for unbiased selection). Consequently, when client participation is uneven, in this case due to biased sampling, the error of FedAvg is dominated by the non-vanishing  term $Q(\pi)$, while the dependency on $\zeta^2$ is  of order $1/K$. Conversely, all error terms in \eqref{eq:short_rate_strong_convex} vanish with $K$ regardless of the distribution of $p_i$ for $i \in [n]$, and the dependency on $\zeta$ is improved.  

Asynchronous SGD  exhibits a  similar limitation: under uneven client participation and non-IID data, a non-vanishing term dominates the convergence rate (e.g., [Theorem 4,\cite{mishchenko2022asynchronous}] for nonconvex functions and $\alpha_k = \Theta(1/\sqrt{K})$ for all $k$; to the best of our knowledge, no equivalent result has been reported for strongly convex, smooth functions), while the bound in \eqref{eq:short_rate_strong_convex} vanishes with $K$ for $M\geq 1$.

Finally, the asynchronous FL method \textsc{AsyncCommSGD} \cite{avdiukhin2021federated} attains a convergence rate of $\mathcal{O}\left( \sigma^2 /K + \zeta^2 \tau_{\max}^2/K^2\right)$ for strongly convex, smooth functions, where $\tau_{\max}$ is the maximum delay. Unlike \eqref{eq:short_rate_strong_convex} where the leading error term depends on $\pmi$, implying high sensitivity to stragglers, the dominant $1/K$ term  of \textsc{AsyncCommSGD} is independent of $\tau_{\max}$. In addition, \textsc{AsyncCommSGD} does not suffer from synchronization error.  However, \textsc{AsyncCommSGD} requires a global step size schedule coupled with the server's clock, and that clients reduce their local step sizes in accordance to this clock between server communications. This requirement can be challenging in practical FL deployment. By contrast, \textsc{AREA} couples the local and global clocks only at communication events: between communications with the server, clients train at their own pace using the step size last assigned to them, without tracing a global clock.

\subsection{Results for convex and nonsmooth functions}
We conclude with a result on the convergence of \textsc{AREA} under Assumptions \ref{assum:grad_stoch} and \ref{assum:bounded_g}. In this version of \textsc{AREA}, clients  have access only to stochastic approximations of the subgradients of their local functions. The corresponding proof can be found in Appendix \ref{sec:area_convex_proof}.
\begin{theorem}
\label{thm:conv_as} 
    Let $\bar{x}_{s,K} \triangleq \frac{1}{K}\sum_{k=0}^{K-1}x_{s,k}$ be the ergodic average of the global server iterates  $x_{s,k}$ of Algorithm~\ref{alg:area_server} after $K$ iterations, and $x^\star$ an arbitrary solution of Problem \ref{eq:prob_gen}. Suppose that Assumptions \ref{assum:grad_stoch} and~\ref{assum:bounded_g} hold, and that the step size sequence $\{\alpha_k\}$ in Algorithm \ref{alg:area_client} is   $\alpha =\sqrt{\left( \frac{1}{p_s} + 2\bar{q}\right) \|x_{s,0} - x^\star\|^2/\left(\left( \sigma^2 + \frac{(M+1) B^2}{2 }\right) MK\right)} $ for all $k$, where $\bar{q} \triangleq \frac{1}{n}\sum_{i=1}^n \frac{1}{p_i}$. Then the following inequality holds
\begin{align*}
E[f \left( \bar{x}_{s,K}  \right) - f^\star ]  &\leq \sqrt{\frac{\left( \frac{1}{p_s} + 2\bar{q}\right) \left( \sigma^2 + \frac{(M+1) B^2}{2 }\right) \|x_{s,0} - x^\star\|^2}{ MK} }.
\end{align*}
\end{theorem}

\paragraph{Comparison with existing methods.} To our knowledge, convergence guarantees under non-IID data and convex, Lipschitz objectives are only available for the averaged iterates of asynchronous \textsc{SGD} ($M=1$) with sampling proportional to client importance in the global objective as reported in \cite{even2024asynchronous}: $\mathcal{O}\left( \sqrt{(\sigma^2 + B^2)\|x_{s,0}- x^\star\|^2 / K \cdot n p_{\max}/\sqrt{\bar{p}} }\right)$, where $p_{\max}$ and $\bar{p}$ denote the maximum and average sampling probabilities, respectively. Conversely, \textsc{AREA} achieves an optimal $\mathcal{O}(1/K)$ convergence rate without client sampling and under arbitrary client participation probabilities $p_i$ for $i \in [n]$. The convergence rate of \textsc{AREA} does not include a factor of $\sqrt{n}$, and depends on the quantity $p_s^{-1} + 2 \bar{q} = p_s^{-1} + \tfrac{2}{n}\sum_{i=1}^n p_i^{-1}$ rather than the ratio $p_{\max}/\sqrt{\bar{p}}$, indicating increased robustness to outlier participation probabilities compared with asynchronous \textsc{SGD} with sampling. 
\paragraph{Optimal tuning of the server policy.} Under the Poisson-arrival model of Example \ref{rmk:poisson}, the aggregation frequency $\lambda_s^\star$ that minimizes the convergence rate bound reported in Theorem \ref{thm:conv_as} admits a closed-form solution. Substituting $p_s = \lambda_s/(\lambda_s + \sum_{j=1}^n \lambda_j )$ and  $ p_i = \lambda_i/(\lambda_s + \sum_{j=1}^n \lambda_j)$ for $i \in [n]$ into the bound of Theorem \ref{thm:conv_as}, differentiating with respect to $\lambda_s$ and setting the derivative to zero yields $\lambda_s^\star =   \sqrt{(\sum_{i=1}^n \lambda_i)/(\frac{2}{n}\sum_{i=1}^n  \lambda_i^{-1})} $. Notably, this tuning requires only the aggregate statistics $\sum_{i=1}^n \lambda_i$ and $\tfrac{1}{n}\sum_{i=1}^n \lambda_i^{-1}$ , rather than the individual client rates $\lambda_i$ for $i \in [n]$.

\section{Numerical Results}
\label{sec:log_reg}
We evaluate the empirical performance of $\textsc{AREA}$ on a multinomial $\ell_2$-regularized logistic regression task. Let $S_i$ denote the set of samples available to client $i \in [n]$ and define $s_i \triangleq |S_i|$. The empirical loss function at client $i$ can be expressed as
\begin{equation}
\label{eq:logreg}
    \textstyle  \hat{f}_i(w^1,\hdots,w^C) = \frac{1}{s_i}
    \sum_{j \in S_i} \sum_{l=1}^C \mathbbm{1}(y_j = l) \ln \left(  \frac{\exp{\langle w^l, x_j \rangle } }{\sum_{h=1}^C \exp{\langle w^h, x_j \rangle } }\right) + \frac{\nu}{2} \sum_{l=1}^C\|w^l\|^2,
\end{equation}
where $\mathbbm{1}(\cdot)$ is the indicator function, $(x_j,y_j)$ denote the input feature vector $x_j \in \R^d$ and the output label $y_j \in [C]$ of the $j^{th}$ data-point, $w^l \in \R^d$ for $l \in [C]$ are the model parameters associated with the $l^{th}$ label, and $\nu > 0$ is a regularization parameter.

Summing the preceding relation over clients $i \in [n]$ yields the global objective
\begin{align*}
 f(w^1,\hdots,w^C) = \sum_{i=1}^n \frac{s_i}{S} \hat{f}_i (w^1,\hdots,w^C), \quad \text{where}\quad  S = \sum_{i=1}^n s_i.
\end{align*}
\paragraph{Dataset and metrics.} We report results for the MNIST dataset \cite{deng2012mnist} divided over $n=128$ clients by sampling a Dirichlet distribution with parameter $a = 0.1$ to create a non-IID data split. We tracked two performance metrics over wall-clock time: $i$) the loss function value at the server model using all training data and $ii$) the test accuracy of the server model using all available test data.
\paragraph{Methods and aggregation criterion.} 
We compare with the following baseline methods: $i$) synchronous \textsc{FedAvg}  (\textsc{S-FedAvg}) \cite{fedavg}, $ii$) asynchronous  \textsc{FedAvg}  (\textsc{AS-FedAvg}), and $iii$) \textsc{FedBuff} \cite{fedbuff}. For \textsc{S-FedAvg} and \textsc{AS-FedAvg}, clients share their local model estimates, while for \textsc{FedBuff} clients share local pseudo-gradients. All asynchronous methods use the buffered aggregation policy of \cite{fedbuff}:  global updates are triggered by the reception of $\Delta=4$ client updates. For \textsc{S-FedAvg}, the server samples $4$ clients at each round with uniform probability. 
\paragraph{Asynchrony model.} The time $t_i$ between two consecutive updates from client $i \in [n]$ is modeled as an exponential random variable $t_i \sim \mathrm{Exp}(\lambda_i)$ (Example \ref{rmk:poisson}).  We tested two different cases: $i$) $\lambda_i = 10$ for $i \in [n]$ (uniform), and $ii$)  $\lambda_i \sim \mathcal{N}(10,5)$ (non-uniform). 
\paragraph{Hyper-parameters.} All methods use a mini-batch size $b=32$ and  regularization parameter $\nu = 10^{-3}$. A grid search over $\alpha \in \{10^{-2},10^{-1},10^0,10^1,10^2,10^3,10^4\}$ was performed for all methods. For each method we retain the two largest step sizes that yield finite loss function values at the end of the prescribed experimental horizon.   
 \paragraph{Implementation details.} We ran $10$ independent trials for every experimental configuration on a CPU cluster ($36$ nodes, $128$ GB DDR4 each) and extracted the mean, minimum and maximum metric values across trials. Supplemental experiments for non-IID data and $n=4$, and IID data for $n \in \{4,128\}$, can be found in Appendices \ref{sec:n4_appendix} and \ref{sec:IID_sim}, respectively.
\begin{table}[t] \tiny \setlength\tabcolsep{3pt} 
\centering 
\begin{tabularx}{\textwidth}{ C{1.1cm} | C{0.4cm} C{0.5cm} C{2.4cm} C{2.2cm} | C{0.4cm} C{0.5cm} C{2.4cm} C{2.2cm}}
\toprule
 & \multicolumn{4}{c|}{\textbf{Uniform}$\;(\lambda_i = 10)$} & \multicolumn{4}{c}{\textbf{Non-uniform}$\;(\lambda_i \sim \mathcal{N}(10,5))$}\\ \textbf{Method} & $\alpha$ & $\rho$ & Test acc \% ($t_h$) & Loss ($t_h$) & $\alpha$ & $\rho$ & Test acc \% ($t_h$) & Loss  ($t_h$)\\ 
\midrule \multirow{2}{*}{\textsc{S-FedAvg}}  & $10^1$ & $>1$ & $61.16$ ($52.75$ - $66.46$) & $1.56$ ($1.41$ - $1.78$) & $10^1$ & $>1$ & $42.23$ ($19.07$ - $55.31$) & $2.22$ ($1.72$ - $3.55$) \\
\multirow{2}{*}{\textsc{S-FedAvg}}   & $10^2$ & $0.80$ & $74.06$ ($41.21$ - $85.20$) & $1.69$ ($0.90$ - $5.55$) & $10^2$ & $>1$ & $60.36$ ($9.65$ - $78.08$) & $2.78$ ($1.16$ - $9.48$) \\
\midrule
\multirow{2}{*}{\textsc{AS-FedAvg}}   & $10^2$ & $0.28$ & $84.14$ ($79.15$ - $87.89$) & $0.89$ ($0.74$ - $1.09$) & $10^2$ & $0.28$ & $84.40$ ($70.04$ - $88.96$) & $0.89$ ($0.71$ - $1.56$) \\
\multirow{2}{*}{\textsc{AS-FedAvg}}   & $10^3$ & $0.30$ & $79.83$ ($66.79$ - $88.70$) & $8.19$ ($5.27$ - $11.91$) & $10^3$ & $0.33$ & $70.20$ ($20.91$ - $88.59$) & $16.29$ ($4.80$ - $76.90$) \\
\midrule
\multirow{2}{*}{\textsc{FedBuff}}   & $10^0$ & $0.46$ & $85.38$ ($84.74$ - $86.21$) & $0.81$ ($0.79$ - $0.82$) & $10^0$ & $0.48$ & $85.21$ ($83.82$ - $86.15$) & $0.81$ ($0.79$ - $0.85$) \\
\multirow{2}{*}{\textsc{FedBuff}}   & $10^1$ & \boldmath{$0.06$} & $78.56$ ($61.01$ - $89.78$) & $3.79$ ($1.74$ - $8.06$) & $10^1$ & \boldmath{$0.05$} & $78.57$ ($65.32$ - $88.88$) & $3.51$ ($1.66$ - $9.69$) \\
\midrule
\multirow{2}{*}{\textsc{AREA}}  & $10^2$ & $0.29$ & $87.38$ ($87.07$ - $87.92$) & $0.74$ ($0.73$ - $0.75$) & $10^2$ & $0.55$ & $84.37$ ($81.83$ - $86.11$) & $0.84$ ($0.79$ - $0.93$) \\
\multirow{2}{*}{\textsc{AREA}}   & $10^3$ & $0.09$ & \boldmath{$87.54$} ($79.38$ - $91.41$) & \boldmath{$0.59$} ($0.47$ - $0.91$)  & $10^3$ & $0.09$ & \boldmath{$87.53$} ($81.46$ - $90.43$) & \boldmath{$0.64$} ($0.55$ - $1.01$) \\
\bottomrule 
\end{tabularx}
\caption{Performance for $M=1$ local SGD step ($t_h=15$)}
\label{tab:M_1_n_128_non_iid}
\end{table}
\paragraph{One local SGD step.} Our results for $M=1$ local SGD step (asynchronous \textsc{SGD} and its variants) are summarized in Table \ref{tab:M_1_n_128_non_iid}. We report the mean test accuracy and loss at the horizon $t_h=15$ across $10$ independent trials, while the ranges indicate the minimum and maximum values observed across these trials. The convergence speed is quantified by $\rho=t_{80}/t_h$, where $t_{80}$ is wall-clock time required for the method to reach $80\%$ test accuracy ($\rho>1$ indicates the target was not met within $t_h$). For both uniform and non-uniform client update rates, \textsc{FedBuff} with $\alpha = 10^1$ reaches $80\%$ accuracy slightly earlier than \textsc{AREA}, but plateaus at $\sim 78\%$ final accuracy.  Conversely, \textsc{AREA} with $\alpha = 10^3$ exceeds all baselines by $2$-$3$\% in final accuracy and achieves the lowest loss. Consistent with our theoretical findings, \textsc{AREA} can accommodate large step sizes, which can yield fast convergence without affecting accuracy. An additional observation is that non-uniform client update rates adversely impact \textsc{S-FedAvg}, while the performance of asynchronous methods remains relatively unaffected at this level of rate heterogeneity.
\begin{table}[t] \tiny \setlength\tabcolsep{3pt} 
\centering 
\begin{tabularx}{\textwidth}{ C{1.1cm} | C{0.4cm} C{0.5cm} C{2.4cm} C{2.2cm} | C{0.4cm} C{0.5cm} C{2.4cm} C{2.2cm}}
\toprule
 & \multicolumn{4}{c|}{\textbf{Uniform}$\;(\lambda_i = 10)$} & \multicolumn{4}{c}{\textbf{Non-uniform}$\;(\lambda_i \sim \mathcal{N}(10,5))$}\\ \textbf{Method} & $\alpha$ & $\rho$ & Test acc \% ($t_h$) & Loss ($t_h$) & $\alpha$ & $\rho$ & Test acc \% ($t_h$) & Loss  ($t_h$)\\ 
\midrule \multirow{2}{*}{\textsc{S-FedAvg}}  & $10^{2}$ & $>1$ & $73.99$ ($45.06$ - $85.90$) & $1.17$ ($0.63$ - $3.08$) & $10^2$ & $>1$ & $59.69$ ($11.02$ - $81.14$) & $2.44$ ($0.84$ - $10.11$) \\
\multirow{2}{*}{\textsc{S-FedAvg}}  & $10^{3}$ & $>1$ & $53.45$ ($24.65$ - $77.09$) & $13.20$ ($3.48$ - $50.69$) 
& $10^3$ & $>1$ & $48.52$ ($9.71$ - $75.88$) & $11.12$ ($2.78$ - $34.85$) \\
\midrule
\multirow{2}{*}{\textsc{AS-FedAvg}}  & $10^{1}$ & $>1$ & $76.92$ ($66.63$ - $81.14$) & $1.00$ ($0.89$ - $1.24$)
 & $10^1$ & $>1$ & $77.17$ ($68.80$ - $80.59$) & $0.99$ ($0.93$ - $1.19$) \\
\multirow{2}{*}{\textsc{AS-FedAvg}}   & $10^{2}$ & $0.54$ & $79.66$ ($69.47$ - $87.49$) & $0.86$ ($0.54$ - $1.38$) 
& $10^2$ & $0.57$ & $80.98$ ($74.76$ - $86.09$) & $0.78$ ($0.56$ - $1.37$) \\
\midrule
\multirow{2}{*}{\textsc{FedBuff}}   & $10^{-1}$ & $>1$ & $40.22$ ($28.34$ - $47.01$) & $2.21$ ($2.00$ - $2.63$)
& $10^{-1}$ & $>1$ & $39.63$ ($32.13$ - $46.41$) & $2.25$ ($2.08$ - $2.50$) \\
\multirow{2}{*}{\textsc{FedBuff}}   & $10^{0}$ & $0.36$ & $84.16$ ($78.19$ - $89.38$) & $0.80$ ($0.63$ - $1.39$) 
& $10^0$ & $0.27$ & $85.12$ ($78.60$ - $89.44$) & $0.73$ ($0.64$ - $0.88$) \\
\midrule
\multirow{2}{*}{\textsc{AREA}}  & $10^{2}$ & $0.38$ & \boldmath{$88.02$} ($87.53$ - $88.32$) & \boldmath{$0.59$} ($0.58$ - $0.61$) & $10^2$ & $0.52$ & $86.03$ ($85.14$ - $87.24$) & \boldmath{$0.68$ ($0.65$ - $0.73$)}
\\
\multirow{2}{*}{\textsc{AREA}}  & $10^{3}$& \boldmath{$0.12$} & $87.29$ ($85.45$ - $88.11$) & $1.58$ ($1.49$ - $1.79$)& $10^3$ & \boldmath{$0.11$} & \boldmath{$87.28$ ($86.21$ - $87.99$)} & $1.44$ ($1.33$ - $1.55$) \\
\bottomrule 
\end{tabularx}
\caption{Performance for $M=50$ local SGD steps ($t_h=2$)}
\label{tab:M_50_n_128_non_iid}
\end{table}
 \paragraph{Fifty local SGD steps.} Table \ref{tab:M_50_n_128_non_iid} summarizes our results for $M=50$ local SGD steps (FL setting). The experiment horizon in this case was $t_h = 2$. \textsc{AREA} with $\alpha=10^3$ reaches $80\%$ test accuracy in less than half the wall-clock time required by baseline methods in all cases, with minimal effect on the final test accuracy but higher final loss; this is due to  client drift errors amplified by the large step size choice. \textsc{AREA} with $\alpha=10^2$ yields the lowest loss at the end of the experiment. Overall, our empirical results indicate that \textsc{AREA} combines fast learning with high final accuracy levels, while remaining robust to aggressive step size choices.


\section{Conclusion}
We proposed and analyzed \textsc{AREA}, a new asynchronous stochastic (sub)gradient method that provably corrects the bias induced by heterogeneous client update rates without server orchestration, and provides order-optimal convergence guarantees under heterogeneous and unbounded delays with finite mean. For strongly convex and smooth functions, we showed theoretically and empirically that  \textsc{AREA} admits larger step sizes compared to existing methods, which may yield faster convergence in practice without affecting the generalization error. For convex and nonsmooth functions, to our knowledge we are the first to recover order-optimal convergence rates under non-IID data without relying on client sampling techniques. Moreover, the recovered rate scales with the average participation probabilities instead of their minimum or maximum, indicating robustness to outliers, and we provide an explicit closed-form expression for the optimal global update frequency for our our scheme using only aggregate client statistics.

	\paragraph{Acknowledgements} This work was supported by the U.S. Department of Energy, Office of Science, Advanced Scientific Computing Research, under Contract DE-AC02-06CH11357. We gratefully acknowledge the computing resources provided on Bebop, a high-performance computing cluster operated by the Laboratory Computing Resource Center at Argonne National Laboratory.

\newpage

\bibliographystyle{plain}
\bibliography{refs}

\begin{thebibliography}{10}

\bibitem{avdiukhin2021federated}
Dmitrii Avdiukhin and Shiva Kasiviswanathan.
\newblock Federated learning under arbitrary communication patterns.
\newblock In {\em International Conference on Machine Learning}, pages
  425--435. PMLR, 2021.

\bibitem{bertsekas2015parallel}
Dimitri Bertsekas and John Tsitsiklis.
\newblock {\em Parallel and distributed computation: numerical methods}.
\newblock Athena Scientific, 2015.

\bibitem{secagg}
K.~A. Bonawitz, Vladimir Ivanov, Ben Kreuter, Antonio Marcedone, H.~Brendan
  McMahan, Sarvar Patel, Daniel Ramage, Aaron Segal, and Karn Seth.
\newblock Practical secure aggregation for federated learning on user-held
  data.
\newblock In {\em NIPS Workshop on Private Multi-Party Machine Learning}, 2016.

\bibitem{chen_fedsa_2021}
Ming Chen, Bingcheng Mao, and Tianyi Ma.
\newblock {FedSA}: {A} staleness-aware asynchronous {Federated} {Learning}
  algorithm with non-{IID} data.
\newblock {\em Future Generation Computer Systems}, 120:1--12, 2021.

\bibitem{chen2020vafl}
Tianyi Chen, Xiao Jin, Yuejiao Sun, and Wotao Yin.
\newblock Vafl: a method of vertical asynchronous federated learning.
\newblock {\em arXiv preprint arXiv:2007.06081}, 2020.

\bibitem{cho2020client}
Yae~Jee Cho, Jianyu Wang, and Gauri Joshi.
\newblock Client selection in federated learning: Convergence analysis and
  power-of-choice selection strategies.
\newblock {\em arXiv preprint arXiv:2010.01243}, 2020.

\bibitem{dean2012large}
Jeffrey Dean, Greg Corrado, Rajat Monga, Kai Chen, Matthieu Devin, Mark Mao,
  Marc'aurelio Ranzato, Andrew Senior, Paul Tucker, Ke~Yang, et~al.
\newblock Large scale distributed deep networks.
\newblock {\em Advances in Neural Information Processing Systems}, 25, 2012.

\bibitem{deng2012mnist}
Li~Deng.
\newblock The mnist database of handwritten digit images for machine learning
  research.
\newblock {\em IEEE Signal Processing Magazine}, 29(6):141--142, 2012.

\bibitem{even2024asynchronous}
Mathieu Even, Anastasia Koloskova, and Laurent Massouli{\'e}.
\newblock Asynchronous sgd on graphs: a unified framework for asynchronous
  decentralized and federated optimization.
\newblock In {\em International Conference on Artificial Intelligence and
  Statistics}, pages 64--72. PMLR, 2024.

\bibitem{feyzmahdavian2023asynchronous}
Hamid~Reza Feyzmahdavian and Mikael Johansson.
\newblock Asynchronous iterations in optimization: New sequence results and
  sharper algorithmic guarantees.
\newblock {\em J. Mach. Learn. Res.}, 24:158--1, 2023.

\bibitem{fraboni_general_theory}
Yann Fraboni, Richard Vidal, Laetitia Kameni, and Marco Lorenzi.
\newblock A general theory for federated optimization with asynchronous and
  heterogeneous clients updates.
\newblock {\em Journal of Machine Learning Research}, 24(110):1--43, 2023.

\bibitem{gu2021fast}
Xinran Gu, Kaixuan Huang, Jingzhao Zhang, and Longbo Huang.
\newblock Fast federated learning in the presence of arbitrary device
  unavailability.
\newblock {\em Advances in Neural Information Processing Systems},
  34:12052--12064, 2021.

\bibitem{scaffold}
Sai~Praneeth Karimireddy, Satyen Kale, Mehryar Mohri, Sashank Reddi, Sebastian
  Stich, and Ananda~Theertha Suresh.
\newblock Scaffold: Stochastic controlled averaging for federated learning.
\newblock In {\em International conference on machine learning}, pages
  5132--5143. PMLR, 2020.

\bibitem{koloskova2022sharper}
Anastasiia Koloskova, Sebastian~U Stich, and Martin Jaggi.
\newblock Sharper convergence guarantees for asynchronous {SGD} for distributed
  and federated learning.
\newblock {\em Advances in Neural Information Processing Systems},
  35:17202--17215, 2022.

\bibitem{konecny_federated_2015}
Jakub Konečný, Brendan McMahan, and Daniel Ramage.
\newblock Federated optimization:distributed optimization beyond the
  datacenter, 2015.
\newblock arXiv:1511.03575 [cs, math].

\bibitem{leconte_favas_2023}
Louis Leconte, Van~Minh Nguyen, and Eric Moulines.
\newblock {FAVAS}: Federated averaging with asynchronous clients, May 2023.
\newblock arXiv:2305.16099 [cs, stat].

\bibitem{li_fedcompass_2023}
Zilinghan Li, Pranshu Chaturvedi, Shilan He, Han Chen, Gagandeep Singh,
  Volodymyr Kindratenko, E.~A. Huerta, Kibaek Kim, and Ravi Madduri.
\newblock {FedCompass}: Efficient cross-silo federated learning on
  heterogeneous client devices using a computing power aware scheduler,
  September 2023.
\newblock arXiv:2309.14675 [cs].

\bibitem{lian_asynchronous_2019}
Xiangru Lian, Yijun Huang, Yuncheng Li, and Ji~Liu.
\newblock Asynchronous parallel stochastic gradient for nonconvex optimization.
\newblock In {\em Proceedings of the 28th International Conference on Neural
  Information Processing Systems - Volume 2}, NIPS'15, page 2737–2745,
  Cambridge, MA, USA, 2015. MIT Press.

\bibitem{mansoori_newton}
Fatemeh Mansoori and Ermin Wei.
\newblock A fast distributed asynchronous {Newton}-based optimization
  algorithm.
\newblock {\em IEEE Transactions on Automatic Control}, 65(7):2769--2784, 2020.

\bibitem{fedavg}
H.~B. McMahan, Eider Moore, Daniel Ramage, Seth Hampson, and Blaise~Ag{\"u}era
  y~Arcas.
\newblock Communication-efficient learning of deep networks from decentralized
  data.
\newblock In {\em International Conference on Artificial Intelligence and
  Statistics}, 2016.

\bibitem{mishchenko2022asynchronous}
Konstantin Mishchenko, Francis Bach, Mathieu Even, and Blake~E Woodworth.
\newblock Asynchronous {SGD} beats minibatch {SGD} under arbitrary delays.
\newblock {\em Advances in Neural Information Processing Systems}, 35:420--433,
  2022.

\bibitem{fedlin}
Aritra Mitra, Rayana~H. Jaafar, George~J. Pappas, and Hamed Hassani.
\newblock Linear convergence in federated learning: Tackling client
  heterogeneity and sparse gradients.
\newblock In {\em Neural Information Processing Systems}, 2021.

\bibitem{nesterov}
Yurii Nesterov.
\newblock Lectures on convex optimization, 2018.

\bibitem{fedbuff}
John Nguyen, Kshitiz Malik, Hongyuan Zhan, Ashkan Yousefpour, Mike Rabbat, Mani
  Malek, and Dzmitry Huba.
\newblock Federated learning with buffered asynchronous aggregation.
\newblock In {\em International Conference on Artificial Intelligence and
  Statistics}, pages 3581--3607. PMLR, 2022.

\bibitem{peng_arock_2016}
Zhimin Peng, Yangyang Xu, Ming Yan, and Wotao Yin.
\newblock {ARock}: an algorithmic framework for asynchronous \&parallel
  coordinate updates.
\newblock {\em SIAM Journal on Scientific Computing}, 38(5):A2851--A2879,
  January 2016.
\newblock arXiv:1506.02396 [cs, math, stat].

\bibitem{recht2011hogwild}
Benjamin Recht, Christopher Re, Stephen Wright, and Feng Niu.
\newblock Hogwild!: A lock-free approach to parallelizing stochastic gradient
  descent.
\newblock {\em Advances in Neural Information Processing Systems}, 24, 2011.

\bibitem{stich2021critical}
Sebastian Stich, Amirkeivan Mohtashami, and Martin Jaggi.
\newblock Critical parameters for scalable distributed learning with large
  batches and asynchronous updates.
\newblock In {\em International Conference on Artificial Intelligence and
  Statistics}, pages 4042--4050. PMLR, 2021.

\bibitem{stich_error-feedback_2021}
Sebastian~U. Stich and Sai~Praneeth Karimireddy.
\newblock The error-feedback framework: Better rates for {SGD} with delayed
  gradients and compressed updates.
\newblock {\em J. Mach. Learn. Res.}, 21(1), jan 2020.

\bibitem{stich_critical_2021}
Sebastian~U. Stich, Amirkeivan Mohtashami, and Martin Jaggi.
\newblock Critical parameters for scalable distributed learning with large
  batches and asynchronous updates, March 2021.
\newblock arXiv:2103.02351 [cs, stat].

\bibitem{ram_async_gossip}
S.~Sundhar~Ram, A.~Nedić, and V.~V. Veeravalli.
\newblock Asynchronous gossip algorithms for stochastic optimization.
\newblock In {\em Proceedings of the 48h IEEE Conference on Decision and
  Control (CDC) held jointly with 2009 28th Chinese Control Conference}, pages
  3581--3586, 2009.

\bibitem{toghani2022unbounded}
Mohammad~Taha Toghani and C{\'e}sar~A Uribe.
\newblock Unbounded gradients in federated learning with buffered asynchronous
  aggregation.
\newblock In {\em 2022 58th Annual Allerton Conference on Communication,
  Control, and Computing (Allerton)}, pages 1--8. IEEE, 2022.

\bibitem{tsitsiklis_distributed_1986}
J.~Tsitsiklis, D.~Bertsekas, and M.~Athans.
\newblock Distributed asynchronous deterministic and stochastic gradient
  optimization algorithms.
\newblock {\em IEEE Transactions on Automatic Control}, 31(9):803--812,
  September 1986.

\bibitem{wang2022asyncfeded}
Qiyuan Wang, Qianqian Yang, Shibo He, Zhiguo Shi, and Jiming Chen.
\newblock Asyncfeded: Asynchronous federated learning with euclidean distance
  based adaptive weight aggregation.
\newblock {\em arXiv preprint arXiv:2205.13797}, 2022.

\bibitem{wang_tackling_nodate}
Yujia Wang, Yuanpu Cao, Jingcheng Wu, Ruoyu Chen, and Jinghui Chen.
\newblock Tackling the data heterogeneity in asynchronous federated learning
  with cached update calibration.
\newblock In {\em The Twelfth International Conference on Learning
  Representations}, 2024.

\bibitem{fedasync}
Cong Xie, Oluwasanmi Koyejo, and Indranil Gupta.
\newblock Asynchronous federated optimization.
\newblock {\em ArXiv}, abs/1903.03934, 2019.

\bibitem{yang2022anarchic}
Haibo Yang, Xin Zhang, Prashant Khanduri, and Jia Liu.
\newblock Anarchic federated learning.
\newblock In {\em International Conference on Machine Learning}, pages
  25331--25363. PMLR, 2022.

\bibitem{zakerinia_communication-efficient_2023}
Hossein Zakerinia, Shayan Talaei, Giorgi Nadiradze, and Dan Alistarh.
\newblock Communication-efficient federated learning with data and client
  heterogeneity, 2023.
\newblock arXiv:2206.10032 [cs].

\end{thebibliography}

\appendix

\label{appendix}

\section{Derivation of iterates}
\label{sec:iterate_derivation}
Define the Bernoulli random variables $\phi_{s,k}$ (resp. $\phi_{i,k}$) taking the value $1$ if $J_k = s$ (resp. $J_k = i$) at the $k^{th}$ iteration and $0$ otherwise. Then for all $k$ we have
\begin{align*}
    E[\phi_{s,k}|\F_{k-1}] = p_s, \quad \E[\phi_{i,k}|\F_{k-1}] = p_i, \quad i \in [n].
\end{align*}
We also define the stochastic (sub)gradient operator $G_{i,\alpha}[x] \triangleq x - \alpha g_i(x)$ for $i \in [n]$, and the composition $G_{i,\alpha}^m$ of $m \in [M]$ applications of $G_{i,\alpha}$ 
\begin{align*}
    G_{i,\alpha}^m[x] \triangleq G_{i,\alpha}[ G_{i,\alpha}^{m-1}[x]], \quad  G_{i,\alpha}^0[x] \triangleq x.
\end{align*}
Using the notation above, the iterates of \textsc{AREA} can be expressed as
\begin{align*}
      x_{i,k+1}^M &=  x_{i,k}^M + \phi_{i,k+1} \left( G_{i,\alpha_k}[x_{s,k}] - x_{i,k}^M \right), \quad \forall i \in [n],\\
    y_{i,k+1} &=  y_{i,k} + \phi_{i,k+1}(x_{i,k}^M - y_{i,k} ), \quad \forall i \in [n],\\
 u_{s,k+1} &= \textstyle \sum_{i=1}^n \phi_{i,k+1} \left( u_{s,k} + \frac{1}{n}(x_{i,k}^M - y_{i,k})\right),\\
    x_{s,k+1} &=  x_{s,k} + \phi_{s,k+1} u_{s,k}.
\end{align*}
By construction, only a single variable in the set $\{\phi_{1,k},...,\phi_{n,k},\phi_{s,k}\}$ is equal to $1$ and the rest are $0$ for all $k$. To further simplify the above iterates, let us define the variables $z_{i,k} \triangleq x_{s,k} + u_{s,k} - y_{i,k}$ for $i \in [n]$, and recall that $u_{s,0} = 0$ and $y_{i,0} = x_{s,0}$ for $i \in [n]$. We then have 
\begin{align*}
    \bar{z}_{k+1} =  x_{s,k+1} + u_{k+1} - \bar{y}_{k+1}= x_{s,k} + \left(\phi_{s,k+1} + \sum_{i=1}^n \phi_{i,k+1} \right) u_{s,k} - \bar{y}_k = \bar{z}_k.
\end{align*}
Hence, $\bar{z}_k = \bar{z}_0 = 0$ for all $k$, and we can eliminate the aggregator variable by setting $  u_{s,k} = \bar{y}_k - x_{s,k}$. The iterates of \textsc{AREA} become (we drop the superscript in $x_i^M$ for convenience)
\begin{align}
    x_{i,k+1} &=  x_{i,k} + \phi_{i,k+1} \left( G_{i,\alpha_k}[x_{s,k}] - x_{i,k}\right), \quad \forall i \in [n], \label{eq:x_k}\\
    y_{i,k+1} &= y_{i,k} + \phi_{i,k+1}(x_{i,k} - y_{i,k}), \quad \forall i \in [n],  \label{eq:y_k}\\
    x_{s,k+1} &= x_{s,k} + \phi_{s,k+1}(\bar{y}_k  - x_{s,k})\label{eq:xs_k}.
\end{align}

\section{Proof of Theorem \ref{thm:convergence_strong_convex}}
\label{sec:proof_of_thm_8}
In this Appendix, we derive the convergence properties of \textsc{AREA} under Assumptions \ref{assum:grad_stoch}, \ref{assum:smooth} and \ref{assum:strong_convex}. Under Assumptions  \ref{assum:smooth} and \ref{assum:strong_convex}, Problem~\eqref{eq:prob_gen} has a unique solution $x^\star = \arg \min_{x \in \R^d} f(x)$. We define the points $x_{i,k}^\star \in \R^d$ for $i \in [n]$ and $k=0,1,2,...$ 
\begin{equation}
\label{eq:xi_star}
    x_{i,k}^\star \triangleq x^\star - \alpha_k M \nabla f_i(x^\star),
\end{equation}
where $M$ is the number of local SGD steps.

Let us also define the transition matrix $A_k \in \R^{(2n+1)^2}$ which will play a key role in our analysis
\begin{align}
\label{eq:transition_matrix}
    A_k \triangleq \begin{bmatrix}
        I-P & 0 & \left(1 - \frac{\alpha_k M \gamma}{4} \right)  P1_n\\
        P & I-P & 0 \\
        0 & \frac{p_s}{n}1_n' & 1-p_s \end{bmatrix},
\end{align}
where $\{\alpha_k\}$ is the step size sequence in \ref{alg:area_client}, $\gamma = \frac{2\mu L}{\mu + L}$, $L$ and $\mu$ are the strong convexity and Lipschitz constants defined in Assumptions \ref{assum:smooth} and \ref{assum:strong_convex}, respectively, and $P \in \R^{n^2}$ is a diagonal matrix with $[P]_{ii}=p_i$.

Finally, we list a number of standard inequalities that we will invoke throughout  our analysis. The Bernoulli inequalities  below hold for any $x > -1$
\begin{equation}
\label{eq:standard_ineq_power}
\begin{split}
    (1+x)^r &\leq 1 + rx, \quad \text{for} \quad r \in [0,1]\\
    (1+x)^r &\geq 1 + rx, \quad \text{for} \quad r \in  (-\infty,0) \cup (1, \infty).
\end{split}
\end{equation}
The following inequality holds for any vectors $a$ and $b$ and scalar $c>0$ 
\begin{align}
    \label{eq:young_ineq}
    \pm 2 \langle a, b\rangle \leq c\|a\|^2 + \frac{1}{c}\|b\|^2.
\end{align}

\subsection{Preliminaries}
We now list a number of preliminary lemmas and theorems that we will employ in the proof of Theorem \ref{thm:convergence_strong_convex}. The following lemma has been adapted from [Theorem 2.1.14, \cite{nesterov}] and implies that the standard gradient method can be viewed as a contractive mapping.
\begin{lemma} 
\label{lem:nesterov}
Let $f :\R^n \rightarrow\R$ be a $\mu$-strongly convex function with $L$-Lipschitz gradients. Then for any $x,y \in \R^n$ and $\alpha < \frac{2}{\mu + L}$ the following relation holds:
\begin{equation*}
    \begin{split}
        \|x - \alpha \nabla f(x) - (y - \alpha \nabla f(y))\|^2 &\leq (1-\alpha \gamma) \|x-y\|^2,
    \end{split}
\end{equation*}
where $\gamma = \frac{2\mu L}{\mu + L}$.
\end{lemma}
Next, we outline some useful properties of the transition matrix $A_k$, defined in \eqref{eq:transition_matrix}.
\begin{lemma}
\label{lem:properties_of_Ak}Define $\beta_k \triangleq  1 - \frac{\alpha_k M \gamma}{4}$ and suppose that  $\alpha_k < \frac{4}{M\gamma} $ in \eqref{eq:transition_matrix}. Moreover, define the sequence of vectors $\{w_k\} \in \R^{2n+1}$ as follows
\begin{align}
\label{eq:w_k_definition}
    w_k' &\triangleq \begin{bmatrix}
\frac{1}{n}\frac{1}{ \beta_k^{2/3}}1_n'P^{-1} &\frac{1}{n}\frac{1}{ \beta_k^{1/3}}1_n'P^{-1} &
    \frac{1}{p_s}
    \end{bmatrix}.
\end{align}
The following statements are true for $k=0,1,2,...$
\begin{enumerate}
\item The vector $w_k$ satisfies the inequalities below
\begin{align}
\label{eq:eigenvector_products}
\sum_{i=1}^n \left([w_k]_i + [w_k]_{i+n} \right)  \leq \frac{2\bar{q}}{\beta_k}, \quad \sum_{i=1}^n p_i [w_k]_i   = \frac{1}{\beta_k},
\end{align}
where $\bar{q} \triangleq \frac{1}{n}\sum_{i=1}^n \frac{1}{p_i}$.
    \item Define $\pmi \triangleq \min\{ \min_i p_i, p_s\}$ and $\rho_k \triangleq 1-\pmi + \pmi\beta_k^{1/3} $. Then
        \begin{align}
\label{eq:rho_k_bound}
    w_k' A_k \leq \rho_k w_k'.
\end{align}
\end{enumerate}
\end{lemma}
\begin{proof} The first claim of this lemma holds trivially after observing that $\beta_k^{1/3} \leq \beta_k^{2/3} \leq \beta_k$. To prove the second claim, define $v_k' \triangleq w_k' A_k \in \R^{2n+1}$. The elements of $v_k$ are
\begin{align*}
    [v_k]_i &= \frac{1}{n} \left( \frac{1-p_i}{ p_i \beta_k^{2/3}} + \frac{1}{\beta_k^{1/3}}\right) = [w_k]_i\left( 1 -p_i + p_i \beta^{1/3}_k \right)  , \quad &i \in [n]\\
    [v_k]_{i+n} &= \frac{1}{n} \frac{1-p_i}{p_i \beta_k^{1/3}} + \frac{1}{n} =[w_k]_i \left( 1 - p_i + p_i \beta_k^{1/3}\right)  , \quad &i \in [n]\\
    [v_k]_{2n+1} &= \beta_k^{1-2/3} + \frac{1-p_s}{p_s} = [w_k]_{2n+1} \left( 1 - p_s +  p_s\beta_k^{1/3} \right).
\end{align*}
Hence, for $i \in [2n+1]$ the vectors $v_k$ and $w_k$ satisfy
\begin{align}
   [v_k]_i \leq \max \left\{1 - p_i +  p_i\beta_k^{1/3} \quad i \in [n], \quad  1 - p_s +  p_s\beta_k^{1/3} \right\} [w_k]_i.
\end{align}
Applying the definition of $\pmi$ to the preceding relation concludes the proof.
\end{proof}
In the next lemma, we define an appropriate decreasing step size sequence $\{\alpha_k\}$ for Algorithm \ref{alg:area_client} under Assumptions \ref{assum:grad_stoch},  \ref{assum:smooth} and \ref{assum:strong_convex}.
\begin{lemma}
\label{lem:step size_sequence} For $p \in (0,1)$, consider the sequence $\{\alpha_k\}$ defined as follows
    \begin{equation}
    \label{eq:step size_definition}
        \begin{gathered}
           \alpha_k \triangleq \frac{c}{k+d}, \quad c \triangleq \frac{48}{p M \gamma } , \quad  d\triangleq \frac{48}{p M \gamma \mathcal{D}},\\
       \mathcal{D} \triangleq \min\left\{ \frac{2}{M(\mu + L)}, \frac{\gamma}{L^2(M-1)}, \sqrt{\frac{ M \gamma^2}{64L^4 (M-1)^3 }}, \sqrt[3]{\frac{ \gamma}{32 L^4 (M-1)^3 }} \right\},
        \end{gathered}
    \end{equation}
where $M$ is the number of local SGD steps, $\gamma = \frac{2\mu L}{\mu + L}$, and $L$ and $\mu$ are defined in Assumptions   \ref{assum:smooth} and \ref{assum:strong_convex}, respectively.

The following statements hold:
\begin{enumerate}
    \item The sequence $\{\alpha_k\}$ is upper bounded by $\mathcal{D}$.
    \item Let $\beta_k \triangleq 1 - \frac{\alpha_k M \gamma}{4}$. The sequence $\{\beta_k\}$ satisfies the inequalities below
 \begin{gather}
     \beta_k \geq \frac{1}{2}. \label{eq:step_lower_bound}\\
     \left( 1-p + p \beta_k ^{1/3} \right) (k+d)^4 \leq (k+d-1)^4 \label{eq:spectral_radius_monotonicity}.
 \end{gather}
\end{enumerate}
\end{lemma}
\begin{proof}
To prove the first claim of this lemma, we substitute the definitions of $c$ and $d$ in $\alpha_k$ to obtain
\begin{align*}
    \alpha_k &= \frac{48}{p M \gamma\left( k + \frac{48}{p M \gamma \mathcal{D}}\right) } = \left(\frac{1}{1 +  \frac{p M \gamma \mathcal{D} k}{48}  } \right)\mathcal{D} \leq \mathcal{D}.
\end{align*}
To prove \eqref{eq:step_lower_bound}, recall that  $\mathcal{D} \leq \frac{2}{M(\mu + L)}$ and $\gamma \triangleq \frac{2\mu L}{\mu + L}$ by definition. Hence, for the constant $d$ we have
\begin{align*}
    d=\frac{48}{p M \gamma \mathcal{D}} \geq  \frac{24}{p   } \frac{ (\mu + L)^2}{2\mu L} \geq \frac{24}{p}.
\end{align*}
Since $\{\alpha_k\}$ is strictly decreasing for $k \geq 0$, we have $\min_k \beta_k = \beta_0$. Substituting the definition of $\alpha_k$ in $\beta_k$ and applying the preceding relation yields 
\begin{align*}
    \beta_0 = 1 - \frac{12}{p d} \geq \frac{1}{2}.
\end{align*}

We now prove \eqref{eq:spectral_radius_monotonicity}. Let $\rho \triangleq   1-p + p \beta_k^{1/3}$. Applying the definition of $\alpha_k$  yields
\begin{align*}
    \rho = 1 - p + p \left(1  - \frac{12}{p(k+d)}\right)^{1/3}\leq 1 - \frac{4}{k+d},
\end{align*}
where the last inequality follows from \eqref{eq:standard_ineq_power}.

We multiply both sides of the relation above with $(k+d)^4$ to obtain
\begin{align*}
    (k+d)^4 \rho &\leq (k+d)^4 \left(1 - \frac{4}{k+d} \right) \leq  (k+d)^4\left(1 - \frac{1}{k+d} \right)^4,
\end{align*}
where the last inequality follows again from \eqref{eq:standard_ineq_power}.

Observing that $1 - \frac{1}{k+d} = \frac{k+d-1}{k+d}$ completes the proof.
\end{proof}

\subsection{Convergence of norms} 
\label{sec:convergence_norms}
Next, we derive an intermediate result on the convergence of the expected $\ell_2$-norms of the distances to optimality of the local and global iterates of \textsc{AREA}. 

For ease of reference, we begin by defining a number of quantities that we will invoke several times throughout this proof. We first define a global upper bound for the distances between $x^\star$ and the minima of the local functions $f_i$ for $i \in [n]$:
\begin{align}
\label{eq:zeta_definition}
    u_i^\star \triangleq \min_x f_i(x), \quad \zeta \triangleq \max_i \|x^\star - u_i^\star\|.
\end{align}
Note that we can re-write $  G_{i,\alpha_{k}}[x_{s,k}]$ in \eqref{eq:x_k} as follows
\begin{equation}
\label{eq:tilde_g}
    \begin{gathered}
         G_{i,\alpha_{k}}[x_{s,k}] = x_{s,k} - \alpha_k M \nabla f_i(x_{s,k}) + \alpha_k\epsilon^d_{i,k}  + \alpha_k\epsilon^d_{i,k}, \quad \text{where}\\
   \epsilon^d_{i,k} \triangleq   \sum_{m=1}^{M-1} \left(\nabla f_i(x_{s,k}) -  \nabla f_i \left( G_{i,\alpha_k}^{m} [x_{s,k}] \right) \right), \\\epsilon^g_{i,k} \triangleq  \sum_{m=0}^{M-1} \left( \nabla f_i \left( G_{i,\alpha_k}^{m-1} [x_{s,k}]\right) -g_i \left( G_{i,\alpha_k}^{m-1} [x_{s,k}]\right)\right) .
    \end{gathered}
\end{equation}
In the next lemma, we bound the $\ell_2$ norms of the stochastic gradient and client drift errors in an iteration of \textsc{AREA}.
\begin{lemma}
\label{lem:errors_norms}
The expected norm of the stochastic gradient error $\epsilon_{i,k}^g$ defined in \eqref{eq:tilde_g} is bounded for all $i \in [n]$ and $k=0,1,2,...$
\begin{align*}
    \E[\|\epsilon_{i,k}^g\||\F_k] \leq M \sigma,
\end{align*}
where $\sigma$ is defined in Assumption \ref{assum:grad_stoch}.

Moreover, if $\alpha_k < \frac{2}{\mu + L}$ in \eqref{eq:x_k}, the expected norm of the client drift error defined in \eqref{eq:tilde_g} is also bounded for all $i \in [n]$ and $k=1,2,...$
\begin{align*}
    \E[\|\epsilon_{i,k}^d\||\F_k] 
   &\leq \frac{\alpha_k L^2 M (M-1) \|x_{s,k} - x^\star\| }{4} + \frac{\alpha_k L^2 M (M-1)\zeta }{2} + \frac{\alpha_k L M (M-1) \sigma }{4},
\end{align*}
where $\zeta >0 $ is defined in \eqref{eq:zeta_definition}.
\end{lemma}
\begin{proof} Taking the norm of $\epsilon_{i,k}^g$ and applying the triangle inequality yields
\begin{align*}
    \|\epsilon_{i,k}^g\| \leq \sum_{m=0}^{M-1} \left\|\nabla f_i \left( G_{i,\alpha_k}^{m-1} [x_{s,k}]\right) -g_i \left( G_{i,\alpha_k}^{m-1} [x_{s,k}]\right)\right\|.
\end{align*}
By linearity, taking the expectation conditional on $\F_k$ on both sides of the relation above yields
\begin{align*}
    \E[\|\epsilon_{i,k}^g\| |\F_k]&\leq \sum_{m=0}^{M-1} \E[\left\|\nabla f_i \left( G_{i,\alpha_k}^{m-1} [x_{s,k}]\right) -g_i \left( G_{i,\alpha_k}^{m-1} [x_{s,k}]\right)\right\||\F_k]\\
    &\leq \sum_{m=0}^{M-1} \sqrt{\E[\left\|\nabla f_i \left( G_{i,\alpha_k}^{m-1} [x_{s,k}]\right) -g_i \left( G_{i,\alpha_k}^{m-1} [x_{s,k}]\right)\right\|^2|\F_k]} \leq M \sigma,
\end{align*}
where the second inequality follows from Jensen's inequality and the concavity of the square root function, and the last inequality from Assumption \eqref{assum:grad_stoch}.

We now prove the second claim of this lemma. Taking the $\ell_2$ norm of $\epsilon_{i,k}^d$ and applying the triangle inequality yields
\begin{align*}
  \|\epsilon_{i,k}^d\|
    &\leq  \sum_{m=1}^{M-1}  \left\| \nabla f_i(x_{s,k}) -  \nabla f_i \left( G_{i,\alpha_k}^{m} [x_{s,k}] \right) \right\|.
\end{align*}
Define $\delta^m_{i,k} \triangleq G_{i,\alpha_k}^{m} [x_{s,k}]  - x_{s,k}$. We take the expectation conditional on $\F_k$ on both sides of the relation above to obtain
\begin{align}
\label{eq:norm_drift_starter}
    \E[\|\epsilon_{i,k}^d\||\F_k] 
    &\leq  \sum_{m=1}^{M-1}  \E\left[\left\| \nabla f_i(x_{s,k}) -  \nabla f_i \left( G_{i,\alpha_k}^{m} [x_{s,k}] \right) \right\||\F_k \right]\leq L \sum_{m=1}^{M-1}  \E\left[\left\| \delta^m_{i,k}    \right\||\F_k \right],
\end{align}
where the last inequality follows from Assumption \ref{assum:smooth}.

For any $m \in [M]$ we have
\begin{align*}
     \E\left[\left\|\delta^m_{i,k}    \right\||\F_k \right] &=  \E\left[\left\|   G_{i,\alpha_k}^{m-1} [x_{s,k}] - \alpha_k g_i \left(G_{i,\alpha_k}^{m-1} [x_{s,k}] \right)  - x_{s,k}  \right\||\F_k \right]\\
     &\leq \E\left[\left\|   G_{i,\alpha_k}^{m-1} [x_{s,k}] - \alpha_k \nabla f_i \left(G_{i,\alpha_k}^{m-1} [x_{s,k}] \right)  - x_{s,k}  \right\||\F_k \right] \\
     &\quad + \alpha_k \E[\|\nabla f_i \left(G_{i,\alpha_k}^{m-1} [x_{s,k}] \right) - g_i \left(G_{i,\alpha_k}^{m-1} [x_{s,k}] \right)\||\F_k]\\
     &\leq  \E\left[\left\|   G_{i,\alpha_k}^{m-1} [x_{s,k}] - \alpha_k \nabla f_i \left(G_{i,\alpha_k}^{m-1} [x_{s,k}] \right)  - x_{s,k}   + \alpha_k \nabla f_i(x_{s,k})\right\||\F_k \right]  \\
     &\quad + \alpha_k\|\nabla f_i(x_{s,k})\|+ \alpha_k \E[\|\nabla f_i \left(G_{i,\alpha_k}^{m-1} [x_{s,k}] \right) - g_i \left(G_{i,\alpha_k}^{m-1} [x_{s,k}] \right)\||\F_k],
\end{align*}
where the second inequality follows from adding and subtracting $\alpha_k\nabla f_i \left(G_{i,\alpha_k}^{m-1} [x_{s,k}] \right)$ inside the norm and applying the triangle inequality, and the second inequality from from adding and subtracting $\alpha_k\nabla f_i \left(x_{s,k}\right)$ inside the first norm and applying the triangle inequality.

By Jensen's inequality and the concavity of the square root function, the following inequality holds
\begin{align*}
     &\E[\|\nabla f_i \left(G_{i,\alpha_k}^{m-1} [x_{s,k}] \right) - g_i \left(G_{i,\alpha_k}^{m-1} [x_{s,k}] \right)\||\F_k]\\ 
     &\leq \sqrt{\E[\|\nabla f_i \left(G_{i,\alpha_k}^{m-1} [x_{s,k}] \right) - g_i \left(G_{i,\alpha_k}^{m-1} [x_{s,k}] \right)\|^2|\F_k]} \leq \sigma.
\end{align*}
We combine the two preceding relations to obtain
\begin{align*}
     \E\left[\left\|\delta^m_{i,k}    \right\||\F_k \right]
     &\leq  \E\left[\left\|   G_{i,\alpha_k}^{m-1} [x_{s,k}] - \alpha_k \nabla f_i \left(G_{i,\alpha_k}^{m-1} [x_{s,k}] \right)  - x_{s,k}   + \alpha_k \nabla f_i(x_{s,k})\right\||\F_k \right]  \\
     &\quad + \alpha_k\|\nabla f_i(x_{s,k})\|+ \alpha_k\sigma\\
     &\leq  \sqrt{1-\alpha_k \gamma} \E\left[\left\|  \delta^{m-1}_{i,k}\right\||\F_k \right]  + \alpha_k\|\nabla f_i(x_{s,k})\|+ \alpha_k \sigma\\
      &\leq \left(1-\frac{\alpha_k \gamma }{2}\right) \E\left[\left\|  \delta^{m-1}_{i,k}\right\||\F_k \right]  + \alpha_k\|\nabla f_i(x_{s,k})\|+ \alpha_k \sigma,
\end{align*}
where the second inequality follows from Lemma \ref{lem:nesterov} and the last inequality from \eqref{eq:standard_ineq_power}.

Applying the preceding relation recursively over local steps $0,...,m$ yields
\begin{align*}
     \E\left[\left\|\delta^m_{i,k}    \right\||\F_k \right]&\leq \left(1-\frac{\alpha_k \gamma }{2}\right) ^{m}\left\|  \delta^{0}_{i,k}\right\|  + (\alpha_k\|\nabla f_i(x_{s,k})\|+ \alpha_k \sigma) \sum_{l=0}^{m-1} \left(1-\frac{\alpha_k \gamma }{2}\right)^l\\
&=  (\alpha_k\|\nabla f_i(x_{s,k})\|+ \alpha_k \sigma) \frac{1 - \left(1-\frac{\alpha_k \gamma }{2}\right)^m}{\alpha_k \gamma}  \leq \frac{\alpha_k m}{2} \|\nabla f_i(x_{s,k})\|+ \frac{\alpha_k m }{2} \sigma,
\end{align*}
where the last inequality follows from \eqref{eq:standard_ineq_power}.

We substitute the bound above in \eqref{eq:norm_drift_starter} to obtain
\begin{align*}
    \E[\|\epsilon_{i,k}^d\||\F_k] 
    &\leq L \left( \frac{\alpha_k }{2} \|\nabla f_i(x_{s,k})\|+ \frac{\alpha_k  }{2} \sigma\right)  \sum_{m=1}^{M-1} m\\
    &= \frac{\alpha_k L M (M-1) \|\nabla f_i(x_{s,k})\| }{4} + \frac{\alpha_k L M (M-1) \sigma }{4}\\
   &\leq \frac{\alpha_k L^2 M (M-1) \|x_{s,k} - u_i^\star\| }{4} + \frac{\alpha_k L M (M-1) \sigma }{2},
\end{align*}
where the last inequality follows from Assumption \eqref{assum:smooth}.

Adding and subtracting $x^\star$ inside the norm in the first term on the right-hand side of the preceding relation and applying the triangle inequality yields
\begin{align*}
    \|x_{s,k} - u_i^\star\| &\leq \|x_{s
    ,k} - x^\star\| + \|x^\star-u_i^\star\|.
    \end{align*}
Combining the two preceding relations and applying the definition of $\zeta$ concludes the proof.
\end{proof}
In the next three lemmas, we individually bound the expected $\ell_2$ norms of the distances to optimality of the global and local iterates of \textsc{AREA}, starting with the server iterates in the lemma below.

\begin{lemma} The expected $\ell_2$ distance to optimality of the server iterates generated by \eqref{eq:xs_k} is bounded for $k=0,1,2,...$
\label{lem:s_norm}
    \begin{align*}
    \E[\|x_{s,k+1} - x^\star\|] 
&\leq (1-p_s)\E[\|x_{s,k} - x^\star\||\F_k]+ \frac{p_s}{n}\sum_{i=1}^n\E[\|y_{i,k} - x_{i,k}^\star\|],
\end{align*}
where $x_{i,k}^\star$ is defined in \eqref{eq:xi_star}.
\end{lemma}
\begin{proof}
    Subtracting $x^\star$ on both sides of \eqref{eq:xs_k} and taking the norm yields
\begin{align*}
    \|x_{s,k+1} - x^\star\| &= (1-\phi_{s,k+1})\|x_{s,k} - x^\star\| + \phi_{s,k+1} \|\bar{y}_k - x^\star\|.
\end{align*}
Taking the expectation conditional on $\F_k$ on both sides of the relation above yields
\begin{align*}
    \E[\|x_{s,k+1} - x^\star\||\F_k] &= 
     (1-p_s)\|x_{s,k} - x^\star\| + p_s \|\bar{y}_k - x^\star\|\\
     &=   (1-p_s)\|x_{s,k} - x^\star\| + p_s \left\|\bar{y}_k - \alpha_k M \nabla f(x) - x^\star \right\|\\
&\leq (1-p_s)\|x_{s,k} - x^\star\|+ \frac{p_s}{n}\sum_{i=1}^n \|y_{i,k} - x_{i,k}^\star\|,
\end{align*}
where the first equality follows from the optimality of $x^\star$, and the last inequality from application of the triangle inequality and the definitions of $x_{i,k}^\star$ for $i \in [n]$.

Taking the total expectation on both sides of the preceding relation completes the proof.
\end{proof}
In the next lemma, we bound the $\ell_2$ distance to the points $x_{i,k}^\star$ for the local memory iterates of \textsc{AREA}.
\begin{lemma} Recall the definition of $x_{i,k}^\star \triangleq x^\star - \alpha_k M \nabla f_i(x^\star)$ for $i \in [n]$ where $M$ is the number of local SGD steps, and suppose that $\alpha_{k+1} \leq \alpha_k$. The expected $\ell_2$ distances to the points $x_{i,k}^\star$ of the local memory iterates $\{y_{i,k}\}$ generated by \eqref{eq:y_k} are bounded for $i \in [n]$ and $k=0,1,2,...$ 
    \label{lem:y_norm}
    \begin{align*}
    \E[\|y_{i,k+1} - x_{i,k+1}^\star\||] 
        &\leq (1-p_i)\E[\|y_{i,k} - x_{i,k}^\star\|] + p_i\E[\|x_{i,k} - x_{i,k}^\star\|] \\
        &\quad + (\alpha_k - \alpha_{k+1})LM\zeta,
\end{align*}
where $\zeta$ is defined in \eqref{eq:zeta_definition} and $L$ in Assumption \ref{assum:smooth}.
\end{lemma}
\begin{proof}
    Subtracting $x_{i,k+1}^\star$ from both sides of \eqref{eq:y_k} and taking the norm yields
\begin{align*}
    \|y_{i,k+1} - x_{i,k+1}^\star\| &= (1-\phi_{i,k+1})\|y_{i,k} - x_{i,k+1}^\star\| +  \phi_{i,k+1}\|x_{i,k} - x_{i,k+1}^\star\|\\
    &\leq (1-\phi_{i,k+1})\|y_{i,k} - x_{i,k}^\star\| +  (1-\phi_{i,k+1})\| x_{i,k+1}^\star - x_{i,k}^\star\| \\
    &\quad +  \phi_{i,k+1}\|x_{i,k} - x_{i,k}^\star\|+ \phi_{i,k+1}\|x_{i,k+1}^\star - x_{i,k}^\star\|\\
    &\leq (1-\phi_{i,k+1})\|y_{i,k} - x_{i,k}^\star\|  +  \phi_{i,k+1}\|x_{i,k} - x_{i,k+1}^\star\| \\
    &\quad +  (\alpha_k - \alpha_{k+1})M\|\nabla f_i(x^\star)\|,
\end{align*}
where we obtain the second inequality by applying the triangle inequality, and the last inequality by applying the definitions of $x_{i,k}^\star$ and $x_{i,k+1}^\star$.

We take the expectation conditional on $\F_k$ on both sides of the relation above to obtain
\begin{align*}
    \E[\|y_{i,k+1} - x_{i,k+1}^\star\||\F_k] 
    &\leq (1-p_i)\|y_{i,k} - x_{i,k}^\star\| + p_i\|x_{i,k} - x_{i,k}^\star\| + (\alpha_k - \alpha_{k+1})M\|\nabla f_i(x^\star)\|\\
        &\leq (1-p_i)\|y_{i,k} - x_{i,k}^\star\| + p_i\|x_{i,k} - x_{i,k}^\star\| + (\alpha_k - \alpha_{k+1})LM\zeta,
\end{align*}
where we applied Assumption \ref{assum:smooth} and the definition of $\zeta$ to obtain the last inequality.

Taking the total expectation on both sides of the relation above completes the proof.
\end{proof}
In the next lemma, we bound the distance to optimality of the local model iterates of \textsc{AREA}.

\begin{lemma}
    \label{lem:x_norm} Suppose that the sequence $\{\alpha_k\}$ in \eqref{eq:x_k} satisfies the following relations for $k=0,1,2,...$
    \begin{align*}
        \alpha_{k+1} \leq \alpha_k, \quad \alpha_k < \min\left\{\frac{\gamma}{L^2 (M-1)}, \frac{2}{M (\mu + L)}\right\},
    \end{align*}
 where $M$ is the number of local SGD steps, $\gamma = \frac{2\mu L}{\mu + L}$ and $L$ and $\mu$ are defined in Assumptions \ref{assum:smooth} and \ref{assum:strong_convex}, respectively. 
 
 Then the expected $\ell_2$ distances of the local model iterates $\{x_{i,k}\}$ generated by \eqref{eq:x_k} to the points $x_{i,k}^\star$ defined in \eqref{eq:xi_star} are bounded
    \begin{align*}
    \E[\|x_{i,k+1} - x_{i,k+1}^\star\|] 
    &\leq (1-p_i)\E[\|x_{i,k} - x_{i,k}^\star\|] + (\alpha_{k} - \alpha_{k+1})LM\zeta  \\
    &\quad + p_i \left( 1-\frac{\alpha_k M \gamma }{4} \right) \E[\|x_{s,k} - x^\star\|] +\alpha_k^2 p_i \hat{c}_\zeta \zeta + \alpha_k p_i\hat{c}_\sigma  \sigma,
\end{align*}
where $ \hat{c}_\zeta \triangleq \frac{ L^2 M (M-1) }{2}$ and $\hat{c}_\sigma \triangleq   \left( 1 + \frac{\gamma}{4L}\right) M$.
\end{lemma}
\begin{proof}
    Subtracting $x_{i,k+1}^\star$ from both sides of \eqref{eq:x_k} and taking the norm on both sides yields
\begin{align*}
    \|x_{i,k+1} - x_{i,k+1}^\star\| &= (1-\phi_{i,k+1})\|x_{i,k} - x_{i,k+1}^\star\| +  \phi_{i,k+1}\|G_{i,\alpha_k}^M[x_{s,k}] - x_{i,k+1}^\star\|\\
    &\leq (1-\phi_{i,k+1})\|x_{i,k} - x_{i,k}^\star\| +  (1-\phi_{i,k+1})\| x_{i,k+1}^\star - x_{i,k}^\star\| \\
    &\quad +  \phi_{i,k+1}\|G_{i,\alpha_k}^M[x_{s,k}]  - x_{i,k}^\star\|+ \phi_{i,k+1}\|x_{i,k+1}^\star - x_{i,k}^\star\|\\
    &\leq (1-\phi_{i,k+1})\|x_{i,k} - x_{i,k}^\star\|  +  \phi_{i,k+1}\|G_{i,\alpha_k}^M[x_{s,k}]  - x_{i,k+1}^\star\|\\
    &\quad +  (\alpha_k - \alpha_{k+1})M\|\nabla f_i(x^\star)\|,
\end{align*}
where we obtain the second inequality by applying the triangle inequality, and the last inequality by applying the definitions of $x_{i,k}^\star$ and $x_{i,k+1}^\star$.

We take the expectation conditional on $\F_k$ on both sides of the preceding relation  and apply \eqref{eq:tilde_g} to obtain
\begin{align*}
    &\E[\|x_{i,k+1} - x_{i,k+1}^\star\||\F_k] = (1-p_i)\|x_{i,k} - x_{i,k}^\star \|  + (\alpha_k - \alpha_{k+1})M \|\nabla f_i(x^\star)\|\\
    &\qquad + p_i \E[\|x_{s,k} - \alpha_k M \nabla f_i(x_{s,k}) - x_{i,k}^\star  + \alpha_k \epsilon_{i,k}^d + \alpha_k \epsilon_{i,k}^g\||\F_k]\\
    &\quad \leq (1-p_i)\|x_{i,k} - x_{i,k}^\star\| + (\alpha_{k} - \alpha_{k+1})M \|\nabla f_i(x^\star)\| \\
    &\qquad + p_i \|x_{s,k} - \alpha_k M \nabla f_i(x_{s,k}) - x_{i,k}^\star \|+ \alpha_k p_i   \E[\| \epsilon_{i,k}^d \||\F_k]+ \alpha_k p_i \E[\|
    \epsilon_{i,k}^g\||\F_k]\\
    &\quad \leq (1-p_i)\|x_{i,k} - x_{i,k}^\star\| + (\alpha_{k} - \alpha_{k+1})M \|\nabla f_i(x^\star)\| \\
    &\qquad + p_i \sqrt{1-\alpha_k M \gamma} \|x_{s,k} - x^\star\| +\alpha_k p_i  \E[ \| \epsilon_{i,k}^d \||\F_k] + \alpha_k p_i \E[\|\epsilon_{i,k}^g\||\F_k],
\end{align*}
where the second inequality follows from applying the triangle inequality, and the last inequality from Lemma \ref{lem:nesterov}.

Bounding the norms of the errors $\epsilon_{i,k}^d$ and $\epsilon_{i,k}^g$ in the preceding relation using Lemma \ref{lem:errors_norms} yields
\begin{align*}
    &\E[\|x_{i,k+1} - x_{i,k+1}^\star\||\F_k] 
    \leq (1-p_i)\|x_{i,k} - x_{i,k}^\star\| + (\alpha_{k} - \alpha_{k+1})M \|\nabla f_i(x^\star)\| \\
    &\qquad + p_i \left( \sqrt{1-\alpha_k M \gamma}  +\frac{\alpha_k^2 L^2 M (M-1)  }{4} \right) \|x_{s,k} - x^\star\| \\
    &\qquad + \frac{\alpha_k^2 p_i L^2 M (M-1)\zeta }{2} + \frac{\alpha_k^2 p_i  L M (M-1) \sigma }{4}+ \alpha_k p_i M \sigma\\
    &\quad \leq (1-p_i)\|x_{i,k} - x_{i,k}^\star\| + (\alpha_{k} - \alpha_{k+1})M\|\nabla f_i(x^\star)\| \\
    &\qquad + p_i \left( \sqrt{1-\alpha_k M \gamma}  +\frac{\alpha_k^2 L^2 M (M-1)  }{4} \right) \|x_{s,k} - x^\star\|  +\alpha_k^2 p_i \hat{c}_\zeta \zeta +\alpha_k p_i\hat{c}_\sigma  \sigma\\
      &\quad \leq (1-p_i)\|x_{i,k} - x_{i,k}^\star\| + (\alpha_{k} - \alpha_{k+1})LM\zeta  \\
      &\qquad + p_i  \overbrace{\left( \sqrt{1-\alpha_k M \gamma}  +\frac{\alpha_k^2 L^2 M (M-1)  }{4} \right)}^{\triangleq r} \|x_{s,k} - x^\star\|  +\alpha_k^2 p_i \hat{c}_\zeta \zeta +\alpha_k p_i\hat{c}_\sigma  \sigma,
\end{align*}
where we obtain the second inequality by observing that $\frac{\alpha_k^2 p_i  L M (M-1) \sigma }{4} <\frac{\alpha_k p_i  \gamma M \sigma}{4 L} $ under $\alpha_k < \frac{\gamma}{L^2(M-1)}$  and applying the definitions of $\hat{c}_\zeta$ and $\hat{c}_\sigma$, and the last inequality by applying Assumption \ref{assum:smooth} and the definition of $\zeta$.

Due to $ \alpha_k   \leq \frac{ \gamma}{ L^2  (M-1)}$, we have
\begin{align*}
    r &\leq \sqrt{1-\alpha_k M \gamma} +  \frac{\alpha_k M \gamma}{4} \leq 1 - \frac{\alpha_k M \gamma}{2} + \frac{\alpha_k M \gamma}{4}= 1-\frac{\alpha_k M \gamma}{4},
\end{align*}
where the second inequality follows from \eqref{eq:standard_ineq_power}.

Combining the two preceding relations and taking the total expectation on both sides concludes the proof.
\end{proof}
In the next lemma, we combine the individual $\ell_2$ distances to optimality for the global and local iterates of \textsc{AREA} into a unified Lyapunov function, and provide guarantees for its convergence. 
\begin{lemma}
\label{lem:small_lyapunov}
 Define $\lambda_k \triangleq w_k'\hat{u}_k$, where $w_k\in \R^{2n+1}$ is defined in Lemma \ref{lem:properties_of_Ak} and  $\hat{u}_k \in \R^{2n+1}$ is given by
    \begin{align*}
    \hat{v}_k &\triangleq \begin{bmatrix}
    \left(    \{\E[\|x_{i,k} - x_{i,k}^\star\|]\}_{i=1}^n \right)'&
    \left( \{\E[\|y_{i,k} - x_{i,k}^\star\|]\}_{i=1}^n\right)' &
     \E[\|x_{s,k} - x^\star\|]
    \end{bmatrix}.
\end{align*}
If the step size sequence $\{\alpha_k\}$ in \eqref{eq:x_k} is defined as in Lemma \ref{lem:step size_sequence} with $p =\pmi$, the sequence $\{\lambda_k\}$ satisfies the following relation for $k=0,1,2,...$
\begin{align*}
\lambda_{k}
&\leq   \frac{(d-1)^4 \lambda_0}{(k+d-1)^4 } +  \frac{Q \zeta k^3}{3(k+d-1)^4 }  +  
\frac{\hat{c}_\sigma  \sigma k^4}{4(k+d-1)^4 } + \mathcal{O}\left(\frac{\sigma}{k}  +  \frac{\zeta}{k^2} \right). 
\end{align*}
where $Q \triangleq  4 c \bar{q} LM   +  2c^2  \hat{c}_\zeta $, $c,d$  are constants defined in Lemma \ref{lem:step size_sequence}, $\hat{c}_\sigma , \hat{c}_\zeta$ are constants defined in Lemma \ref{lem:x_norm}, and $\bar{q}$ is defined in \eqref{eq:eigenvector_products}.
\end{lemma}
\begin{proof}
We combine the results of Lemmas \ref{lem:s_norm}, \ref{lem:y_norm} and \ref{lem:x_norm}, and apply the definition of the transition matrix $A_k$ \eqref{eq:transition_matrix} to obtain the system of inequalities below
        \begin{align*}
  \hat{v}_{k+1} \leq A_k\hat{v}_k + (\alpha_k - \alpha_{k+1}) LM\zeta \begin{bmatrix}
        1_n\\1_n\\0
    \end{bmatrix} + ( \alpha_k^2 \hat{c}_\zeta \zeta + \alpha_k \hat{c}_\sigma  \sigma) \begin{bmatrix}
        P1_n \\0 \\0
    \end{bmatrix}.
\end{align*}
Multiplying both sides with $w_k'$ and applying \eqref{eq:eigenvector_products} and the definition of $\rho_k$ from Lemma \ref{lem:properties_of_Ak} yields
\begin{align*}
w_k'\hat{u}_{k+1} &\leq    \rho_k \lambda_k + (\alpha_k - \alpha_{k+1}) \frac{2 \bar{q} LM\zeta}{\beta_k} +  \frac{\alpha_k^2  \hat{c}_\zeta \zeta }{\beta_k}+   \frac{\alpha_k 
\hat{c}_\sigma  \sigma}{\beta_k}\\
&\leq    \rho_k \lambda_k + 4(\alpha_k - \alpha_{k+1})  \bar{q} LM\zeta +  2\alpha_k^2  \hat{c}_\zeta \zeta +   2 \alpha_k 
\hat{c}_\sigma  \sigma,
\end{align*}
where the last inequality follows from \eqref{eq:step_lower_bound}.

By the definition of $w_k$, we have $[w_{k+1}]_i \leq [w_k]_i$ for $i \in [2n+1]$ and hence $\lambda_{k+1} \leq w_k'\hat{u}_{k+1}.$ Applying this fact to the preceding relation and the definition of the sequence $\{\alpha_k\}$ from \eqref{eq:step size_definition} with $p=\pmi$ yields
\begin{align*}
\lambda_{k+1}
&\leq    \rho_k \lambda_k +  \frac{4 c \bar{q} LM\zeta}{(k+d)(k+d+1)}   +  \frac{2c^2  \hat{c}_\zeta \zeta}{(k+d)^2}+   \frac{2c 
\hat{c}_\sigma  \sigma}{k+d},
\end{align*}

We multiply both sides of the preceding relation with $ (k+d)^4$ and apply \eqref{eq:spectral_radius_monotonicity} and the definition of $Q$ to obtain
\begin{align*}
(k+d)^4 \lambda_{k+1}
&\leq   (k+d-1)^4  \lambda_k +  Q \zeta (k+d)^2  +  
\hat{c}_\sigma  \sigma (k+d)^3.
\end{align*}

Summing the relation above over iterations $k=0,1,...,K-1$ and applying the bounds in \eqref{eq:small_lyap_sum_k_2} and \eqref{eq:small_lyap_sum_k_3} yields
\begin{align*}
(K+d-1)^4 \lambda_{K}
&\leq   (d-1)^4 \lambda_0 +  \frac{Q \zeta K^3}{3}  +  
\frac{\hat{c}_\sigma  \sigma K^4}{4} + \mathcal{O}\left(\sigma K^3 +  \zeta K^2 \right). 
\end{align*}
Dividing both sides of the preceding relation with $(K+d-1)^4$ and setting $K=k$ completes the proof.
\end{proof}
Note that Lemma \eqref{lem:small_lyapunov} does not imply convergence for the $\ell_2$ distances to optimality, but rather that they are $ \mathcal{O}(\sigma)$. Obtaining this bound, however,  is essential for establishing the convergence of \textsc{AREA} later in Theorem \ref{thm:convergence_strong_convex}. Specifically, proving Lemma \eqref{lem:small_lyapunov} is necessitated by the fact that  utilizing a decreasing step size sequence over asynchronous communication introduces to each iteration of \textsc{AREA} a synchronization error of order $\mathcal{O}((\alpha_k  - \alpha_{k+1})
\lambda_k)$. As we will demonstrate later, it is possible to further improve this bound to $\mathcal{O}(\sqrt{1/k})$.

\subsection{Convergence of squared norms}
\label{sec:convergence_squared_norms}

We now prove our main theorem on the convergence of \textsc{AREA} under Assumptions \ref{assum:smooth} and \ref{assum:strong_convex}, namely the convergence of the expected squared $\ell_2$-norms of the distances to optimality of the global and local iterates. As in the previous subsection, we begin by bounding the errors induced by gradient stochasticity and client drift over an  iteration of \textsc{AREA}.
\begin{lemma}  \label{lem:drift_norm_squared}
The following inequality holds for the stochastic gradient error $\epsilon_{i,k}^g$ defined  in \eqref{eq:tilde_g} for all $i \in [n]$ and $k=0,1,2,...$
\begin{align*}
    \E[\|\epsilon_{i,k}^g\|^2|\F_k] \leq M \sigma^2,
\end{align*}
where $\sigma^2$ is defined in Assumption \ref{assum:grad_stoch}.

Moreover, if $\alpha_k < \frac{2}{\mu + L}$ in \eqref{eq:x_k}, the client drift error defined in \eqref{eq:tilde_g} also satisfies the inequality below for $i \in [n]$ and $k=1,2,...$
\begin{align*}
    \E[\|\epsilon_{i,k}^d\|^2|\F_k] 
 &\leq 4 \alpha_k^2 L^4 (M-1)^3 M \|x_{s,k} - x^\star\|^2 + 4 \alpha_k^2 L^4 (M-1)^3 M \zeta^2 +  2\alpha_k^2 L^2 (M-1)^3\sigma^2,
\end{align*}
where $\zeta >0 $ is defined in \eqref{eq:zeta_definition}.
\end{lemma}
\begin{proof} The first claim of this lemma is a direct consequence of the unbiasedness of the stochastic gradient approximations and the boundedness of the stochastic gradient errors due to Assumption \ref{assum:grad_stoch}, and the mutual independence of the stochastic gradient errors over the $M$ local SGD steps.

To prove the second claim, define $\delta^m_{i,k} \triangleq G_{i,\alpha_k}^{m} [x_{s,k}]  - x_{s,k}$. Taking the squared norm of $\epsilon_{i,k}^d$ and then the expectation conditional on $\F_k$ yields
\begin{equation}
\label{eq:drift_square_starter}
    \begin{split}
         \E[\|\epsilon_{i,k}^d\|^2|\F_k] 
    &\leq (M-1) \sum_{m=1}^{M-1}  \E\left[\left\| \nabla f_i(x_{s,k}) -  \nabla f_i \left( G_{i,\alpha_k}^{m} [x_{s,k}] \right) \right\|^2|\F_k \right]\\
    &\leq (M-1) L^2 \sum_{m=1}^{M-1}  \E\left[\left\| \delta^m_{i,k}    \right\|^2|\F_k \right],
    \end{split}
\end{equation}
where the first inequality follows from Jensen's inequality and the last inequality from Assumption \ref{assum:smooth}.

For the squared norm term on the right-hand side of the preceding relation we have
\begin{align*}
     &\E\left[\left\|\delta^m_{i,k}    \right\|^2|\F_k \right] =  \E\left[\left\|   G_{i,\alpha_k}^{m-1} [x_{s,k}] - \alpha_k g_i \left(G_{i,\alpha_k}^{m-1} [x_{s,k}] \right)  - x_{s,k}  \right\|^2|\F_k \right]\\
     &\quad\leq  \E\left[\left\|   G_{i,\alpha_k}^{m-1} [x_{s,k}] - \alpha_k \nabla f_i \left(G_{i,\alpha_k}^{m-1} [x_{s,k}] \right)  - x_{s,k}  \right\|^2|\F_k \right] + \alpha_k^2 \sigma^2\\
     &\quad\leq \overbrace{\left(1 + \frac{1}{M-1} \right)  \E\left[\left\|   G_{i,\alpha_k}^{m-1} [x_{s,k}] - \alpha_k \nabla f_i \left(G_{i,\alpha_k}^{m-1} [x_{s,k}] \right)  - x_{s,k}   + \alpha_k \nabla f_i(x_{s,k})\right\|^2|\F_k \right]}^{\triangleq T_3} \\
     &\qquad + \alpha_k^2 M \|\nabla f_i(x_{s,k})\|^2+ \alpha_k^2 \sigma^2,
\end{align*}
where the second inequality follows from Assumption \eqref{assum:grad_stoch}, and the last inequality from adding and subtracting $\alpha_k \nabla f_i(x_{s,k})$ inside the squared norm and applying \eqref{eq:young_ineq} with $c = 1/(M-1)$.

We apply Lemma \ref{lem:nesterov} to the first term on the right-hand side of the relation above to obtain
\begin{align*}
    T_3 &\leq  \left(1 + \frac{1}{M-1} \right) (1-\alpha_k \gamma) \E\left[\left\| \delta^{m-1}_{i,k} \right\|^2|\F_k \right]\leq  \left(1 + \frac{1}{M-1} \right)  \E\left[\left\| \delta^{m-1}_{i,k} \right\|^2|\F_k \right].
\end{align*}
Combining the two preceding relations and applying the resulting inequality recursively over local SGD steps $0,...,m$ yields
\begin{align*}
     \E\left[\left\|\delta^m_{i,k}    \right\|^2|\F_k \right]
     &\leq \left(1 + \frac{1}{M-1} \right)^{m}  \left\| \delta^0_{i,k} \right\|^2 + \left( \alpha_k^2 M \|\nabla f_i(x_{s,k})\|^2+ \alpha_k^2 \sigma^2\right) \sum_{l=0}^{m-1}\left( 1 + \frac{1}{M-1}\right)\\
     &\leq  2 \alpha_k^2 (M-1) M \|\nabla f_i(x_{s,k})\|^2+ 2\alpha_k^2 (M-1)\sigma^2,
\end{align*}
where the last inequality follows from the fact that $\sum_{l=0}^{m-1}\left( 1 + \frac{1}{M-1}\right) = \frac{\left(1 + \frac{1}{M-1} \right)^m - 1 }{1 + \frac{1}{M-1} - 1} \leq (e-1) (M-1) \leq 2(M-1)$.

We substitute the relation above in \eqref{eq:drift_square_starter} to obtain
\begin{align*}
    \E[\|\epsilon_{i,k}^d\|^2|\F_k] 
    &\leq 2 \alpha_k^2 L^2 (M-1)^3 M \|\nabla f_i(x_{s,k})\|^2+ 2\alpha_k^2 L^2 (M-1)^3\sigma^2\\
    &\leq 2 \alpha_k^2 L^4 (M-1)^3 M \|x_{s,k} - u_i^\star\|^2 +  2\alpha_k^2 L^2 (M-1)^3\sigma^2,
\end{align*}
where the last inequality follows from Assumption \ref{assum:smooth}.

Moreover, \eqref{eq:young_ineq} yields $ \|x_{s,k} - u_i^\star\|^2 \leq 2 \|x_{s,k} - x^\star\|^2  + 2 \|x^\star - u_i^\star\|^2  \leq  2 \|x_{s,k} - x^\star\|^2  + 2\zeta^2 $, where the last inequality follows from the definition of $\zeta$. Applying this fact to the inequality above concludes the proof.
\end{proof}
Next, we bound the distance to optimality for the global iterates of \textsc{AREA}.
\begin{lemma}
    \label{lem:s_norm_squared}
    The expected distance to the optimal solution $x^\star$ of the sequence of server iterates $\{x_{s,k}\}$ generated by \eqref{eq:xs_k} is bounded for $k=0,1,2,...$
    \begin{align*}
    \E[\|x_{s,k+1} - x^\star\|^2]
&\leq (1-p_s)\E[\|x_{s,k} - x^\star\|^2]+ \frac{p_s}{n}\sum_{i=1}^n \E[\|y_{i,k} - x_{i,k}^\star\|^2].
\end{align*}
\end{lemma}
\begin{proof}
  Subtracting $x^\star$ on both sides of \eqref{eq:xs_k} and taking the squared norm yields
\begin{align*}
    \|x_{s,k+1} - x^\star\|^2 &= (1-\phi_{s,k+1})\|x_{s,k} - x^\star\|^2 + \phi_{s,k+1} \|\bar{y}_k - x^\star\|^2.
\end{align*}
Taking the expectation conditional on $\F_k$ on both sides of the relation above yields
\begin{align*}
    \E[\|x_{s,k+1} - x^\star\|^2|\F_k] &= 
     (1-p_s)\|x_{s,k} - x^\star\|^2 + p_s \|\bar{y}_k - x^\star\|^2\\
     &=   (1-p_s)\|x_{s,k} - x^\star\|^2 + p_s \left\|\bar{y}_k - \alpha_k M \nabla f(x) - x^\star \right\|^2\\
&\leq (1-p_s)\|x_{s,k} - x^\star\|^2+ \frac{p_s}{n}\sum_{i=1}^n \|y_{i,k} - x_{i,k}^\star\|^2,
\end{align*}
where the first equality follows from the optimality of $x^\star$, and the last inequality from applying Jensen's inequality to the second term on the right-hand side.

Taking the total expectation on both sides of the preceding relation completes the proof.
\end{proof}
In the next lemma, we bound the distance of the local memory iterates of \textsc{AREA} to the sequence of points $\{x_{i,k}^\star\}$ defined in \eqref{eq:xi_star}.
\begin{lemma} 
\label{lem:y_norm_squared} Recall the definition of $x_{i,k}^\star \triangleq x^\star - \alpha_k M \nabla f_i(x^\star)$, where $M$ is the number of local SGD steps and $\{\alpha_k\}$ the step size sequence in \eqref{eq:x_k}. Then if $\alpha_{k+1} \leq \alpha_k$, the expected distance between the sequence of local memory iterates $\{y_{i,k}\}$ generated by \eqref{eq:y_k} and the sequence $\{x_{i,k}^\star\}$ is bounded for $i\in[n]$ and $k=0,1,2,...$
    \begin{align*}
    \E[\|y_{i,k+1} - x_{i,k+1}^\star\|^2] 
      &\leq (1-p_i)\E[\|y_{i,k} - x_{i,k}^\star\|^2] + (\alpha_k - \alpha_{k+1})^2 L^2 M^2\zeta^2\\
      &\quad + 2(1-p_i)(\alpha_k - \alpha_{k+1})LM\zeta \E[\| y_{i,k} - x_{i,k}^\star\|]\\
    &\quad   +  p_i\E[\|x_{i,k} - x_{i,k}^\star\|^2] + 2p_i(\alpha_k - \alpha_{k+1})LM\zeta\E[\|x_{i,k} - x_{i,k}^\star\|],
\end{align*}
where $\zeta > 0$ is defined in \eqref{eq:zeta_definition}.
\end{lemma}
\begin{proof}
      Subtracting $x_{i,k+1}^\star$ from both sides of \eqref{eq:y_k} and taking the squared norm on both sides yields
\begin{align*}
    \|y_{i,k+1} - x_{i,k+1}^\star\|^2 &= (1-\phi_{i,k+1})\|y_{i,k} - x_{i,k+1}^\star\|^2 +  \phi_{i,k+1}\|x_{i,k} - x_{i,k+1}^\star\|^2\\
    &= (1-\phi_{i,k+1})\|y_{i,k} - x_{i,k}^\star\|^2 + \| x_{i,k+1}^\star - x_{i,k}^\star\|^2 \\
    &\quad + 2(1-\phi_{i,k+1})\langle y_{i,k} - x_{i,k}^\star, x_{i,k+1}^\star - x_{i,k}^\star\rangle \\
    &\quad   + \phi_{i,k+1}\|x_{i,k} - x_{i,k}^\star\|^2 + 2\phi_{i,k+1} \langle x_{i,k} - x_{i,k}^\star, x_{i,k+1}^\star - x_{i,k}^\star \rangle  \\
    &\leq (1-\phi_{i,k+1})\|y_{i,k} - x_{i,k}^\star\|^2 + \| x_{i,k+1}^\star - x_{i,k}^\star\|^2 \\
    &\quad + 2(1-\phi_{i,k+1})\| y_{i,k} - x_{i,k}^\star\|\| x_{i,k+1}^\star - x_{i,k}^\star\|\\
    &\quad   +  \phi_{i,k+1}\|x_{i,k} - x_{i,k}^\star\|^2 + 2\phi_{i,k+1}\|x_{i,k} - x_{i,k}^\star\|\| x_{i,k+1}^\star - x_{i,k}^\star\| ,
\end{align*}
where we obtain the second inequality by adding and subtracting $x_{i,k}^\star$ inside the two squared norms and expanding the squares, and the last inequality from applying the Cauchy-Schwarz inequality.

Taking the expectation conditional on $\F_k$ on both sides of the relation above yields
\begin{align*}
    \E[\|y_{i,k+1} - x_{i,k+1}^\star\|^2|\F_k] 
    &\leq (1-p_i)\|y_{i,k} - x_{i,k}^\star\|^2 + \| x_{i,k+1}^\star - x_{i,k}^\star\|^2 \\
    &\quad + 2(1-p_i)\| y_{i,k} - x_{i,k}^\star\|\| x_{i,k+1}^\star - x_{i,k}^\star\|\\
    &\quad   +  p_i\|x_{i,k} - x_{i,k}^\star\|^2 + 2p_i\|x_{i,k} - x_{i,k}^\star\|\| x_{i,k+1}^\star - x_{i,k}^\star\|\\
    &\leq (1-p_i)\|y_{i,k} - x_{i,k}^\star\|^2 + (\alpha_k - \alpha_{k+1})^2 M^2\| \nabla f_i(x^\star)\|^2 \\
    &\quad + 2(1-p_i)(\alpha_k - \alpha_{k+1})M\| y_{i,k} - x_{i,k}^\star\|\| \nabla f_i(x^\star)\|\\
    &\quad   +  p_i\|x_{i,k} - x_{i,k}^\star\|^2 + 2p_i(\alpha_k - \alpha_{k+1})M\|x_{i,k} - x_{i,k}^\star\|\|\nabla f_i(x^\star)\|\\
      &\leq (1-p_i)\|y_{i,k} - x_{i,k}^\star\|^2 + (\alpha_k - \alpha_{k+1})^2 L^2 M^2\zeta^2 \\
      &\quad + 2(1-p_i)(\alpha_k - \alpha_{k+1})LM\zeta \| y_{i,k} - x_{i,k}^\star\|\\
    &\quad   +  p_i\|x_{i,k} - x_{i,k}^\star\|^2 + 2p_i(\alpha_k - \alpha_{k+1})LM\zeta\|x_{i,k} - x_{i,k}^\star\|,
\end{align*}
where the second inequality follows from the definitions of $x_{i,k}^\star$ and $x_{i,k+1}^\star$, and the last inequality from  applying Assumption \ref{assum:smooth} and the definition of $\zeta$.

Taking the total expectation on both sides of the relation above concludes the proof.
\end{proof}
Before stating our main theorem, we prove one more intermediate lemma bounding the distance of the local model iterates to the sequence of points $\{x_{i,k}^\star\}$.

\begin{lemma}
    \label{lem:x_norm_squared} Suppose that
the step size sequence $\{\alpha_k\}$ in \eqref{eq:x_k} satisfies the following inequalities for $k=0,1,2,...$
\begin{align*}
    \alpha_{k+1} \leq \alpha_k, \quad  \alpha_k  \leq \min \left\{\frac{2}{M(\mu + L)},\sqrt{\frac{ M \gamma^2}{64L^4 (M-1)^3 }}, \sqrt[3]{\frac{ \gamma}{32 L^4 (M-1)^3 }} \right\},
\end{align*}
where $M$ is the number of local SGD steps, $\gamma \triangleq \frac{2\mu L}{\mu + L}$, and $L$ and $\mu$ are the Lipschitz and strong convexity constants defined in Assumptions \ref{assum:smooth} and \ref{assum:strong_convex}, respectively.

Then the distance between the sequence of local iterates $\{x_{i,k}\}$ generated by \eqref{eq:x_k} and the sequence of points $\{x_{i,k}^\star\}$ is bounded for $i \in [n]$ and $k=0,1,2,...$
\begin{align*}
   \E[\|x_{i,k+1} - x_{i,k+1}^\star\|^2] &\leq (1-p_i) \E[\|x_{i,k} - x_{i,k}^\star\|^2] +  (\alpha_{k+1} - \alpha_k)^2 L^2 M^2\zeta^2\\
   &\quad + 2(1-p_i) (\alpha_{k} - \alpha_{k+1})M L \zeta  \E[\|x_{i,k} - x_{i,k}^\star\|] \\
   &\quad + p_i\left( 1 - \frac{\alpha_k M \gamma}{4} \right) \E[\|x_{s,k} - x^\star\|^2]  \\
    &\quad + 2p_i (\alpha_{k} - \alpha_{k+1})LM\zeta \left(1-\frac{\alpha_k M \gamma}{2}\right) \E[\| x_{s,k} - x^\star\|]  \\
    &\quad + p_i\alpha_k^2 c_\sigma \sigma^2   + p_i\alpha_k^3  c_\zeta  \zeta^2,  
\end{align*}
where $\zeta>0$ is defined in \eqref{eq:zeta_definition}, $c_\sigma \triangleq 2 M \tilde{c}_\sigma$, $c_\zeta \triangleq 4L^4 (M-1)^3 \tilde{c}_\zeta $, $\tilde{c}_\sigma \triangleq 1 + \frac{2L\left( 1 - \frac{1}{M}\right) }{\mu} + \frac{4 L^2 \left( 1 - \frac{1}{M}\right)}{M(\mu + L)^2}$ and $ \tilde{c}_\zeta \triangleq  \frac{2}{\gamma} + \frac{2}{\mu + L}$.
\end{lemma}
\begin{proof}
      Subtracting $x_{i,k+1}^\star$ from both sides of \eqref{eq:x_k} and taking the squared norm on both sides yields
\begin{align*}
    \|x_{i,k+1} - x_{i,k+1}^\star\|^2 &= (1-\phi_{i,k+1})\|y_{i,k} - x_{i,k+1}^\star\|^2 +  \phi_{i,k+1}\|G_{i,\alpha_k}^M[x_{s,k}] - x_{i,k+1}^\star\|^2.
\end{align*}
Taking the expectation conditional on $\F_k$ on both sides of the relation above and applying \eqref{eq:tilde_g} yields
\begin{equation}
\label{eq:lemma_x_norm_1}
    \begin{split}
    \E[\|x_{i,k+1} - x_{i,k+1}^\star\|^2|\F_k] &= (1-p_i) \overbrace{\|x_{i,k} - x_{i,k+1}^\star\|^2}^{\triangleq T_1} \\
   &\quad + p_i \overbrace{\E[\|x_{s,k} - \alpha_k M \nabla f_i(x_{s,k}) + \alpha_k \epsilon_{i,k}^d + \alpha_k \epsilon_{i,k}^g - x_{i,k+1}^\star\|^2|\F_k]}^{\triangleq T_2}.       
    \end{split}
\end{equation}
First, we bound the term $T_1$ in \eqref{eq:lemma_x_norm_1}. Adding and subtracting $x_{i,k}^\star$ inside the norm and then expanding the square yields
    \begin{align*}
    T_1 
    &= \|x_{i,k} - x_{i,k}^\star\|^2 +  \|x_{i,k+1}^\star - x_{i,k}^\star\|^2 - 2 \langle x_{i,k} - x_{i,k}^\star, x_{i,k+1}^\star - x_{i,k}^\star\rangle \\
    &= \|x_{i,k} - x_{i,k}^\star\|^2 + (\alpha_k - \alpha_{k+1})^2 M^2  \|\nabla f_i(x^\star)\|^2 + 2 (\alpha_k - \alpha_{k+1})M\langle x_{i,k} - x_{i,k}^\star,\nabla f_i(x^\star)\rangle \\
    &\leq \|x_{i,k} - x_{i,k}^\star\|^2 + (\alpha_{k+1} - \alpha_k)^2 M^2 \|\nabla f_i(x^\star)\|^2 \\
    &\quad + 2(\alpha_{k} - \alpha_{k+1})M \|\nabla f_i(x^\star)\| \|x_{i,k} - x_{i,k}^\star\|,
\end{align*}
where the second equality follows from the definitions of $x_{i,k}^\star$ and $x_{i,k+1}^\star$, and the last inequality from Cauchy-Schwarz.

Applying Assumption \ref{assum:smooth} and the definition of $\zeta$ from \eqref{eq:zeta_definition} to the relation above further yields
    \begin{align}
    \label{eq:x_norm_squared_t1}
    T_1 
    &\leq \|x_{i,k} - x_{i,k}^\star\|^2 + (\alpha_{k+1} - \alpha_k)^2 L^2 M^2\zeta^2 + 2(\alpha_{k} - \alpha_{k+1})M L \zeta  \|x_{i,k} - x_{i,k}^\star\|.
\end{align}
Next, we bound the term $T_2$ in \eqref{eq:lemma_x_norm_1}.  Define $\hat{s}_k \triangleq x_{s,k} - \alpha_k M \nabla f_i(x_{s,k}) - x_{i,k}^\star$. We then have
\begin{align*}
    T_2 &= \E[\|x_{s,k} - \alpha_k M \nabla f_i(x_{s,k}) - x_{i,k}^\star  + (\alpha_{k+1} - \alpha_k)M \nabla f_i(x^\star) + \alpha_k \epsilon_{i,k}^d + \alpha_k \epsilon_{i,k}^g\|^2|\F_k]\\
    &= \|\hat{s}_k\|^2 + (\alpha_{k+1} - \alpha_k)^2M^2\|\nabla f_i(x^\star)\|^2 + \alpha_k^2 \E[\|\epsilon_{i,k}^d\|^2|\F_k] + \alpha_k^2 \E[\|\epsilon_{i,k}^g\|^2|\F_k] \\
    &\quad + 2 (\alpha_{k+1} - \alpha_k)M \langle \hat{s}_k, \nabla f_i(x^\star)\rangle  + 2\E[\langle \hat{s}_k, \alpha_k \epsilon_{i,k}^d\rangle |\F_k ] +  2\E[\langle \hat{s}_k, \alpha_k \epsilon_{i,k}^g\rangle |\F_k ]\\
    &\quad + 2\E[\langle (\alpha_{k+1} - \alpha_k) M \nabla f_i(x^\star), \alpha_k \epsilon_{i,k}^d \rangle|\F_k]  + 2\E[\langle (\alpha_{k+1} - \alpha_k) M \nabla f_i(x^\star), \alpha_k \epsilon_{i,k}^g \rangle|\F_k] \\
    &\quad + 2\E[ \langle \alpha_k \epsilon_{i,k}^d, \alpha_k \epsilon_{i,k}^g \rangle |\F_k]\\
    &\leq \|\hat{s}_k\|^2 + (\alpha_{k+1} - \alpha_k)^2M^2\|\nabla f_i(x^\star)\|^2 + \alpha_k^2 \E[\|\epsilon_{i,k}^d\||\F_k] + \alpha_k^2 \E[\|\epsilon_{i,k}^d\||\F_k]  \\
    &\quad + 2 (\alpha_{k+1} - \alpha_k)M \langle \hat{s}_k, \nabla f_i(x^\star)\rangle + 2\E[\langle \hat{s}_k, \alpha_k \epsilon_{i,k}^d\rangle |\F_k ] \\
    &\quad + 2\E[\langle (\alpha_{k+1} - \alpha_k) M \nabla f_i(x^\star), \alpha_k \epsilon_{i,k}^d \rangle|\F_k]  + 2\E[ \langle \alpha_k \epsilon_{i,k}^d, \alpha_k \epsilon_{i,k}^g \rangle |\F_k],
\end{align*}
where the last inequality follows the unbiasedness of the stochastic gradient approximations due to Assumption \eqref{assum:grad_stoch} (note that $\E[\langle \epsilon_{i,k}^d,\epsilon_{i,k}^g\rangle|\F_k] \neq 0$ as the stochastic gradient and client drift errors are dependent).

We bound the inner products in the preceding relation as follows
\begin{gather*}
   \langle \hat{s}_k, \nabla f_i(x^\star)\rangle \leq \|\hat{s}_k\| \|\nabla f_i(x^\star)\| \\
    2\E[\langle \hat{s}_k, \alpha_k \epsilon_{i,k}^g\rangle |\F_k ]  \leq \frac{\alpha_k M \gamma}{2(1-\alpha_k M\gamma)} \|\hat{s}_k\|^2 + \frac{2\alpha_k (1-\alpha_k M \gamma)}{ M \gamma}\E[\|\epsilon_{i,k}^d\|^2|\F_k]\\
   2\E[\langle (\alpha_{k+1} - \alpha_k) M \nabla f_i(x^\star), \alpha_k \epsilon_{i,k}^d \rangle|\F_k] \leq (\alpha_{k+1}-\alpha_k)^2M^2\|\nabla f_i(x^\star)\|^2 + \alpha_k^2 \E[\|\epsilon_{i,k}^d\|^2|F_k]\\
   2\E[ \langle \alpha_k \epsilon_{i,k}^d, \alpha_k \epsilon_{i,k}^g \rangle |\F_k] \leq \alpha_k^2 \E[\|\epsilon_{i,k}^d\|^2|\F_k] + \alpha_k^2 \E[\|\epsilon_{i,k}^g\|^2|\F_k],
\end{gather*}
where we obtain the first bound by applying the Cauchy-Schwarz inequality and remaining bounds using \eqref{eq:young_ineq}.

Combining the two preceding relations and bounding the squared norms of the stochastic gradient and client drift errors using \ref{lem:drift_norm_squared} yields
\begin{align*}
    &T_2
    \leq \left( 1+ \frac{\alpha_k M \gamma}{2(1-\alpha_k M \gamma)} \right) \|\hat{s}_k\|^2  + 2(\alpha_{k+1} - \alpha_k)^2M^2\|\nabla f_i(x^\star)\|^2  \\
    &\quad + 2 (\alpha_{k} - \alpha_{k+1})M \| \hat{s}_k\| \|\nabla f_i(x^\star)\| \\
    &\quad + 2\alpha_k^2 M \overbrace{\left(   1 +  \frac{2\alpha_k L^2 (M-1)^3}{M^2 \gamma} + \frac{\alpha_k^2 L^2 (M-1)^3}{M^2} \right)}^{\triangleq T_3}\sigma^2   + 4\alpha_k^3 L^4 (M-1)^3 \overbrace{\left( \frac{2  }{\gamma} + \alpha_k M  \right)}^{\triangleq T_4}   \zeta^2\\
    &\quad + \overbrace{\left(\frac{  8 \alpha_k^3 L^4 (M-1)^3  }{ \gamma} + 4 \alpha_k^4 L^4 (M-1)^3 M  \right)}^{T_5}\|x_{s,k} - x^\star\|^2.
\end{align*}
Note that if $ \alpha_k  \leq \min \left\{\frac{2}{M(\mu+L)},\sqrt{\frac{ M \gamma^2}{64L^4 (M-1)^3 }}, \sqrt[3]{\frac{ \gamma}{32 L^4 (M-1)^3 }} \right\}$, the following bounds hold
\begin{gather*}
T_3 \leq 1 + \frac{2 \left(\frac{2}{M(\mu + L)} \right) L^2(M-1)}{\gamma} + \frac{4 L^2 (M-1)}{M^2(\mu + L)^2} \leq \tilde{c}_\sigma\\
T_4 \leq \frac{2}{\gamma} + \frac{2M}{M(\mu + L)} \leq \tilde{c}_\zeta \\
   T_5 \leq \frac{\alpha_k M \gamma}{8} + \frac{\alpha_k M \gamma}{8} = \frac{\alpha_k M \gamma}{4}.
\end{gather*}
We combine the two preceding relations and apply the definitions of $c_\sigma$ and $c_\zeta$ to obtain
\begin{align*}
    T_2
    &\leq \left( 1+ \frac{\alpha_k M \gamma}{2(1-\alpha_k M \gamma)} \right) \|\hat{s}_k\|^2  + 2(\alpha_{k+1} - \alpha_k)^2M^2\|\nabla f_i(x^\star)\|^2  \\
    &\quad + 2 (\alpha_{k} - \alpha_{k+1})M \| \hat{s}_k\| \|\nabla f_i(x^\star)\| \\
    &\quad + \alpha_k^2 c_\sigma \sigma^2   + \alpha_k^3 c_\zeta  \zeta^2 + \frac{\alpha_k M \gamma}{4}\|x_{s,k} - x^\star\|^2.
\end{align*}
By Lemma \ref{lem:nesterov}, we have $\|\hat{s}_k\|^2 \leq (1-\alpha_k M \gamma)\|x_{s,k} - x^\star\|^2$. Applying this fact to the inequality above yields
\begin{equation}
\label{eq:x_norm_squared_t2}
\begin{split}
   T_2
    &\leq \left( 1 - \frac{\alpha_k M \gamma}{4} \right) \|x_{s,k} - x^\star\|^2  + 2(\alpha_{k+1} - \alpha_k)^2M^2\|\nabla f_i(x^\star)\|^2  \\
    &\quad + 2 (\alpha_{k} - \alpha_{k+1})M \sqrt{1-\alpha_k M \gamma} \| x_{s,k} - x^\star\| \|\nabla f_i(x^\star)\| \\
    &\quad + \alpha_k^2  c_\sigma \sigma^2   + \alpha_k^3 c_\zeta  \zeta^2 \\
    &\leq \left( 1 - \frac{\alpha_k M \gamma}{4} \right) \|x_{s,k} - x^\star\|^2  + 2(\alpha_{k+1} - \alpha_k)^2L^2 M^2\zeta^2 \\
    &\quad + 2 (\alpha_{k} - \alpha_{k+1})LM\zeta \sqrt{1-\alpha_k M \gamma} \| x_{s,k} - x^\star\|   + \alpha_k^2  c_\sigma \sigma^2   + \alpha_k^3 c_\zeta  \zeta^2,  
\end{split}
\end{equation}
where the last inequality follows from applying Assumption \ref{assum:smooth} and the definition of $\zeta$.

Substituting \eqref{eq:x_norm_squared_t1} and \eqref{eq:x_norm_squared_t2} in \eqref{eq:lemma_x_norm_1} yields
\begin{align*}
   &\E[\|x_{i,k+1} - x_{i,k+1}^\star\|^2|\F_k] \leq (1-p_i) \|x_{i,k} - x_{i,k}^\star\|^2 +  (\alpha_{k+1} - \alpha_k)^2 L^2 M^2\zeta^2\\
   &\qquad + 2(1-p_i) (\alpha_{k} - \alpha_{k+1})M L \zeta  \|x_{i,k} - x_{i,k}^\star\| + p_i\left( 1 - \frac{\alpha_k M \gamma}{4} \right) \|x_{s,k} - x^\star\|^2  \\
    &\qquad + 2p_i (\alpha_{k} - \alpha_{k+1})LM\zeta \sqrt{1-\alpha_k M \gamma} \| x_{s,k} - x^\star\|   + p_i\alpha_k^2 c_\sigma \sigma^2   + p_i\alpha_k^3  c_\zeta  \zeta^2.  
\end{align*}

Observing that $\sqrt{1-\alpha_k M \gamma} \leq 1 - \frac{\alpha_k M \gamma}{2}$ due to \eqref{eq:standard_ineq_power} and taking the total expectation on both sides of the relation above completes the proof.
\end{proof}
For convenience, we re-state Theorem \ref{thm:convergence_strong_convex} here, followed by its proof.
\begin{thm_strong_convex}
Under Assumptions \ref{assum:grad_stoch}, \ref{assum:smooth} and \ref{assum:strong_convex}, let $x^\star = \arg \min_x f(x)$ and suppose that the step size sequence $\{\alpha_k\}$ in Algorithm \ref{alg:area_client} is defined as follows
\begin{align*}
    \alpha_k \triangleq \tfrac{1}{ \tfrac{\pmi M \gamma k }{48}  +  \tfrac{1}{\mathcal{D}}}, \text{ where }  \mathcal{D} \triangleq \min\left\{ \tfrac{2}{M(\mu + L)}, \tfrac{\gamma}{L^2(M-1)}, \sqrt{\tfrac{ M \gamma^2}{64L^4 (M-1)^3 }}, \sqrt[3]{\tfrac{ \gamma}{32 L^4 (M-1)^3 }} \right\},
\end{align*}
and $\gamma \triangleq \frac{2\mu L}{\mu + L}$, $\pmi \triangleq \min\{\min_i p_i, p_s\}$.

Then after $K$ iterations, the sequence  $\{x_{s,k}\}$ of global server iterates in Algorithm \ref{alg:area_server} satisfies
\begin{equation*}
\begin{split}
\E[\|x_{s,K} - x^\star\|^2] & = \mathcal{O}\Bigg( \frac{r_p (1+\delta \kappa) \sigma^2}{\pmi M K} + \frac{r_p \bar{q}^{1/2} L (1+\delta \kappa)^{1/2}\sigma \zeta}{\pmi M^{1/2} K^{3/2}}\\
    &\quad + \frac{r_p \bar{q}^{1/2} L^{3/2}(\sigma \zeta^3)^{1/2}}{\pmi^{1/2} K^{3/2}} + \frac{r_p  (\delta  + \pmi^{1/2} \bar{q}) L^4\zeta^2}{\pmi^2 K^2}\Bigg),       
    \end{split}
\end{equation*}
where $r_p \triangleq \frac{p_s}{\pmi}$, $\bar{q} \triangleq \frac{1}{n}\sum_{i=1}^n \tfrac{1}{p_i}$, $\zeta \triangleq \max_i\|x^\star - u_i^\star\|$, $u_i^\star \triangleq \arg \min_x f_i(x)$, $\kappa \triangleq \tfrac{L}{\mu}$ is the condition number, and $\delta \triangleq 1 - \tfrac{1}{M}$.
\end{thm_strong_convex}
\begin{proof} Let $\Lambda_k \triangleq w_k'u_k$, where $w_k \in \R^{2n+1}$ is defined in Lemma \ref{lem:properties_of_Ak} and $u_k \in \R^{2n+1}$ is given by
\begin{align*}
 u_k\triangleq \begin{bmatrix}
      \left( \{\E[\|x_{i,k} - x_{i,k}^\star\|^2]\}_{i=1}^n\right)' &
\left(\{\E[\|y_{i,k} - x_{i,k}^\star\|^2]\}_{i=1}^n \right)' &
      \E[\|x_{s,k} - x^\star\|^2]
    \end{bmatrix}'.
\end{align*}
Moreover, define the matrix $B_k \in \R^{(2n+1)^2}$ 
\begin{align*}
 B_k \triangleq  \begin{bmatrix}
       I-P & 0 & \left(1- \frac{\alpha_k M \gamma}{2} \right) P1_n\\
       P & I-P & 0 \\
       0 & 0 & 0
   \end{bmatrix}.
\end{align*}
Combining the results of Lemmas \eqref{lem:s_norm_squared}, \eqref{lem:y_norm_squared} and \eqref{lem:x_norm_squared} and applying the definitions of $u_k$ and $B_k$, and the definitions of $A_k$ and $\hat{u}_k$ from \eqref{eq:transition_matrix} and Lemma \ref{lem:small_lyapunov}, respectively, yields the following system of inequalities
\begin{align*}
  u_{k+1} &\leq A_k u_k + 2(\alpha_k - \alpha_{k+1}) LM \zeta B_k \hat{u}_k + 2(\alpha_{k+1} - \alpha_k)^2 L^2 M^2 \zeta^2 \begin{bmatrix}
   1_n\\
   1_n\\
   0   
   \end{bmatrix} \\
  &\quad  + \alpha_k^2 c_\sigma \sigma^2\begin{bmatrix}
       P1_n\\0\\0
   \end{bmatrix}  + \alpha_k^3 c_\zeta \zeta^2 \begin{bmatrix}
       P1_n\\0\\0
   \end{bmatrix}\\
   &\leq A_k u_k + 2(\alpha_k - \alpha_{k+1}) LM \zeta A_k \hat{u}_k+ 2(\alpha_{k+1} - \alpha_k)^2 L^2 M^2 \zeta^2 \begin{bmatrix}
   1_n\\
   1_n\\
   0   
   \end{bmatrix} \\
   &\quad  + \alpha_k^2 c_\sigma \sigma^2\begin{bmatrix}
       P1_n\\0\\0
   \end{bmatrix}  + \alpha_k^3 c_\zeta \zeta^2 \begin{bmatrix}
       P1_n\\0\\0
   \end{bmatrix},
\end{align*}
where the last inequality follows from the fact that $\hat{u}_k$ is non-negative and the elements of $A_k$ are either equal or greater than the corresponding elements of $B_k$.

Recall the definition of  $\lambda_k \triangleq w_k' \hat{u}_k$ from Lemma \ref{lem:small_lyapunov}. Multiplying both sides of the preceding inequality with $w_k'$ and applying \eqref{eq:rho_k_bound} to bound the product $w'_k A_k$ and \eqref{eq:eigenvector_products} to bound the last 3 terms on the right-hand side yields
\begin{align*}
 w_k' u_{k+1} 
   &\leq \rho_k \Lambda_k + 2(\alpha_k - \alpha_{k+1}) LM \zeta \rho_k \lambda_k  + 4 L^2 M^2 \bar{q} \zeta^2 \frac{(\alpha_{k+1} - \alpha_k)^2}{\beta_k}   +  c_\sigma \sigma^2 \frac{\alpha_k^2}{\beta_k} +  c_\zeta \zeta^2 \frac{\alpha_k^3}{\beta_k}\\
   &\leq \rho_k \Lambda_k + 2(\alpha_k - \alpha_{k+1}) LM \zeta \rho_k \lambda_k  + 8 L^2 M^2 \bar{q} \zeta^2 (\alpha_{k+1} - \alpha_k)^2 + 2 c_\sigma \sigma^2 \alpha_k^2 +  2 c_\zeta \zeta^2 \alpha_k^3,
\end{align*}
where the last inequality follows from \eqref{eq:step_lower_bound}.

Note that $\Lambda_{k+1} = w_{k+1}' u_k \leq w_k'u_k$. Applying this fact and the definition of $\alpha_k$ from \eqref{eq:step size_definition} with $p= \pmi$ to the relation above yields
\begin{align*}
\Lambda_{k+1} 
   &\leq \rho_k \Lambda_k + \frac{c_1 \zeta \rho_k \lambda_k}{(k+d)(k+d+1)}  + \frac{c_2 \zeta^2 }{(k+d)^2(k+d+1)^2} + \frac{c_3 \sigma^2 }{(k+d)^2} +  \frac{c_4 \zeta^2}{(k+d)^3},
\end{align*}
where $c_1 \triangleq 2 c LM$, $c_2 \triangleq 8 c^2L^2 M^2 \bar{q} $, $c_3 \triangleq 2 c^2 c_\sigma$ and $c_4 \triangleq 2 c^3 c_\zeta $.

Next, we multiply both sides of the preceding relation with $(k+d)^4$ and apply \eqref{eq:spectral_radius_monotonicity} to obtain
\begin{align}
\label{eq:Lamba_convergence_starter}
(k+d)^4\Lambda_{k+1} 
   &\leq (k+d-1)^4 \Lambda_k + c_1 \zeta  \lambda_k (k+d-1)^2  + c_2 \zeta^2  + c_3 \sigma^2 (k+d)^2 +  c_4 \zeta^2 (k+d).
\end{align}
Applying the result of Lemma \ref{lem:small_lyapunov} to bound $\lambda_k$ in the preceding relation yields
\begin{align*}
(k+d)^4\Lambda_{k+1} 
   &\leq (k+d-1)^4 \Lambda_k  +    \frac{(d-1)^4 c_1 \zeta \lambda_0}{(k+d-1)^2 } +   (c_1 Q  + c_4) \zeta^2 k  \\
   &\quad +  
\frac{\hat{c}_\sigma c_1   \sigma \zeta  k^2}{4 }  + c_3 \sigma^2 (k+d)^2   + \mathcal{O}\left(\sigma \zeta k  +  \zeta^2  \right).
\end{align*}
We sum the preceding relation over iterations $k=0,1,...,K-1$ to obtain
\begin{align*}
(K+d-1)^4\Lambda_{K} 
   &\leq (d-1)^4 \Lambda_0  +    2(d-1)^4 c_1 \zeta \lambda_0 +   \frac{(c_1 Q  + c_4) \zeta^2 K^2}{2}   \\
   &\quad +  
\frac{\hat{c}_\sigma c_1   \sigma \zeta  K^3}{12 }  + \frac{c_3 \sigma^2 K^3}{3}   + \mathcal{O}\left(\sigma^2 K^2 + \sigma \zeta K^2  +  \zeta^2 K \right).
\end{align*}
where we used the fact that  $ \sum_{k=0}^{K-1} \frac{1}{\tilde{k}^2} \leq \zeta(2)=  \frac{\pi^2}{6} < 2$ to bound the $2^{nd}$ term on the right-hand side  (here $\zeta(\cdot)$ denotes the Riemann zeta function, not to be confused with the constant $\zeta$ in \eqref{eq:zeta_definition}), and applied \eqref{eq:small_lyap_sum_k_2} to bound $5^{th}$ term.

Let $\tilde{k} \triangleq k+d-1$ for $k=0,1,...,K$. Dividing both sides of the inequality above with $\tilde{K}^4  $ yields
\begin{equation}
\label{eq:rough_bound_Lambda}
\begin{split}
\Lambda_{K} 
   &\leq \frac{(d-1)^4 \Lambda_0}{\tilde{K}^4}  +    \frac{2(d-1)^4 c_1 \zeta \lambda_0}{\tilde{K}^4} +   \frac{(c_1 Q  + c_4) \zeta^2 }{2\tilde{K}^2}   \\
   &\quad +  
\frac{\hat{c}_\sigma c_1   \sigma \zeta }{12\tilde{K} }  + \frac{c_3 \sigma^2 }{3\tilde{K}}   + \mathcal{O}\left(\frac{\sigma^2}{K^2} + \frac{\sigma \zeta}{K^2}  + \frac{\zeta^2}{K^3} \right).
\end{split}
\end{equation}
The preceding bound implies that $\Lambda_k = \mathcal{O}\left( \frac{\sigma \zeta + \sigma^2}{K} + \frac{\zeta^2}{K^2}\right)$. However, we can improve the dependency  on $\sigma \zeta$ by using \eqref{eq:rough_bound_Lambda} to re-derive a bound for the sequence $\{\lambda_k\}$. We will then apply  this new bound to \eqref{eq:Lamba_convergence_starter}  and repeat the analysis to tighten the convergence rate for  $\{\Lambda_k\}$. 

Let $\psi \triangleq 1 + \bar{q}$. Using Jensen's inequality for concave functions we obtain 
\begin{align*}
    \sqrt{\Lambda_k} &= \sqrt{\sum_{j=1}^{2n+1} [w_k]_j  [u_k]_j} \geq \frac{\sum_{j=1}^{2n+1} [w_k]_j \sqrt{[u_k]_j}}{\sqrt{\sum_{j=1}^{2n+1}[w_k]_j}} \\
    &\geq \frac{\sum_{j=1}^{2n+1} [w_k]_j [\hat{u}_k]_j}{\sqrt{\sum_{j=1}^{2n+1}[w_k]_j}} =\frac{\lambda_k}{\sqrt{\sum_{j=1}^{2n+1}[w_k]_j}} \geq \frac{\lambda_k}{\sqrt{1+ \frac{2 \bar{q}}{\beta_k}}} \geq \frac{\lambda_k}{(\psi)^{1/2}},
\end{align*}
where the second equality follows again from Jensen's inequality, i.e., $ [\hat{u}_k]_j \leq\sqrt{[u_k]_j}$ for  $j \in [2n+1]$, the third inequality from \eqref{eq:eigenvector_products}, and the last inequality from the definition of $\psi$  and \eqref{eq:step_lower_bound}.

Multiplying both sides of the relation above with $(\psi)^{1/2}$  and applying \eqref{eq:rough_bound_Lambda} yields
\begin{align*}
\lambda_k &\leq  \frac{c_5 \Lambda_0^{1/2} + c_6 ( \zeta \lambda_0)^{1/2}}{\tilde{k}^2}   +   \frac{c_7\zeta}{\tilde{k}}  + \frac{c_8(\sigma \zeta)^{1/2}  + c_9\sigma }{\tilde{k}^{1/2}} + \mathcal{O}\left(\frac{\sigma + (\sigma \zeta)^{1/2}}{k}   + \frac{\zeta}{k^{3/2}} \right).
\end{align*}
where we used the standard relation $\sqrt{a + b} \leq \sqrt{a} + \sqrt{b}$ for any $a,b>0$, and $c_5 \triangleq (\psi)^{1/2} (d-1)^2$, $c_6 \triangleq (d-1)^2 (2\psi c_1 )^{1/2} $, $c_7 \triangleq  \left(\frac{\psi c_1 Q  + \psi c_4 }{2} \right)^{1/2} $, $c_8 \triangleq  \left( 
\frac{\psi \hat{c}_\sigma c_1    }{12 }  \right)^{1/2}$, and  $c_9 \triangleq  \left( \frac{\psi c_3  }{3}\right)^{1/2} $.

Substituting the preceding  inequality in \eqref{eq:Lamba_convergence_starter}
 yields
 \begin{align*}
&(k+d)^4\Lambda_{k+1} 
   \leq \tilde{k}^4 \Lambda_k + c_1  \left( c_5  \zeta\Lambda_0^{1/2}  +       c_6 ( \zeta^3 \lambda_0)^{1/2}\right)  +  \left(  c_1  c_7   +  c_4 \right) \zeta^2 (k+d)\\
   &\quad +  c_1 \zeta \left( c_8(\sigma \zeta)^{1/2}  +  c_9\sigma \right) \tilde{k}^{3/2}     + c_3 \sigma^2 (k+d)^2  + \mathcal{O}\left(\zeta \left( \sigma + (\sigma \zeta)^{1/2} \right) k  + \zeta^2 k^{1/2}\right) .
\end{align*}
We sum the preceding relation over iterations $k=0,1,...,K-1$ to obtain
 \begin{align*}
&\tilde{K}^4\Lambda_{K} 
   \leq (d-1)^4 \Lambda_0 + c_1  \left( c_5  \zeta\Lambda_0^{1/2}  +       c_6 ( \zeta^3 \lambda_0)^{1/2}\right)K +  \frac{\left(  c_1  c_7   +  c_4 \right) \zeta^2 K^2}{2}\\
   &\quad +  \frac{c_1 \zeta \left( c_8(\sigma \zeta)^{1/2}  +  c_9\sigma \right) K^{5/2}}{\sqrt{4}}     + \frac{c_3 \sigma^2 K^3}{3}  + \mathcal{O}\left(\zeta \left( \sigma + (\sigma \zeta)^{1/2} \right) K^2  + \zeta^2 K^{3/2}\right) ,
\end{align*}
where we used \eqref{eq:small_lyap_sum_k_1}, \eqref{eq:small_lyap_sum_k_root_cube}, \eqref{eq:small_lyap_sum_k_2} and \eqref{eq:small_lyap_sum_k_root} to bound the $3^{rd}$, $4^{th}$, $5^{th}$ and $6^{th}$ terms on the right-hand side, respectively.

Dividing both sides of the inequality above with $\tilde{K}^4$ yields
\begin{equation}
    \label{eq:Lambda_k_bound_final}
    \begin{split}
  \Lambda_{K} 
   &\leq \frac{(d-1)^4 \Lambda_0}{\tilde{K}^4} + \frac{c_1 \zeta \left( c_5  \Lambda_0^{1/2}  +       c_6 ( \zeta \lambda_0)^{1/2}\right)}{\tilde{K}^3} +  \frac{\left(  c_1  c_7   +  c_4 \right) \zeta^2 }{2\tilde{K}^2}\\
   &\quad +  \frac{c_1 \zeta \left( c_8(\sigma \zeta)^{1/2}  +  c_9\sigma \right)}{\sqrt{4}\tilde{K}^{3/2}}     + \frac{c_3 \sigma^2 }{3\tilde{K}}  + \mathcal{O}\left(\frac{\zeta \left( \sigma + (\sigma \zeta)^{1/2} \right) }{K^2}  + \frac{\zeta^2 }{K^{5/2}}\right)\\
   &\leq \frac{48^4 \Lambda_0}{\pmi^4 M^4 \mathcal{D}^4 \tilde{K}^4}  + \frac{  2 \cdot 48^{7/2}  (1 + \bar{q})^{1/2} L   \zeta \left( \tfrac{\pmi^{1/2}}{ 48^{1/2}}  \Lambda_0^{1/2}  +       2  L^{1/2} ( \zeta \lambda_0)^{1/2}\right)}{\pmi^{7/2} M^2 \mathcal{D}^2\tilde{K}^3} \\
   &\quad +  \frac{ 48^3 L^3\left( 8 \cdot \delta^3 (1 + L) + \tfrac{\pmi^{1/2}(1+\bar{q})^{1/2}\left(\tfrac{ \pmi    \bar{q} }{6 }  +\delta     +  8 \cdot \delta^3  (1 + L)   \right)^{1/2}
   }{48^{1/2}L^{1/2}}     \right) \zeta^2 }{\pmi^3 \tilde{K}^2}\\
   &\quad +  \frac{ 2 \cdot 48^2 L (1+\bar{q})^{1/2}  \zeta \left(\tfrac{  \pmi^{1/2}  L^{1/2} \left( 
\tfrac{M  }{4 }  \right)^{1/2} (\sigma \zeta)^{1/2} }{48^{1/2}}  +  \left( \tfrac{4    \left( 1 + 2 \delta   \left( \kappa + \tfrac{2}{M}  \right)\right) }{3M}\right)^{1/2}\sigma \right)}{\sqrt{4}\pmi^2 \tilde{K}^{3/2}}  \\
   &\quad + \frac{4 \cdot 48^2  \left( 1 + 2 \delta   \left( \kappa + \tfrac{2}{M}  \right)\right)\sigma^2 }{3\pmi^2 M \tilde{K}}  + \mathcal{O}\left(\frac{\zeta \left( \sigma + (\sigma \zeta)^{1/2} \right) }{K^2}  + \frac{\zeta^2 }{K^{5/2}}\right),
    \end{split}
\end{equation}
where we applied the definitions of the constants $c_1 - c_9$ to obtain the last inequality (for the derivations of simplified constants, see Appendix \ref{sec:simplified_constants})

Comparing \eqref{eq:Lambda_k_bound_final} to \eqref{eq:rough_bound_Lambda}, the convergence rate of $\{\Lambda_k\}$ has been improved from $\mathcal{O}\left( \frac{\sigma^2 + \sigma \zeta}{K} + \frac{\zeta^2}{K^2}\right)$ to $\mathcal{O}\left( \frac{\sigma^2}{K} + \frac{\sigma \zeta + \zeta \sqrt{\sigma \zeta}}{K^{3/2}} + \frac{\zeta^2}{K^2}\right).$ Next, we bound the terms $\Lambda_K$ and $\Lambda_0$ to derive our final result. Due to the non-negativity of the terms $\E[\|x_{i,K} - x_{i,K}^\star\|^2]$ and $\E[\|y_{i,K} - x_{i,K}^\star\|^2]$ for $i \in [n]$ , we have
\begin{align*}
    \Lambda_K &\geq \frac{1}{p_s} \E[\|x_{s,K} - x^\star\|^2].
\end{align*}

Let $s_0 \triangleq \frac{1}{\beta_0^{1/3}} + \frac{1}{\beta_0^{2/3}} $.  Using the fact that $x_{i,0} = y_{i,0} = x_{s,0}$ for $i \in [n]$, we bound $\Lambda_0$ as 
\begin{align*}
    \Lambda_0 &= \frac{1}{p_s}\|x_{s,0} - x^\star\|^2 +  \frac{s_0}{n}\sum_{i=1}^n \frac{1}{p_i} \|x_{s,0} - x^\star +  \alpha_0 M\nabla f_i(x^\star)\|^2\\
    &\leq \frac{1}{p_s}\|x_{s,0} - x^\star\|^2 +  \frac{4}{n}\sum_{i=1}^n \frac{1}{p_i} \|x_{s,0} - x^\star +  \alpha_0 M\nabla f_i(x^\star)\|^2\\
    &\leq \left( \frac{1}{p_s} +  8\bar{q} \right) \|x_{s,0} - x^\star\|^2  +   \frac{8\alpha_0^2M^2}{n}\sum_{i=1}^n \frac{1}{p_i} \|\nabla f_i(x^\star)\|^2\\
    &\leq \left( \frac{1}{p_s} +  8 \bar{q}\right) \|x_{s,0} - x^\star\|^2  +   8\alpha_0^2 L^2 M^2\zeta^2 \bar{q},
\end{align*}
where we obtain the first inequality from \eqref{eq:step_lower_bound}, the second inequality by applying \eqref{eq:young_ineq} with $c=1$, and the last inequality from Assumption \ref{assum:smooth} and the definition of $\zeta$.

Similarly, for $\lambda_0$ we have
\begin{align*}
    \lambda_0 &\leq  \frac{1}{p_s}\|x_{s,0} - x^\star\| +  \frac{4}{n}\sum_{i=1}^n \frac{1}{p_i} \|x_{s,0} - x^\star +  \alpha_0 M\nabla f_i(x^\star)\|\\
    &\leq  \left( \frac{1}{p_s} +  4\bar{q} \right)\|x_{s,0} - x^\star\| + 4 \alpha_0 L M \zeta \bar{q},
\end{align*}
where the last inequality follows from applying the triangle inequality.

Using the fact that $\sqrt{a+b} \leq \sqrt{a} + \sqrt{b}$ for any $a,b>0$, we also obtain the following relations from the two inequalities above
\begin{align*}
    \Lambda_0^{1/2}
    &\leq \left( \frac{1}{p_s} +  8 \bar{q}\right)^{1/2} \|x_{s,0} - x^\star\|  +   (8 \bar{q})^{1/2} \alpha_0 L M \zeta\\
 \lambda_0^{1/2} 
    &\leq  \left( \frac{1}{p_s} +  4\bar{q} \right)^{1/2}  \|x_{s,0} - x^\star\|^{1/2} + (4 \alpha_0 L M  \bar{q} \zeta)^{1/2}.
\end{align*}
Let $\Delta_k \triangleq \E[\|x_{s,k} - x^\star\|^2]$. Substituting the 4 preceding relations in \eqref{eq:Lambda_k_bound_final} yields
\begin{equation*}
    \begin{split}
 &\frac{\Delta_k}{p_s}
   \leq \frac{48^4\left( \frac{1}{p_s} +  8 \bar{q}\right) \|x_{s,0} - x^\star\|^2}{\pmi^4 M^4 \mathcal{D}^4 \tilde{K}^4}  + \frac{  2 \cdot 48^{7/2}  (1 + \bar{q})^{1/2} L    \tfrac{\pmi^{1/2}}{ 48^{1/2}}  \left( \frac{1}{p_s} +  8 \bar{q}\right)^{1/2} \zeta \|x_{s,0} - x^\star\|}{\pmi^{7/2} M^2 \mathcal{D}^2\tilde{K}^3} \\
   &\quad + \frac{  4 \cdot 48^{7/2}  (1 + \bar{q})^{1/2} L^{3/2}     \left( \frac{1}{p_s} +  4\bar{q} \right)^{1/2}     \zeta^{3/2} \|x_{s,0} - x^\star\|^{1/2} }{\pmi^{7/2} M^2 \mathcal{D}^2\tilde{K}^3} \\
   &\quad +  \frac{ 48^3 L^3\left( 8 \cdot \delta^3 (1 + L) + \tfrac{\pmi^{1/2}(1+\bar{q})^{1/2}\left(\tfrac{ \pmi    \bar{q} }{6 }  +\delta     +  8 \cdot \delta^3  (1 + L)   \right)^{1/2}
   }{48^{1/2}L^{1/2}}     \right) \zeta^2 }{\pmi^3 \tilde{K}^2}\\
   &\quad +  \frac{ 2 \cdot 48^2 L (1+\bar{q})^{1/2}  \zeta \left(\tfrac{  \pmi^{1/2}  L^{1/2} \left( 
\tfrac{M  }{4 }  \right)^{1/2} (\sigma \zeta)^{1/2} }{48^{1/2}}  +  \left( \tfrac{4    \left( 1 + 2 \delta   \left( \kappa + \tfrac{2}{M}  \right)\right) }{3M}\right)^{1/2}\sigma \right)}{\sqrt{4}\pmi^2 \tilde{K}^{3/2}}  \\
   &\quad + \frac{4 \cdot 48^2  \left( 1 + 2 \delta   \left( \kappa + \tfrac{2}{M}  \right)\right)\sigma^2 }{3\pmi^2 M \tilde{K}}  + \mathcal{O}\left(\frac{\zeta \left( \sigma + (\sigma \zeta)^{1/2} \right) }{K^2}  + \frac{\zeta^2 }{K^{5/2}}\right).
    \end{split}
\end{equation*}
Multiplying both sides of the relation above with $p_{s}$ and applying the definition of $r_p$ yields the final result.
\end{proof}

\subsection{Simplified constants}
\label{sec:simplified_constants}
In this subsection, we simplify the constants appearing in the proofs of Appendix \ref{sec:proof_of_thm_8}. First, we define the condition number 
\begin{align*}
    \kappa \triangleq \tfrac{L}{\mu} \geq 1.
\end{align*}
Hence, we have $\gamma \triangleq \frac{2\mu L}{\mu + L} =\frac{2 }{1 + \frac{1}{\kappa}}$ and  $\gamma \in [1,2)$. We also define $\delta \triangleq 1 - \frac{1}{M}$, where $M$ is the number of local SGD steps. Note that $\delta = 0$ when $M=1$ (SGD).

\textbf{Constants appearing in Lemma \ref{lem:x_norm}.}
\begin{align*}
    \hat{c}_\zeta &\triangleq \tfrac{ L^2 M (M-1) }{2} \leq \tfrac{\delta L^2 M^2}{2} , \quad 
    \hat{c}_\sigma \triangleq   \left( 1 + \tfrac{\gamma}{4L}\right) M\leq   \left( 1 + \tfrac{ \mu  }{2 (\mu + L) }\right) M \leq \tfrac{3M}{2}.
\end{align*}

\textbf{Constants appearing in Lemma \ref{lem:small_lyapunov}.}
\begin{align*}
    Q \triangleq  4 c \bar{q} LM   +  2c^2  \hat{c}_\zeta \leq    \tfrac{4 \cdot 48 \bar{q} L}{\pmi  \gamma }   +  \tfrac{48^2\delta L^2  }{\pmi^2  \gamma^2 } \leq \tfrac{48^2 L^2}{\pmi^2 } \left( \delta +  \tfrac{ \pmi    \bar{q} }{6 } \right) .
\end{align*}

\textbf{Constants appearing in  Lemma \ref{lem:x_norm_squared}.}
\begin{align*}
 \tilde{c}_\sigma &\triangleq 1 + \tfrac{2L\left( 1 - \tfrac{1}{M}\right) }{\mu} + \tfrac{4 L^2 \left( 1 - \tfrac{1}{M}\right)}{M(\mu + L)^2} \leq 1 + 2 \delta \left( \kappa + \tfrac{2}{M} \right) , \quad 
     \tilde{c}_\zeta \triangleq  \tfrac{2}{\gamma} + \tfrac{2}{\mu + L} \leq 2\left( 1 + \tfrac{1}{L}\right)\\
     c_\sigma &\triangleq 2 M \tilde{c}_\sigma \leq 2M + 4 \delta   \left( 2 + \kappa M  \right) , \quad 
     c_\zeta \triangleq 4L^4 (M-1)^3 \tilde{c}_\zeta \leq 8  \delta^3 M^3 L^3 (1 + L).
\end{align*}
\textbf{Constants appearing in Theorem \ref{thm:convergence_strong_convex}.}
\begin{align*}
c_1 &\triangleq 2 c LM \leq    \tfrac{2 \cdot 48 L }{\pmi}   , \quad c_2 \triangleq 8 c^2L^2 M^2 \bar{q} \leq    \tfrac{8 \cdot 48^2L^2\bar{q}}{\pmi^2   }  \\
c_3 &\triangleq 2 c^2 c_\sigma \leq \tfrac{4 \cdot 48^2  \left( 1 + 2 \delta   \left( \kappa + \tfrac{2}{M}  \right)\right) }{\pmi^2 M }, \quad 
c_4 \triangleq 2 c^3 c_\zeta \leq  \tfrac{16 \cdot 48^3\delta^3  L^3 (1 + L)  }{\pmi^3   }   \\
    c_5 &\triangleq (\psi)^{1/2} (d-1)^2 \leq \tfrac{48^2(1 + \bar{q})^{1/2}}{\pmi^2 M^2 \mathcal{D}^2} , \quad  c_6 \triangleq (d-1)^2 (2\psi c_1 )^{1/2} \leq \tfrac{2 \cdot 48^{5/2} ((1+\bar{q})L)^{1/2}}{\pmi^{5/2} M^2 \mathcal{D}^2}\\
    c_7 &\triangleq  \left(\frac{\psi c_1 Q  + \psi c_4 }{2} \right)^{1/2} \leq  \tfrac{48^{3/2} L^{3/2}(1+\bar{q})^{1/2}\left(\tfrac{ \pmi    \bar{q} }{6 }  +\delta     +  8 \cdot \delta^3  (1 + L)   \right)^{1/2}}{\pmi^{3/2}} 
    \\
    c_8 &\triangleq  \left( 
\tfrac{\psi \hat{c}_\sigma c_1    }{12 }  \right)^{1/2} \leq \tfrac{48^{1/2} L^{1/2} \left( 
\tfrac{(1+\bar{q})M  }{4 }  \right)^{1/2}}{\pmi^{1/2}} , \quad
c_9 \triangleq  \left( \frac{\psi c_3  }{3}\right)^{1/2} \leq \tfrac{48 \left( \tfrac{4 (1+\bar{q})   \left( 1 + 2 \delta   \left( \kappa + \tfrac{2}{M}  \right)\right) }{3M}\right)^{1/2}}{\pmi}.
\end{align*}

\subsection{List of sums}
\label{sec:calculation_sums}
In this Appendix we derive upper bounds for the sums appearing in Appendix \ref{sec:proof_of_thm_8}. For positive scalar $a$, the following sums are bounded:
\begin{align}
     \sum_{k=0}^K (k+a) &=\frac{K^2}{2} + \left( a +\tfrac{1}{2}\right)K + a\label{eq:small_lyap_sum_k_1}\\
 \sum_{k=0}^K (k+a)^2 &= \frac{K^3}{3} + \left( a + \tfrac{1}{2}\right) K^2 + \left(a^2 + a + \tfrac{1}{6} \right) K + a^2 \label{eq:small_lyap_sum_k_2}\\
  \sum_{k=0}^K (k+a)^3 &= \frac{K^4}{4} + \left(a + \tfrac{1}{3} \right) K^3 + \theta_1 K^2 + \theta_2 K + a^3 \label{eq:small_lyap_sum_k_3}\\
 \sum_{k=0}^K (k+a)^{1/2} &\leq  \frac{(K+1)^{3/2}}{\sqrt{2}} + \sqrt{ a +\tfrac{1}{2}} (K+1) + \sqrt{a} \sqrt{K+1}.\label{eq:small_lyap_sum_k_root},
\end{align}
\begin{equation}
    \label{eq:small_lyap_sum_k_root_cube}
    \begin{split}
         \sum_{k=0}^K (k+a)^{3/2} &\leq \frac{(K+1)^{5/2}}{\sqrt{4}} + \sqrt{a + \tfrac{1}{3} } (K+1)^{2} + \sqrt{\theta_1} (K+1)^{3/2} \\
         &\quad + \sqrt{ \theta_2} (K+1) + a^{3/2} (K+1)^{1/2} .
    \end{split}
\end{equation}
where $\theta_1 \triangleq \tfrac{3a^2}{2}+\tfrac{3a}{2}+\tfrac{1}{4}$ and $\theta_2 \triangleq a^3 + \tfrac{3a^2}{2} + \tfrac{a}{2}$.
\begin{proof} Proving \eqref{eq:small_lyap_sum_k_1} is trivial. To prove \eqref{eq:small_lyap_sum_k_2}, we expand the square inside the sum to obtain
\begin{align*}
     \sum_{k=0}^K (k+a)^2 =  \sum_{k=0}^K (k^2 +2ak + a^2) = \frac{K^3}{3} + \left( a + \tfrac{1}{2}\right) K^2 + \left(a^2 + a + \tfrac{1}{6} \right) K + a^2.
\end{align*}
Similarly, we obtain \eqref{eq:small_lyap_sum_k_3} as follows
\begin{align*}
    \sum_{k=0}^K (x+a)^3 &= \sum_{k=0}^K (x^3 + 3ax^2 + 3a^2x + a^3)\\
    &= \frac{K^4}{4} + \frac{K^3}{3} + \frac{K^2}{4} + 3a \left( \frac{K^3}{3} + \frac{K^2}{2} + \frac{K}{6}\right) +\frac{3a^2  K(K+1)}{2} + a^3 (K+1)\\
    &= \frac{K^4}{4} + \left(a + \tfrac{1}{3} \right) K^3 + \left( \tfrac{3a^2}{2}+\tfrac{3a}{2}+\tfrac{1}{4}\right) K^2 + \left( a^3 + \tfrac{3a^2}{2} + \tfrac{a}{2} \right) K + a^3.
\end{align*}
Applying the definitions of $\theta_1$ and $\theta_2$ yields \eqref{eq:small_lyap_sum_k_3}.

Next, we apply Jensen's inequality for concave functions to $ \sum_{k=0}^k \sqrt{k+a}$ to obtain
\begin{align*}
    \sum_{k=0}^k (k+a)^2 &\leq \sqrt{K+1} \sqrt{  \sum_{k=0}^k (k+a)}.
\end{align*}
Applying \eqref{eq:small_lyap_sum_k_2} to the preceding relation and using the standard inequality $\sqrt{a+b} \leq \sqrt{a} + \sqrt{b}$ for $a,b>0$ recovers \eqref{eq:small_lyap_sum_k_root}.

To prove \eqref{eq:small_lyap_sum_k_root_cube}, we again apply Jensen's inequality for concave functions  to obtain
\begin{align*}
    \sum_{k=0}^k (k+a)^{3/2} &\leq \sqrt{K+1} \sqrt{  \sum_{k=0}^k (k+a)^3}.
\end{align*}
Applying \eqref{eq:small_lyap_sum_k_3} to the relation above completes the proof.
\end{proof}

\section{Proof of Theorem \ref{thm:conv_as}}
\label{sec:area_convex_proof}
For arbitrary $x^\star \in \mathcal{X}^\star$, we define the Lyapunov function
\begin{align}
\label{eq:luap_co}
 V_k \triangleq \frac{1}{n}\sum_{i=1}^n p_i^{-1}\|x_{i,k} - x^\star\|^2 +\frac{1}{n}\sum_{i=1}^n p_i^{-1}\|y_{i,k} - x^\star\|^2 + p_s^{-1} \|x_{s,k} - x^\star\|^2   .
\end{align}
We also state the following standard property of the subgradient $d(x)$ of a convex function $f$ at $x$ 
\begin{align}
\label{eq:subgradient_property}
    f(x_1)\geq f(x_2) + \langle d(x_2), x_1-x_2\rangle, \quad \forall x_1, x_2 \in \R^d.
\end{align}
In the next preliminary lemma, we bound the divergence in the objective value caused by multiple local subgradient steps.
\begin{lemma} 
\label{lem:bounded_drift_co}
Under Assumption \ref{assum:bounded_g}, for any positive integer $M$ and $i \in [n]$, the sequence of global server iterates $\{x_{s,k}\}$ satisfies
\begin{align}
\label{eq:lemma_convex_drift}
    - \sum_{m=0}^{M-1} \E[f_i(G_{i,\alpha}^m[x_{s,k}])|\F_k]\leq - M f_i(x_{s,k}) + \frac{\alpha M (M-1) B^2 }{2}.
\end{align}
\end{lemma}
\begin{proof} We prove this lemma by induction. First, define $d_i(x)$ to be a subgradient of $f_i$ at $x$. For any $m > 1$, applying \eqref{eq:subgradient_property} for $f=f_i$, $d=d_i$, $x_1 =G_{i,\alpha}^{m}[x_{s,k}] $ and $x_2=G_{i,\alpha}^{m-1}[x_{s,k}]$ yields
\begin{equation*}
    \begin{split}
       -\E[f_i\left(G_{i,\alpha}^{m}[x_{s,k}] \right)|\F_k]    &\leq  - \E\left[f_i(G_{i,\alpha}^{m-1}[x_{s,k}]) - \left\langle d_i\left( G_{i,\alpha}^{m-1}[x_{s,k}]\right), -\alpha g_i(G_{i,\alpha}^{m-1}[x_{s,k}])\right \rangle  |\F_k\right]\\
     &=  - \E\left[f_i(G_{i,\alpha}^{m-1}[x_{s,k}]) - \left\langle d_i\left( G_{i,\alpha}^{m-1}[x_{s,k}]\right), -\alpha d_i(G_{i,\alpha}^{m-1}[x_{s,k}])\right \rangle  |\F_k\right]\\
     &\leq  - \E\left[f_i(G_{i,\alpha}^{m-1}[x_{s,k}]) |\F_k\right] + \alpha B^2,
    \end{split}
\end{equation*}
where the equality follows from Assumption \ref{assum:grad_stoch} and the last inequality from Assumption \ref{assum:bounded_g} and the Cauchy-Schwarz inequality.

We sum the relation above over local steps $0,1,...,m$ to obtain
\begin{equation}
\label{eq:nonsmooth_local_sgd}
    \begin{split}
       -\E[f_i\left(G_{i,\alpha}^{m}[x_{s,k}] \right)|\F_k]    
     &\leq  - f_i(x_{s,k}) + \alpha m B^2.
    \end{split}
\end{equation}

Equation \eqref{eq:lemma_convex_drift} holds trivially for $M=1$. Suppose that \eqref{eq:lemma_convex_drift}  holds up to the $(M-2)^{th}$ local step. Then after the $(M-1)^{th}$ local step we have
\begin{align*}
     - &\sum_{m=0}^{M-1} \E[f_i(G_{i,\alpha}^m[x_{s,k}])|\F_k] \leq  \E[f_i(G_{i,\alpha}^{M-1} [x_{s,k}])|\F_k] - \sum_{m=0}^{M-2} \E[f_i(G_{i,\alpha}^m[x_{s,k}])|\F_k]\\
     &\quad \leq \E[f_i(G_{i,\alpha}^{M-1} [x_{s,k}])|\F_k]- (M-1) f_i(x_{s,k}) + \frac{\alpha (M-1) (M-2) B^2 }{2}\\
     &\quad \leq M f_i(x_{s,k}) + \alpha (M-1) B^2  + \frac{\alpha (M-1) (M-2) B^2 }{2},
\end{align*}
where the last inequality follows from applying \eqref{eq:nonsmooth_local_sgd} with $m=M-1$.

Observing that $M-1 + \frac{(M-1)(M-2)}{2} = \frac{M(M-1)}{2}$ concludes the proof.
\end{proof}
We re-state Theorem \ref{thm:conv_as} below for ease of reference below, followed by its proof.
\begin{thm_convex}
Let $\bar{x}_{s,K} \triangleq \frac{1}{K}\sum_{k=0}^{K-1}x_{s,k}$ be the ergodic average of the global server iterates  $x_{s,k}$ of Algorithm \ref{alg:area_server} after $K$ iterations. Suppose that Assumptions \ref{assum:grad_stoch} and \ref{assum:bounded_g} hold, and that the step size sequence $\{\alpha_k\}$ in Algorithm \ref{alg:area_client} is   $\alpha =\sqrt{\frac{\left( \frac{1}{p_s} + 2\bar{q}\right) \|x_{s,0} - x^\star\|^2}{\left( \sigma^2 + \frac{(M+1) B^2}{2 }\right) MK}} $ for all $k$, where $\bar{q} \triangleq \frac{1}{n}\sum_{i=1}^n \frac{1}{p_i}$. Then following inequality holds for $\bar{x}_{s,K}$
\begin{align*}
E[f \left( \bar{x}_{s,K}  \right) - f^\star ]  &\leq \sqrt{\frac{\left( \frac{1}{p_s} + 2\bar{q}\right) \left( \sigma^2 + \frac{(M+1) B^2}{2 }\right) \|x_{s,0} - x^\star\|^2}{ MK} }.
\end{align*}
\end{thm_convex}
 \begin{proof}
Let $x^\star$ be an arbitrary minimum of $f$. Subtracting $x^\star$ from Eq.~\eqref{eq:x_k} and taking the squared norm on both sides and the expectation conditional on $\F_k$ yields
\begin{equation}
\label{eq:init_dist_co}
    \begin{split}
          \E[\|x_{i,k+1} - x^\star\|^2|\F_k] &= (1-p_i)\|x_{i,k}  - x^\star\|^2 + p_i \overbrace{\E[\left\|G_{i,\alpha}^M[x_{s,k}] - x^\star\right\|^2|\F_k]}^{\triangleq T_1}.
    \end{split}
\end{equation}
Define $d_i(x)$ to  be a subgradient of $f_i$ at $x$. We apply Assumption \ref{assum:grad_stoch} to the second term on the right-hand side of \eqref{eq:init_dist_co} to obtain
\begin{align*}
   T_1 &= \E[\left\|G_{i,\alpha}^{M-1}[x_{s,k}] -\alpha d_i \left(G_{i,\alpha}^{M-1}[x_{s,k}] \right) -  x^\star\right\|^2|\F_k] + \alpha^2 \sigma^2\\
    &\leq \E[\left\|G_{i,\alpha}^{M-1}[x_{s,k}] -  x^\star\right\|^2|\F_k] + \alpha^2 B^2\\
    &\quad +  2\E[\left\langle G_{i,\alpha}^{M-1}[x_{s,k}] - x^\star,  -\alpha d_i \left(G_{i,\alpha}^{M-1}[x_{s,k}] \right) \right\rangle|\F_k]+ \alpha^2 \sigma^2\\
    &\leq  \E[\left\|G_{i,\alpha}^{M-1}[x_{s,k}] -  x^\star\right\|^2|\F_k]  - 2\alpha \E[ f_i \left( G_{i,\alpha}^{M-1}[x_{s,k}]  \right) - f_i(x^\star) |\F_k] + \alpha^2 (\sigma^2 + B^2),
\end{align*}
where the second inequality follows from Assumption \ref{assum:bounded_g} and the last inequality from  \eqref{eq:subgradient_property} with $f=f_i$, $d=d_i$, $x_1 =x^\star$ and $x_2=G_{i,\alpha}^{M-1}[x_{s,k}]$.

Evaluating the telescoping sum above yields
\begin{align*}
    T_1  
    &\leq \left\|x_{s,k} -  x^\star\right\|^2  - 2\alpha \sum_{m=0}^{M-1} \E[ f_i \left( G_{i,\alpha}^{m}[x_{s,k}]  \right) - f_i(x^\star) |\F_k]+ \alpha^2 M (\sigma^2 + B^2).
\end{align*}
We apply Lemma \ref{lem:bounded_drift_co} to the preceding relation  to obtain
\begin{align*}
   T_1 
    &\leq \left\|x_{s,k} -  x^\star\right\|^2  - 2\alpha  M \left( f_i(x_{s,k}) - f_i(x^\star) \right)   + \alpha^2 M \sigma^2  + \frac{\alpha^2 M (M+1) B^2}{2}.
\end{align*}
Substituting the relation above back in \eqref{eq:init_dist_co}, and multiplying both sides with $p_i^{-1}$ yields
    \begin{equation*}
    \begin{split}
          \E[p_i^{-1} \|x_{i,k+1} - x^\star\|^2|\F_k] 
         &\leq (p_i^{-1}-1)\|x_{i,k}  - x^\star\|^2 +
 \left\|x_{s,k} -  x^\star\right\|^2- 2\alpha M  (f_i \left( x_{s,k}  \right) - f_i(x^\star) ) \\
 &\quad  + \alpha^2 M \sigma^2  + \frac{\alpha^2 M (M+1) B^2}{2}.
    \end{split}
\end{equation*}
We average the inequality above over $i \in [n]$ and apply the definition of $f^\star$ to obtain
    \begin{equation}
    \label{eq:lyap_convex_1}
    \begin{split}
          \E\left[\frac{1}{n}\sum_{i=1}^n p_i^{-1} \|x_{i,k+1} - x^\star\|^2|\F_k\right] 
         &\leq \frac{1}{n}\sum_{i=1}^n (p_i^{-1}-1)\|x_{i,k}  - x^\star\|^2 +
 \left\|x_{s,k} -  x^\star\right\|^2 \\
 &\quad - 2\alpha M  (f \left( x_{s,k}  \right) - f^\star ) + \alpha^2 M \sigma^2  + \frac{\alpha^2 M (M+1) B^2}{2}.
    \end{split}
\end{equation}
Next, we subtract $x^\star$ from \eqref{eq:y_k}, and take the squared norm and the expectation conditional on $\F_k$ on both sides of the resulting equality to obtain
\begin{equation*}
    \begin{split}
         \E[\|y_{i,k+1} - x^\star\|^2|\F_k] &= (1-p_i) \|y_{i,k} - x^\star \|^2 + p_i\|x_{i,k} - x^\star \|^2.
    \end{split}
\end{equation*}
Multiplying both sides of the preceding relation with $p_i^{-1}$ and averaging over $i \in [n]$ yields
\begin{equation}
\label{eq:lyap_convex_2}
    \begin{split}
         \E\left[\frac{1}{n}\sum_{i=1}^n p_i\|y_{i,k+1} - x^\star\|^2|\F_k\right] &= \frac{1}{n}\sum_{i=1}^n(p_i^{-1}-1) \|y_{i,k} - x^\star \|^2 +\frac{1}{n}\sum_{i=1}^n \|x_{i,k} - x^\star \|^2.
    \end{split}
\end{equation}
Similarly, subtracting $x^\star$ from Eq.~\eqref{eq:xs_k} and taking the squared norm on both sides and the expectation conditional on $\F_k$, yields
\begin{equation*}
    \begin{split}
     \E[\|x_{s,k+1} - x^\star\|^2|\F_k] &= (1-p_s) \|x_{s,k} - x^\star\|^2 + p_s\|\bar{y}_k - x^\star\|^2\\
     &\leq  (1-p_s) \|x_{s,k} - x^\star\|^2 + \frac{p_s}{n}\sum_{i=1}^n \|y_{i,k} - x^\star\|^2,
    \end{split}
\end{equation*}
where the last inequality follows from applying Jensen's inequality.

Multiplying both sides of the preceding relation with $p_s^{-1}$ yields
\begin{equation}
\label{eq:lyap_convex_3}
    \begin{split}
     \E[p_s^{-1}\|x_{s,k+1} - x^\star\|^2|\F_k] 
     &\leq  (p_s^{-1}-1) \|x_{s,k} - x^\star\|^2 + \frac{1}{n}\sum_{i=1}^n \|y_{i,k} - x^\star\|^2.
    \end{split}
\end{equation}
We add Equations \eqref{eq:lyap_convex_1}, \eqref{eq:lyap_convex_2}  and \eqref{eq:lyap_convex_3} together and apply the definition of $V_k$ to obtain
\begin{align*}
    \E[V_{k+1}|\F_k] &\leq V_k
   - 2\alpha M  (f \left( x_{s,k}  \right) - f^\star ) + \alpha^2 M \sigma^2  + \frac{\alpha^2 M (M+1) B^2}{2}.
\end{align*}
We take the total expectation on both sides of the relation above and re-arrange to obtain 
\begin{align*}
 2\alpha M  \E[f \left( x_{s,k}  \right) - f^\star ]  &\leq \E[V_k]- \E[V_{k+1}]  + \alpha^2 M \sigma^2  + \frac{\alpha^2 M (M+1) B^2}{2}.
\end{align*}
Applying the preceding equation recursively over iterations $k=0,1,...,K-1$ and multiplying both sides of the result with $\frac{1}{2\alpha M K }$ yields
\begin{align*}
\frac{1}{K} \sum_{K=0}^{K-1}\E[f \left( x_{s,k}  \right) - f^\star ]  &\leq \frac{V_0- \E[V_{K}]}{2 \alpha MK} + \frac{\alpha \sigma^2}{2} + \frac{\alpha (M+1) B^2}{4 }.
\end{align*}
By the non-negativity of $V_k$, for the first term on the right-hand side of the preceding relation we have
\begin{align*}
   \frac{V_0- \E[V_{K}]}{2 \alpha MK} &\leq \frac{V_0}{2 \alpha MK}= \frac{\left( \frac{1}{p_s} + 2\bar{q}\right) \|x_{s,0} - x^\star\|^2 }{2 \alpha MK},
\end{align*}
where the last equality follows from the definitions of $V_0$ and $\bar{q}$, and from the fact that $x_{s,0}=x_{i,0} = y_{i,0}$ for all $i \in [n]$.

Combining the two preceding relations and applying the definition of the ergodic average $\bar{x}_{s,K}$ and the convexity of $f$ yields
\begin{align*}
E[f \left( \bar{x}_{s,K}  \right) - f^\star ]  &\leq \frac{\left( \frac{1}{p_s} + 2\bar{q}\right) \|x_{s,0} - x^\star\|^2 }{2 \alpha MK} + \frac{\alpha \sigma^2}{2} + \frac{\alpha (M+1) B^2}{4 }.
\end{align*}
Setting $\alpha =\sqrt{\frac{\left( \frac{1}{p_s} + 2\bar{q}\right) \|x_{s,0} - x^\star\|^2}{\left( \sigma^2 + \frac{(M+1) B^2}{2 }\right) MK}} $ in the relation above concludes the proof.
\end{proof}

\section{Additional numerical results for non-IID data (cross-silo FL)}
\label{sec:n4_appendix}
We present our empirical results for a small cross-silo FL setup ($n=4$) and $M \in \{1,50\}$ local SGD steps. For this set of experiments, \textsc{S-FedAvg} samples all $4$ clients at each iteration, while the asynchronous methods perform global updates upon receiving  $\Delta=2$ client updates. The remaining setup remains identical to Section \ref{sec:log_reg}.

\begin{table}[t] \tiny \setlength\tabcolsep{3pt} 
\centering 
\begin{tabularx}{\textwidth}{ C{1.1cm} | C{0.4cm} C{0.5cm} C{2.4cm} C{2.2cm} | C{0.4cm} C{0.5cm} C{2.4cm} C{2.2cm}}
\toprule
 & \multicolumn{4}{c|}{\textbf{Uniform}$\;(\lambda_i = 10)$} & \multicolumn{4}{c}{\textbf{Non-uniform}$\;(\lambda_i \sim \mathcal{N}(10,5))$}\\ \textbf{Method} & $\alpha$ & $\rho$ & Test acc \% ($t_h$) & Loss ($t_h$) & $\alpha$ & $\rho$ & Test acc \% ($t_h$) & Loss  ($t_h$)\\ 
\midrule \multirow{2}{*}{\textsc{S-FedAvg}}  & $10^0$ & $0.15$ & $89.34$ ($88.90$ - $90.00$) & $0.65$ ($0.64$ - $0.65$)& $10^0$ & $0.34$ & $86.86$ ($67.80$ - $90.24$) & $0.74$ ($0.60$ - $1.37$) \\
\multirow{2}{*}{\textsc{S-FedAvg}}  & $10^1$ & $0.06$ & $84.17$ ($77.35$ - $88.94$) & $0.85$ ($0.60$ - $1.21$)
& $10^1$ & $0.16$ & $84.70$ ($65.53$ - $90.78$) & $1.00$ ($0.53$ - $3.90$) \\
\midrule
\multirow{2}{*}{\textsc{AS-FedAvg}}  & $10^0$ & $0.13$ & $89.72$ ($89.13$ - $90.26$) & $0.61$ ($0.60$ - $0.62$)
& $10^0$ & $0.19$ & $87.39$ ($69.92$ - $90.30$) & $0.70$ ($0.58$ - $1.59$) \\
\multirow{2}{*}{\textsc{AS-FedAvg}}  & $10^1$ & $0.08$ & $84.94$ ($79.78$ - $89.68$) & $1.38$ ($0.99$ - $2.07$)
& $10^1$ & $0.09$ & $81.75$ ($65.08$ - $89.97$) & $1.76$ ($0.90$ - $5.66$) \\
\midrule
\multirow{2}{*}{\textsc{FedBuff}}   & $10^{-1}$ & $>1$ & $60.43$ ($57.18$ - $65.16$) & $1.55$ ($1.43$ - $1.64$) & $10^{-1}$ & $>1$ & $58.53$ ($37.87$ - $64.27$) & $1.67$ ($1.44$ - $2.76$) \\
\multirow{2}{*}{\textsc{FedBuff}}   & $10^0$ & \boldmath{$0.05$} & \boldmath{$90.80$} ($89.42$ - $91.48$) & \boldmath{$0.48$} ($0.46$ - $0.53$) & $10^0$ & \boldmath{$0.08$} & \boldmath{$89.20$} ($82.85$ - $91.33$) & \boldmath{$0.52$} ($0.43$ - $0.70$) \\
\midrule
\multirow{2}{*}{\textsc{AREA}} & $10^0$ & $0.15$ & $89.30$ ($88.87$ - $89.69$) & $0.65$ ($0.64$ - $0.66$) 
& $10^0$ & $0.16$ & $89.14$ ($87.75$ - $90.04$) & $0.65$ ($0.61$ - $0.71$) \\
\multirow{2}{*}{\textsc{AREA}}   & $10^1$ & $0.07$ & $84.66$ ($74.37$ - $90.43$) & $0.96$ ($0.56$ - $2.57$)
 & $10^1$ & \boldmath{$0.08$} & $84.27$ ($67.71$ - $89.69$) & $0.90$ ($0.60$ - $2.04$) \\
\bottomrule 
\end{tabularx}
\caption{Performance for $n=4$ and $M=1$ local SGD step ($t_h=100$)}
\label{tab:M_1_n_4_non_iid}
\end{table}
\textbf{One local SGD step.} In Table \ref{tab:M_1_n_4_non_iid}, \textsc{FedBuff} with $\alpha=1$ is the first method to cross the $80\%$ test accuracy threshold, and attains the lowest final loss and highest final test accuracy out of all methods. A comparison with Table  \ref{tab:M_1_n_128_non_iid}  ($n=128$) suggests that \textsc{FedBuff} benefits from smaller client populations in the $M=1$ regime, whereas \textsc{AREA} scales better to medium and large settings.

\begin{table}[t] \tiny \setlength\tabcolsep{3pt} 
\centering 
\begin{tabularx}{\textwidth}{ C{1.1cm} | C{0.4cm} C{0.5cm} C{2.4cm} C{2.2cm} | C{0.4cm} C{0.5cm} C{2.4cm} C{2.2cm}}
\toprule
 & \multicolumn{4}{c|}{\textbf{Uniform}$\;(\lambda_i = 10)$} & \multicolumn{4}{c}{\textbf{Non-uniform}$\;(\lambda_i \sim \mathcal{N}(10,5))$}\\ \textbf{Method} & $\alpha$ & $\rho$ & Test acc \% ($t_h$) & Loss ($t_h$) & $\alpha$ & $\rho$ & Test acc \% ($t_h$) & Loss  ($t_h$)\\ 
\midrule \multirow{2}{*}{\textsc{S-FedAvg}} & $10^0$ & $0.13$ & \boldmath{$87.59$} ($86.59$ - $88.35$) & \boldmath{$0.48$} ($0.46$ - $0.49$) & $10^0$ & $0.26$ & \boldmath{$86.17$} ($75.78$ - $87.69$) & \boldmath{$0.56$} ($0.46$ - $1.01$) \\
\multirow{2}{*}{\textsc{S-FedAvg}}  & $10^1$ & $0.07$ & $82.95$ ($78.47$ - $85.66$) & $1.27$ ($1.09$ - $1.71$) & $10^1$ & \boldmath{$0.11$} & $82.93$ ($77.05$ - $85.61$) & $1.27$ ($1.07$ - $1.76$) \\
\midrule
\multirow{2}{*}{\textsc{AS-FedAvg}}  & $10^0$ & $0.72$ & $75.36$ ($53.02$ - $88.79$) & $0.80$ ($0.44$ - $1.49$) & $10^0$ & $>1$ & $70.48$ ($49.17$ - $88.10$) & $1.01$ ($0.43$ - $1.74$) \\
\multirow{2}{*}{\textsc{AS-FedAvg}}  & $10^1$ & $>1$ & $64.83$ ($51.04$ - $79.13$) & $3.80$ ($1.50$ - $10.22$)
& $10^1$ & $>1$ & $63.72$ ($42.68$ - $76.98$) & $4.13$ ($2.05$ - $8.45$) \\
\midrule
\multirow{2}{*}{\textsc{FedBuff}}   & $10^{-1}$ & $0.92$ & $80.76$ ($78.67$ - $81.98$) & $0.93$ ($0.91$ - $0.98$) & $10^{-1}$ & $>1$ & $77.12$ ($66.49$ - $81.93$) & $1.04$ ($0.87$ - $1.33$) \\
\multirow{2}{*}{\textsc{FedBuff}}   & $10^0$ & $>1$ & $55.13$ ($27.04$ - $78.08$) & $2.80$ ($0.79$ - $6.25$)
& $10^0$ & $>1$ & $57.48$ ($37.98$ - $84.94$) & $2.08$ ($0.52$ - $3.79$) \\
\midrule
\multirow{2}{*}{\textsc{AREA}}   & $10^0$ & $0.13$ & $87.46$ ($86.12$ - $88.38$) & \boldmath{$0.48$} ($0.46$ - $0.52$)& $10^0$ & $0.21$ & $85.83$ ($53.72$ - $88.55$) & \boldmath{$0.56$} ($0.45$ - $1.65$) \\
\multirow{2}{*}{\textsc{AREA}}  & $10^1$ & \boldmath{$0.06$} & $82.88$ ($76.87$ - $85.19$) & $1.27$ ($1.08$ - $1.63$) & $10^1$ & $0.13$ & $82.65$ ($73.56$ - $84.81$) & $1.28$ ($1.11$ - $1.92$) \\
\bottomrule 
\end{tabularx}
\caption{Performance for $n=4$ and $M=50$ local SGD steps ($t_h=100$)}
\label{tab:M_50_n_4_non_iid}
\end{table}
\textbf{Fifty local steps.}
In Table \ref{tab:M_50_n_4_non_iid}, the introduction of multiple local SGD steps amplifies client drift and stochastic gradient errors and destabilizes \textsc{FedBuff}, especially under non-uniform client update frequencies. \textsc{S-FedAvg} achieves the best performance in this setting, followed by \textsc{AREA}.

\section{Numerical results for IID data}
\label{sec:IID_sim}

Although the case of IID local data is outside the scope of this work, we include an evaluation of AREA and the baselines under this regime in this appendix. As in the experiments for non-IID data, we vary the number of clients $n \in \{4,128\}$ and the number of local SGD steps $M \in \{1,50\}$. All remaining settings are identical to those used in Section \ref{sec:log_reg} for $n=128$ and Appendix \ref{sec:n4_appendix} for $n=4$. In the experiments reported in this appendix \textsc{AREA} does not always outperform the baselines, but consistently remains a close second. More broadly, \textsc{AREA} exhibits versatile and robust performance across a wide range of settings and hyper-parameter choices.

\subsection{Cross-device FL ($n=128$)}
Because a perfectly uniform label split is not possible for MNIST when $n=128$, we constructed a nearly-IID partition as follows: the first 127 clients receive identical label proportions, while the final client is assigned the remaining samples (Fig. \ref{fig:distro_128_iid}).
\begin{figure}[!ht] \centering \includegraphics[width=0.55\linewidth]{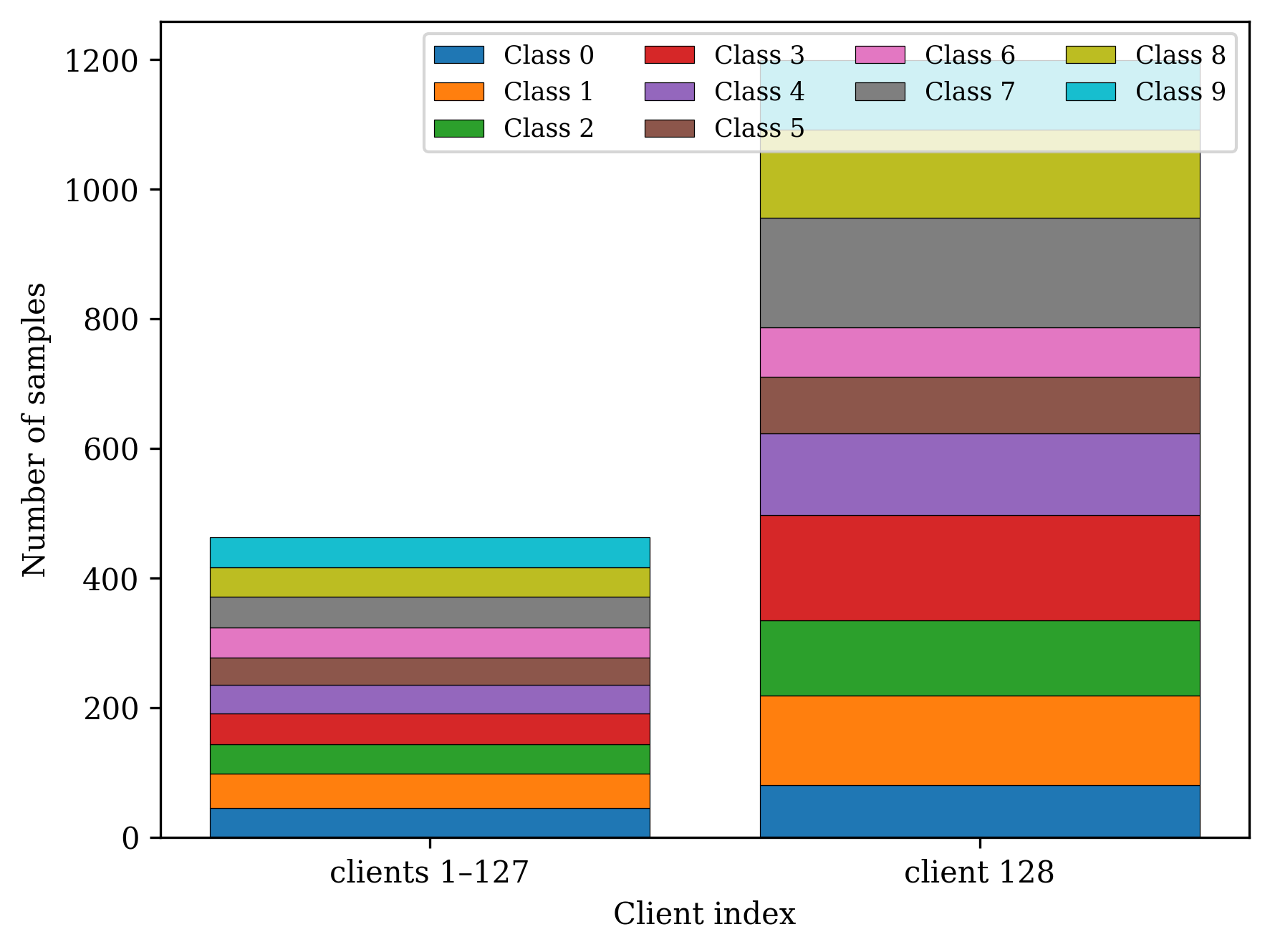} \caption{Nearly-IID label distribution for 128 clients: clients 1–127 have identical proportions; client 128 contains the leftover samples.} \label{fig:distro_128_iid} \end{figure}

\begin{table}[t] \tiny \setlength\tabcolsep{3pt} 
\centering 
\begin{tabularx}{\textwidth}{ C{1.1cm} | C{0.4cm} C{0.5cm} C{2.4cm} C{2.2cm} | C{0.4cm} C{0.5cm} C{2.4cm} C{2.2cm}}
\toprule
 & \multicolumn{4}{c|}{\textbf{Uniform}$\;(\lambda_i = 10)$} & \multicolumn{4}{c}{\textbf{Non-uniform}$\;(\lambda_i \sim \mathcal{N}(10,5))$}\\ \textbf{Method} & $\alpha$ & $\rho$ & Test acc \% ($t_h$) & Loss ($t_h$) & $\alpha$ & $\rho$ & Test acc \% ($t_h$) & Loss  ($t_h$)\\ 
\midrule \multirow{2}{*}{\textsc{S-FedAvg}} & $10^1$ & $>1$ & $62.86$ ($58.96$ - $65.74$) & $1.50$ ($1.40$ - $1.62$)
& $10^1$ & $>1$ & $41.51$ ($22.28$ - $56.14$) & $2.21$ ($1.77$ - $3.07$) \\
\multirow{2}{*}{\textsc{S-FedAvg}}  & $10^2$ & $0.42$ & $85.44$ ($82.94$ - $86.42$) & $0.79$ ($0.76$ - $0.85$)& $10^2$ & $>1$ & $73.16$ ($30.42$ - $84.82$) & $1.24$ ($0.84$ - $3.20$) \\
\midrule
\multirow{2}{*}{\textsc{AS-FedAvg}}  & $10^1$ & $>1$ & $74.71$ ($72.81$ - $76.46$) & $1.14$ ($1.11$ - $1.20$)& $10^1$ & $>1$ & $75.28$ ($68.76$ - $77.54$) & $1.13$ ($1.06$ - $1.31$) \\
\multirow{2}{*}{\textsc{AS-FedAvg}}  & $10^2$ & $0.18$ & $88.80$ ($87.90$ - $89.49$) & $0.67$ ($0.65$ - $0.68$)& $10^2$ & $0.16$ & $88.70$ ($87.33$ - $89.50$) & $0.66$ ($0.64$ - $0.70$) \\
\midrule
\multirow{2}{*}{\textsc{FedBuff}}   & $10^0$ & $0.44$ & $85.64$ ($84.94$ - $86.11$) & $0.80$ ($0.79$ - $0.82$)& $10^0$ & $0.44$ & $85.64$ ($84.90$ - $86.44$) & $0.81$ ($0.79$ - $0.82$) \\
\multirow{2}{*}{\textsc{FedBuff}}    & $10^1$ & \boldmath{$0.04$} & $83.36$ ($68.14$ - $91.03$) & $1.26$ ($0.76$ - $2.80$) & $10^1$ & \boldmath{$0.03$} & $85.31$ ($76.51$ - $89.35$) & $0.99$ ($0.61$ - $2.50$) \\
\midrule
\multirow{2}{*}{\textsc{AREA}}   &  -&- &   -& -
& $10^2$ & $0.53$ & $84.14$ ($80.47$ - $85.80$) & $0.85$ ($0.79$ - $0.98$) \\
\multirow{2}{*}{\textsc{AREA}}  & $10^3$ & $0.05$ & \boldmath{$90.20$} ($88.20$ - $91.53$) & \boldmath{$0.48$} ($0.45$ - $0.52$)& $10^3$ & $0.08$ & \boldmath{$89.95$} ($86.77$ - $91.39$) & \boldmath{$0.54$} ($0.48$ - $0.63$) \\
\bottomrule 
\end{tabularx}
\caption{Performance for $n=128$ and $M=1$ local SGD step ($t_h=15$, IID data)}
\label{tab:M_1_n_128_iid}
\end{table}

\textbf{One local step.} In Table \ref{tab:M_1_n_128_iid}, \textsc{FedBuff} crosses the $80\%$ test accuracy threshold slightly faster than \textsc{AREA}; however, \textsc{AREA} achieves higher final test accuracy and lower final loss in every case. As in Table \ref{tab:M_1_n_128_non_iid}, we observe that  non-uniform update rates have a pronounced negative effect on \textsc{S-FedAvg} while the other methods remain relatively stable across regimes.

\begin{table}[t] \tiny \setlength\tabcolsep{3pt} 
\centering 
\begin{tabularx}{\textwidth}{ C{1.1cm} | C{0.4cm} C{0.5cm} C{2.4cm} C{2.2cm} | C{0.4cm} C{0.5cm} C{2.4cm} C{2.2cm}}
\toprule
 & \multicolumn{4}{c|}{\textbf{Uniform}$\;(\lambda_i = 10)$} & \multicolumn{4}{c}{\textbf{Non-uniform}$\;(\lambda_i \sim \mathcal{N}(10,5))$}\\ \textbf{Method} & $\alpha$ & $\rho$ & Test acc \% ($t_h$) & Loss ($t_h$) & $\alpha$ & $\rho$ & Test acc \% ($t_h$) & Loss  ($t_h$)\\ 
\midrule \multirow{2}{*}{\textsc{S-FedAvg}} & $10^2$ & $0.15$ & $89.60$ ($88.17$ - $90.39$) & $0.57$ ($0.51$ - $0.66$)& $10^2$ & $0.30$ & $87.76$ ($81.34$ - $90.49$) & $0.68$ ($0.52$ - $0.94$) \\
\multirow{2}{*}{\textsc{S-FedAvg}}  & $10^3$ & $0.18$ & $84.57$ ($75.58$ - $87.94$) & $4.01$ ($2.66$ - $13.51$)& $10^3$ & $0.33$ & $85.22$ ($79.39$ - $87.96$) & $3.37$ ($2.80$ - $4.39$) \\
\midrule
\multirow{2}{*}{\textsc{AS-FedAvg}}  & $10^1$ & $0.22$ & $88.05$ ($87.40$ - $88.58$) & $0.71$ ($0.69$ - $0.72$)
& $10^1$ & $0.21$ & $88.01$ ($87.41$ - $88.66$) & $0.70$ ($0.69$ - $0.73$) \\
\multirow{2}{*}{\textsc{AS-FedAvg}}  & $10^2$ & \boldmath{$0.00$} & $90.77$ ($90.04$ - $91.24$) & \boldmath{$0.44$} ($0.43$ - $0.48$)
& $10^2$ & \boldmath{$0.00$} & \boldmath{$90.75$} ($90.32$ - $91.31$) & \boldmath{$0.44$} ($0.42$ - $0.46$) \\
\midrule
\multirow{2}{*}{\textsc{FedBuff}}    & $10^{-1}$ & $>1$ & $43.42$ ($37.16$ - $48.46$) & $2.13$ ($1.96$ - $2.35$)& $10^{-1}$ & $>1$ & $42.39$ ($36.06$ - $48.00$) & $2.19$ ($1.96$ - $2.48$) \\
\multirow{2}{*}{\textsc{FedBuff}}     & $10^0$ & $0.15$ & $90.53$ ($89.48$ - $91.03$) & $0.62$ ($0.59$ - $0.66$)& $10^0$ & $0.14$ & $90.71$ ($90.25$ - $91.04$) & $0.61$ ($0.56$ - $0.69$) \\
\midrule
\multirow{2}{*}{\textsc{AREA}}  & $10^2$ & $0.14$ & \boldmath{$90.87$} ($90.67$ - $91.11$) & $0.51$ ($0.50$ - $0.52$)& $10^2$ & $0.17$ & $90.09$ ($88.98$ - $90.61$) & $0.57$ ($0.54$ - $0.61$) \\
\multirow{2}{*}{\textsc{AREA}} & $10^3$ & $0.04$ & $89.58$ ($89.15$ - $89.77$) & $2.48$ ($2.41$ - $2.55$)& $10^3$ & $0.04$ & $89.47$ ($89.27$ - $89.66$) & $2.42$ ($2.34$ - $2.50$) \\
\bottomrule 
\end{tabularx}
\caption{Performance for $n=128$ and $M=50$ local SGD steps ($t_h=2$, IID data)}
\label{tab:M_50_n_128_iid}
\end{table}

\textbf{Fifty local steps.} In Table \ref{tab:M_50_n_128_iid}, the performance of multiple local SGD steps markedly boosts the performance of \textsc{AS-FedAvg} in the IID setting. \textsc{AREA} remains a close second in this regime. 

\subsection{Cross-silo FL ($n=4$)}
For $n=4$ clients, we obtain a perfectly IID split of the training dataset.

\begin{table}[t] \tiny \setlength\tabcolsep{3pt} 
\centering 
\begin{tabularx}{\textwidth}{ C{1.1cm} | C{0.4cm} C{0.5cm} C{2.4cm} C{2.2cm} | C{0.4cm} C{0.5cm} C{2.4cm} C{2.2cm}}
\toprule
 & \multicolumn{4}{c|}{\textbf{Uniform}$\;(\lambda_i = 10)$} & \multicolumn{4}{c}{\textbf{Non-uniform}$\;(\lambda_i \sim \mathcal{N}(10,5))$}\\ \textbf{Method} & $\alpha$ & $\rho$ & Test acc \% ($t_h$) & Loss ($t_h$) & $\alpha$ & $\rho$ & Test acc \% ($t_h$) & Loss  ($t_h$)\\ 
\midrule \multirow{2}{*}{\textsc{S-FedAvg}} & $10^0$ & $0.15$ & $89.39$ ($88.95$ - $89.81$) & $0.64$ ($0.63$ - $0.65$)& $10^0$ & $0.32$ & $86.87$ ($63.53$ - $89.94$) & $0.74$ ($0.61$ - $1.45$) \\
\multirow{2}{*}{\textsc{S-FedAvg}}  & $10^1$ & $0.05$ & $85.40$ ($75.49$ - $89.57$) & $0.83$ ($0.56$ - $1.75$)  & $10^1$ & $0.51$ & $79.56$ ($15.37$ - $90.78$) & $1.42$ ($0.54$ - $8.56$) \\
\midrule
\multirow{2}{*}{\textsc{AS-FedAvg}}  & $10^0$ & $0.11$ & $90.00$ ($89.65$ - $90.31$) & $0.60$ ($0.59$ - $0.62$)
& $10^0$ & $0.13$ & $89.78$ ($88.17$ - $90.85$) & $0.62$ ($0.56$ - $0.70$) \\
\multirow{2}{*}{\textsc{AS-FedAvg}}  & $10^1$ & $0.05$ & $81.81$ ($64.96$ - $89.18$) & $1.09$ ($0.63$ - $2.79$)
& $10^1$ & $0.05$ & $83.48$ ($55.73$ - $89.90$) & $0.94$ ($0.61$ - $2.84$) \\
\midrule
\multirow{2}{*}{\textsc{FedBuff}}    & $10^{-1}$ & $>1$ & $60.38$ ($55.22$ - $63.03$) & $1.57$ ($1.52$ - $1.71$) & $10^{-1}$ & $>1$ & $60.07$ ($38.34$ - $68.07$) & $1.58$ ($1.33$ - $2.23$) \\
\multirow{2}{*}{\textsc{FedBuff}}     & $10^0$ & \boldmath{$0.04$} & \boldmath{$91.19$} ($90.75$ - $91.69$) & \boldmath{$0.46$} ($0.45$ - $0.48$)& $10^0$ & \boldmath{$0.04$} & \boldmath{$91.26$} ($90.67$ - $91.78$) & \boldmath{$0.46$} ($0.42$ - $0.53$) \\
\midrule
\multirow{2}{*}{\textsc{AREA}}  & $10^0$ & $0.15$ & $89.27$ ($88.87$ - $89.57$) & $0.65$ ($0.64$ - $0.66$) 
& $10^0$ & $0.26$ & $87.88$ ($83.75$ - $90.33$) & $0.71$ ($0.61$ - $0.88$) \\
\multirow{2}{*}{\textsc{AREA}} & $10^1$ & \boldmath{$0.04$} & $85.99$ ($75.94$ - $90.05$) & $0.67$ ($0.49$ - $1.37$)& $10^1$ & $0.06$ & $86.97$ ($81.21$ - $89.95$) & $0.68$ ($0.48$ - $2.04$) \\
\bottomrule 
\end{tabularx}
\caption{Performance for $n=4$ and $M=1$ local SGD step ($t_h=100$, IID data)}
\label{tab:M_1_n_4_iid}
\end{table}

\textbf{One local step.}
In Table \ref{tab:M_1_n_4_iid}, \textsc{FedBuff} consistently outperforms the remaining methods both with respect to convergence speed and final test accuracy and loss. The remaining methods achieve comparable performance.

\begin{table}[t] \tiny \setlength\tabcolsep{3pt} 
\centering 
\begin{tabularx}{\textwidth}{ C{1.1cm} | C{0.4cm} C{0.5cm} C{2.4cm} C{2.2cm} | C{0.4cm} C{0.5cm} C{2.4cm} C{2.2cm}}
\toprule
 & \multicolumn{4}{c|}{\textbf{Uniform}$\;(\lambda_i = 10)$} & \multicolumn{4}{c}{\textbf{Non-uniform}$\;(\lambda_i \sim \mathcal{N}(10,5))$}\\ \textbf{Method} & $\alpha$ & $\rho$ & Test acc \% ($t_h$) & Loss ($t_h$) & $\alpha$ & $\rho$ & Test acc \% ($t_h$) & Loss  ($t_h$)\\ 
\midrule \multirow{2}{*}{\textsc{S-FedAvg}} & $10^0$ & $0.02$ & \boldmath{$91.86$} ($91.66$ - $92.08$) & $0.39$ ($0.38$ - $0.40$) & $10^0$ & $0.11$ & $91.15$ ($87.88$ - $92.12$) & $0.48$ ($0.38$ - $0.73$) \\
\multirow{2}{*}{\textsc{S-FedAvg}}  & $10^1$ & $0.02$ & $89.01$ ($84.38$ - $90.70$) & $0.89$ ($0.77$ - $1.25$)
& $10^1$ & $0.06$ & $88.55$ ($82.65$ - $90.52$) & $0.96$ ($0.80$ - $1.66$) \\
\midrule
\multirow{2}{*}{\textsc{AS-FedAvg}}  & $10^0$ & \boldmath{$0.01$} & $91.81$ ($91.34$ - $92.18$) & \boldmath{$0.37$} ($0.36$ - $0.38$)& $10^0$ & \boldmath{$0.01$} & \boldmath{$91.71$} ($91.21$ - $91.98$) & \boldmath{$0.38$} ($0.36$ - $0.43$) \\
\multirow{2}{*}{\textsc{AS-FedAvg}}  & $10^1$ & \boldmath{$0.01$} & $85.04$ ($76.10$ - $89.73$) & $1.18$ ($0.86$ - $1.88$)& $10^1$ & $0.01$ & $84.65$ ($79.72$ - $90.16$) & $1.18$ ($0.94$ - $1.53$) \\
\midrule
\multirow{2}{*}{\textsc{FedBuff}}     & $10^{-1}$ & $0.47$ & $85.32$ ($84.64$ - $86.11$) & $0.82$ ($0.80$ - $0.85$) & $10^{-1}$ & $0.45$ & $85.65$ ($81.55$ - $87.61$) & $0.81$ ($0.74$ - $0.92$) \\
\multirow{2}{*}{\textsc{FedBuff}}     & $10^0$ & \boldmath{$0.01$} & $86.64$ ($66.85$ - $91.68$) & $0.52$ ($0.36$ - $1.26$) & $10^0$ & $0.01$ & $87.49$ ($75.13$ - $91.84$) & $0.48$ ($0.36$ - $0.85$) \\
\midrule
\multirow{2}{*}{\textsc{AREA}} & $10^0$ & $0.03$ & $91.84$ ($91.58$ - $92.05$)  & $0.39$ ($0.39$ - $0.39$)& $10^0$ & $0.06$ & $91.25$ ($86.54$ - $92.10$) & $0.46$ ($0.38$ - $0.75$) \\
\multirow{2}{*}{\textsc{AREA}}  & $10^1$ & \boldmath{$0.01$} & $89.17$ ($87.25$ - $90.44$) & $0.88$ ($0.78$ - $0.99$)& $10^1$ & $0.02$ & $88.41$ ($79.40$ - $90.59$) & $0.95$ ($0.80$ - $1.59$) \\
\bottomrule 
\end{tabularx}
\caption{Performance for $n=4$ and $M=50$ local SGD steps ($t_h=100$, IID data)}
\label{tab:M_50_n_4_iid}
\end{table}
\textbf{Fifty local steps.} In Table \ref{tab:M_50_n_4_iid} all methods successfully cross the $90\%$ test accuracy threshold with the exception of \textsc{FedBuff}. A similar effect was observed in Table \ref{tab:M_50_n_4_non_iid} ($n=4$, $M=50$, non-IID data), indicating that \textsc{FedBuff} is particularly sensitive to client drift effects in the small cross-silo regime. The remaining methods perform comparably.



\end{document}